\documentclass{article}

     \usepackage[final]{neurips_2021}

\usepackage[utf8]{inputenc} 
\usepackage[T1]{fontenc}    
\usepackage{hyperref}       
\usepackage{url}            
\usepackage{booktabs}       
\usepackage{amsfonts}       
\usepackage{nicefrac}       
\usepackage{microtype}      
\usepackage{xcolor}         
\usepackage{pifont}
\usepackage{amsthm}
\usepackage{amsfonts,amssymb}
\usepackage{amsmath}
\usepackage{amssymb}
\usepackage{amsmath}
\usepackage{amsthm}
\usepackage[utf8]{inputenc} 
\usepackage[T1]{fontenc}    
\usepackage{hyperref}       
\usepackage{url}            
\usepackage{booktabs}       
\usepackage{amsfonts}       
\usepackage{nicefrac}       
\usepackage{microtype}      

\newtheorem{theorem}{Theorem}
\newtheorem{proposition}{Proposition}

\newtheorem{lemma}{Lemma}
\newtheorem{remark}{Remark}
\usepackage{color}
\usepackage{makecell}

\usepackage{algorithm}
\usepackage{algorithmic}
\usepackage{extarrows}
\usepackage{caption}
\usepackage{subfigure}
\usepackage{graphicx}
\usepackage{sidecap}
\usepackage{float}
\usepackage{framed}
\usepackage{tabularx}
\usepackage{twoopt}
\usepackage{adjustbox}

\usepackage{lscape}

\usepackage{tikz}
\usetikzlibrary{fit}

\newcommand\tikzmark[1]{%
  \tikz[remember picture,overlay]\node[inner xsep=0pt] (#1) {};}
\newcommandtwoopt\Textbox[5][11.8cm][2cm]{%
\begin{tikzpicture}[remember picture,overlay]
  \coordinate (aux) at ([xshift=#1]#4);
  \node[inner ysep=3pt,yshift=0.5ex,draw=pink,thick,
    fit=(#3) (aux),baseline] 
    (box) {};
  \node[text width=#2,anchor=north east,
    font=\sffamily\footnotesize,
  align=right
    ] 
    at (box.north east) {#5};
\end{tikzpicture}%
}

\usepackage{authblk}

\usepackage{bm}
\usepackage{makecell}

\newcommand{\policy}{\pi_{\bm{\theta}}}
\newcommand{\policyy}{\pi_{{\bm\theta}^{'}}}

\newcommand{\E}{\mathbb{E}}
\newcommand{\R}{\mathbb{R}}
\newcommand{\Pro}{\mathbb{P}}

\newcommand{\N}{\mathbb{N}}

\newcommand{\calA}{\mathcal{A}}

\newcommand{\calC}{\mathcal{C}}
\newcommand{\calD}{\mathcal{D}}

\newcommand{\calL}{\mathcal{L}}
\newcommand{\calS}{\mathcal{S}}
\newcommand{\calM}{\mathcal{M}}
\newcommand{\calX}{\mathcal{X}}

\newcommand{\bA}{\mathbf{A}}
\newcommand{\bB}{\mathbf{B}}
\newcommand{\bF}{\mathbf{F}}
\newcommand{\bG}{\mathbf{G}}
\newcommand{\bD}{\mathbf{D}}
\newcommand{\bH}{\mathbf{H}}
\newcommand{\bI}{\mathbf{I}}
\newcommand{\bL}{\mathbf{L}}

\newcommand{\ba}{\mathbf{a}}
\newcommand{\bd}{\mathbf{d}}
\newcommand{\bbf}{\mathbf{f}}

\newcommand{\bg}{\mathbf{g}}

\newcommand{\bp}{\mathbf{p}}

\newcommand{\bv}{\mathbf{v}}

\newcommand{\bx}{\mathbf{x}}

\usepackage[colorinlistoftodos, textsize=tiny]{todonotes}
\definecolor{citrine}{rgb}{0.89, 0.82, 0.04}
\definecolor{blued}{RGB}{70,197,221}

\usepackage{bm}
\newcommand{\mathbr}[1]{\bm{\mathbf{#1}}}

\newcommand{\vtheta}{\mathbr{\theta}}
\newcommand{\vrho}{\mathbr{\rho}}

\newcommand{\hyscoreprime}[1][\vtheta]{\nabla_{\vrho'}\log\nu_{\vrho'}}

\newcommand{\bP}{\mathbf{P}}

\title{Constrained Update Projection Approach to Safe Policy Optimization}

\author{%
  Long Yang$^{1,2,*}$, Jiaming Ji$^{1,}$\thanks{L.Yang and J.Ji share equal contributions. $\dag$ G.Pan is the corresponding author.} , Juntao Dai$^1$, Linrui Zhang$^3$, Binbin Zhou,$^4$, Pengfei Li,$^1$, Yaodong Yang$^{2,5}$, Gang Pan$^{1,\dag}$\\
  $^1$College of Computer Science and Technology, Zhejiang University, China \\
  $^2$ School of Artificial Intelligence, Peking University, China\\
  $^3$ Tsinghua Shenzhen International Graduate School, Tsinghua University, China\\
  $^4$ Department of Computer Science and Computing, Zhejiang University City College, China\\
  $^5$ Institute for Artificial Intelligence, Peking University \& BIGAI, China\\
  \texttt{yanglong001@pku.edu.cn,~gpan@zju.edu.cn}\\
}

\begin{document}

\maketitle

\begin{abstract}
Safe reinforcement learning (RL) studies problems where an intelligent agent has to not only maximize reward but also avoid exploring unsafe areas. In this study, we propose CUP, a novel policy optimization method based on \textbf{C}onstrained \textbf{U}pdate \textbf{P}rojection framework that enjoys rigorous safety guarantee. 
Central to our CUP development is the newly proposed surrogate functions along with the performance bound. Compared to previous safe RL methods, 
CUP enjoys the benefits of 
1) CUP generalizes the surrogate functions to generalized advantage estimator (GAE), leading to strong empirical performance. 
2) CUP unifies performance bounds, providing a better understanding and interpretability for some existing algorithms; 
3) CUP provides a non-convex implementation via only first-order optimizers, which does not require any strong approximation on the convexity of the objectives.
To validate our CUP method, we compared CUP against a comprehensive list of safe RL baselines on a wide range of tasks. Experiments show the effectiveness of CUP both in terms of reward and safety constraint satisfaction.
We have opened the source code of CUP at this link \url{https://github.com/zmsn-2077/CUP-safe-rl}.
\end{abstract}

\section{Introduction}
\label{sec:intro}

Reinforcement learning (RL) \citep{sutton1998reinforcement} has achieved significant successes in many fields (e.g., \citep{mnih2015human,silver2017mastering,openaifive2019,afsar2021reinforcement,yang2022policy}).
However, most RL algorithms improve the performance under the assumption that an agent is free to explore any behaviors.
In real-world applications, only considering return maximization is not enough, and we also need to consider safe behaviors.
For example, a robot agent should avoid playing actions that irrevocably harm its hardware, and a recommender system should avoid presenting offending items to users.
Thus, it is crucial to consider \emph{safe exploration} for RL, which is usually formulated as constrained Markov decision processes (CMDP) \citep{altman1999constrained}.

It is challenging to solve CMDP since traditional approaches (e.g., Q-learning \citep{watkins1989learning} \& policy gradient \citep{williams1992simple}) usually violate the safe exploration constraints, which is undesirable for safe RL.
Recently, \cite{AchiamHTA17,yang2020projection,bharadhwaj2021conservative} suggest to use some surrogate functions to replace the objective and constraints.
However, their implementations involve some convex approximations to the non-convex objective and safe constraints,
which leads to many error sources and troubles.
Concretely, \cite{AchiamHTA17,yang2020projection,bharadhwaj2021conservative} approximate the non-convex objective (or constraints) with first-order or second Taylor expansion,
but their implementations still lack a theory to show the error difference between the original objective (or constraints) and its convex approximations.
Besides, their approaches involve the inverse of a high-dimension inverse Fisher information matrix, which causes their algorithms require a costly computation for each update when solving high-dimensional RL problems.

To address the above problems, we propose the \emph{constrained update projection} (CUP) algorithm with a theoretical safety guarantee.
We derive the CUP bases on the newly proposed surrogate functions with respect to objectives and safety constraints, and provide a practical implementation of CUP that does not depend on any convex approximation to adapt high-dimensional safe RL.

Concretely, in Section \ref{sec:generalized-bound}, Theorem \ref{them:general-performance-difference} shows generalized difference bounds between two arbitrary policies for the objective and constraints.
Those bounds provide principled approximations to the objective and constraints, which are theoretical foundations for us to use those bounds as surrogate functions to replace 
objective and constraints to design algorithms.
Although using difference bounds as surrogate functions to replace the objective has appeared in previous works (e.g., \citep{schulman2015trust,AchiamHTA17}), Theorem \ref{them:general-performance-difference} refines those bounds (or surrogate functions) at least two aspects:
\textbf{(i)} Firstly, our rigorous theoretical analysis shows a bound with respect to generalized advantage estimator (GAE) \citep{schulman2016high}.
GAE significantly reduces variance empirically while maintaining a tolerable level of bias, the proposed bound involves GAE is one of the critical steps for us to design efficient algorithms.
\textbf{(ii)} Our new bounds unify the classic result of CPO \cite{AchiamHTA17}, i.e., the classic performance bound of CPO is a special case of our bounds.
Although existing work (e.g., \cite{zhang2020first,kang2021learning}) has applied the key idea of CPO with GAE to solve safe RL problems, their approaches are all empirical and lack a theoretical analysis. Thus, the proposed newly bound partially explains the effectiveness of the above safe RL algorithms.
Finally, we should emphasize that although GAE has been empirically applied to extensive reinforcement learning tasks, this work is the first to show a rigorous theoretical analysis to extend the surrogate functions with respect to GAE.

In Section \ref{sec:algorithm}, we provide the necessary details of the proposed CUP.
The CUP contains two steps: it first performs a policy improvement, which may produce a temporary policy violates the constraint.
Then in the second step, CUP projects the policy back onto the safe region to reconcile the constraint violation.
Theorem \ref{them-re-cost} shows the worst-case performance degradation guarantee and approximate satisfaction of safety constraints of CUP, result shows that with a relatively small  parameter that controls the penalty of the distance between the old policy and current policy, CUP shares a desirable toleration for both policy improvements and safety constraints.
Furthermore, we provide a practical implementation of sample-based CUP. This implementation allows us to use deep neural networks to train a model,
which is an efficient iteration without strongly convex approximation of the objective or constraints (e.g., \citep{AchiamHTA17,yang2020projection}), and it optimizes the policy according to the first-order optimizer.
Finally, extensive high-dimensional experiments on continuous control tasks show the effectiveness of CUP where the agent satisfies safe constraints.

\section{Preliminaries}
\label{Background and Notations}

Reinforcement learning (RL) \citep{sutton1998reinforcement} is often formulated as 
a \emph{Markov decision process} (MDP) \citep{puterman2014markov} that is a tuple $\mathcal{M}=(\mathcal{S},\mathcal{A},\mathbb{P},{r},\rho_0,\gamma)$.
Here $\mathcal{S}$ is state space, $\mathcal{A}$ is action space.
$\mathbb{P}(s^{'}|s,a)$ is probability of state transition from $s$ to $s^{'}$ after playing $a$.
$r(\cdot):\mathcal{S}\times\mathcal{S}\times\mathcal{A}\rightarrow \R$,
and $r(s'|s,a)$ denotes the reward that the agent observes when state transition from $s$ to $s^{'}$ after it plays $a$.
$\rho_{0}(\cdot):\mathcal{S}\rightarrow[0,1]$ is the initial state distribution and $\gamma\in(0,1)$.

A stationary parameterized policy $\policy$ is a probability distribution defined on $\mathcal{S}\times\mathcal{A}$, $\policy(a|s)$ denotes the probability of playing $a$ in state $s$.
We use $\Pi_{\bm{\theta}}$ to denote the set of all stationary policies, where $\Pi_{{{\bm{\theta}}}}=\{\pi_{{{\bm{\theta}}}}:{{\bm{\theta}}}\in\R^{p}\}$, and ${{\bm{\theta}}}$ is a parameter needed to be learned.
Let $\mathbf{P}_{\pi_{\bm \theta}}\in\R^{|\calS|\times|\calS|}$ be a state transition probability matrix, and their components are:
$
\mathbf{P}_{\pi_{\bm \theta}}[s,s'] =\sum_{a\in\mathcal{A}}\pi_{\bm{\theta}}(a|s)\mathbb{P}(s'|s,a)=:\Pro_{\policy}(s^{'}|s),
$
which denotes one-step state transformation probability from $s$ to $s^{'}$ by executing $\policy$.
Let $\tau=\{s_{t}, a_{t}, r_{t+1}\}_{t\ge0}\sim \policy$ be a trajectory generated by $\policy$, 
where $s_{0}\sim\rho_{0}(\cdot)$, $a_{t}\sim\policy(\cdot|s_t)$, $s_{t+1}\sim \mathbb{P}(\cdot|s_{t},a_{t})$, and $r_{t+1}=r(s_{t+1}|s_t,a_t)$.
We use $\mathbb{P}_{\policy}(s_t=s^{'}|s)$ to denote the probability of visiting the state $s^{'}$ after $t$
time steps from the state $s$ by executing $\policy$.
Due to the Markov property in MDP, $\mathbb{P}_{\policy}(s_t=s^{'}|s)$ is $(s,s^{'})$-th component of the matrix $\mathbf{P}^{t}_{\pi_{\bm \theta}}$, i.e.,
$
\mathbb{P}_{\policy}(s_t=s^{'}|s)=\mathbf{P}^{t}_{\pi_{\bm \theta}}[s,s^{'}].
$
Finally, let $d_{\policy}^{s_0}(s)=(1-\gamma)\sum_{t=0}^{\infty}\gamma^{t}\mathbb{P}_{\policy}(s_t=s|s_0)$ be the stationary state distribution of the Markov chain (starting at $s_0$) induced by policy $\policy$.
We define
$
d_{\policy}^{\rho_0}(s)=\mathbb{E}_{s_0\sim\rho_{0}(\cdot)}[d_{\policy}^{s_0}(s)]
$
as the discounted state visitation distribution on initial distribution $\rho_0 (\cdot)$.

The \emph{state value function} of $\policy$ is defined as $V_{\policy}(s) = \mathbb{E}_{\policy}[\sum_{t=0}^{\infty}\gamma^{t}r_{t+1}|s_{0} = s],$
where $\mathbb{E}_{\policy}[\cdot|\cdot]$ denotes a conditional expectation on actions which are selected by $\policy$.
Its \emph{state-action value function} is $Q_{\policy}(s,a) = \mathbb{E}_{\policy}[\sum_{t=0}^{\infty}\gamma^{t}r_{t+1}|s_{0} = s,a_{0}=a]$, 
and advantage function is $A_{\policy}(s,a)=Q_{\policy}(s,a) -V_{\policy}(s)$.
The goal of reinforcement learning is to maximize
$
  \label{J-objectiove}
  J(\policy)=\E_{s\sim \rho_{0}(\cdot)}[V_{\policy}(s)].
$

\subsection{Policy Gradient and Generalized Advantage Estimator (GAE)}

Policy gradient \citep{williams1992simple,sutton2000policy} is widely used to solve policy optimization, which maximizes the expected total reward by repeatedly estimating the gradient $g=\nabla J(\policy)$.
\cite{schulman2016high} summarize several different related expressions for the policy gradient:
\begin{flalign}
g=\nabla J(\policy)=\E\left[
\sum_{t=0}^{\infty}\Psi_{t}\nabla\log\policy(a_t|s_t)
\right],
\end{flalign}
where $\Psi_{t}$ can be total discounted reward of the trajectory, value function, advantage function or temporal difference (TD) error.
As stated by \cite{schulman2016high}, the choice $\Psi_{t}=A(s_t,a_t)$ yields almost the lowest possible variance, which is consistent with the theoretical analysis \citep{greensmith2004variance,wu2018variance}.
Furthermore, \cite{schulman2016high} propose generalized advantage estimator ($\mathtt{GAE}$) $\hat{A}^{\mathtt{GAE}(\gamma,\lambda)}_{t}(s_t,a_t)$ to replace $\Psi_{t}$: for any $\lambda\in[0,1]$,
\begin{flalign}
\hat{A}^{\mathtt{GAE}(\gamma,\lambda)}_{t}(s_t,a_t)=\sum_{\ell=0}^{\infty}(\gamma\lambda)^{\ell}\delta^{V}_{t+\ell},
\end{flalign}
where $\delta^{V}_{t}=r_{t+1}+\gamma V(s_{t+1})-V(s_{t})$ is TD error, and $V(\cdot)$ is an estimator of value function.
GAE is an efficient technique for data efficiency and reliable performance of reinforcement learning.

\subsection{Safe Reinforcement Learning}

Safe RL is often formulated as 
a constrained MDP (CMDP) $\calM\cup\calC$ \citep{altman1999constrained},
which is a standard MDP $\calM$ augmented with an additional constraint set $\calC$.
The set $\calC=\{(c_i,b_i)\}_{i=1}^{m}$,
where $c_i$ are cost functions: $c_i : \calS\times\calA \rightarrow \R$, and limits are $b_i$, $i = 1,\cdot,m$. 
The \emph{cost-return} is defined as:
$J^{c_i}(\policy)=\E_{\policy}\left[\sum_{t=0}^{\infty}\gamma^{t}c_{i}(s_{t},a_{t})\right]$, 
then we define the feasible policy set $\Pi_{\calC}$ as:
$
\Pi_{\calC}=
\cap_{i=1}^{m}
\left\{
\policy\in\Pi_{\bm{\theta}}~~\text{and}~~J^{c_i}(\policy)_\leq b_i
\right\}.
$
The goal of CMDP is to search the optimal policy $\pi_{\star}$:
\begin{flalign}
\label{def:problem-setting}
\pi_{\star}=\arg\max_{\policy\in\Pi_{\calC}} J(\policy).
\end{flalign}
Furthermore, we define value functions, action-value functions, and advantage functions for the auxiliary costs in analogy to $V_{\policy}, Q_{\policy}$, and $A_{\policy}$, 
with $c_i$ replacing $r$ respectively, we denote them as $V^{c_i}_{\policy}, Q^{c_i}_{\policy}$, and $A^{c_i}_{\policy}$.
For example, $V^{c_i}_{\policy}(s) = \mathbb{E}_{\policy}\left[\sum_{t=0}^{\infty}\gamma^{t}c_i(s_{t},a_{t})|s_{0} = s\right]$.
Without loss of generality, we will restrict our discussion to the case of one constraint with a cost function $c$ and upper bound $b$.
Finally, we extend the GAE with respect to auxiliary cost function $c$:
\begin{flalign}
\label{def:gae-cost}
\hat{A}^{\mathtt{GAE}(\gamma,\lambda)}_{C,t}(s_t,a_t)=\sum_{\ell=0}^{\infty}(\gamma\lambda)^{\ell}\delta^{C}_{t+\ell},
\end{flalign}
where $\delta^{C}_{t}=r_{t+1}+\gamma C(s_{t+1})-C(s_{t})$ is TD error, and $C(\cdot)$ is an estimator of cost function $c$.

\section{Generalized Policy Performance Difference Bounds}

\label{sec:generalized-bound}

In this section, we show some generalized policy optimization performance bounds for $J(\policy)$ and $J^{c}(\policy)$. 
The proposed bounds provide some new surrogate functions with respect to the objective and cost function, which are theoretical foundations for us to design efficient algorithms to improve policy performance and satisfy constraints.
Before we present the new performance difference bounds, let us revisit a classic performance difference from \citet{kakade2002approximately},
\begin{flalign}
\label{performance-difference-2002}
J(\pi_{{\bm{\theta}}})-J(\policyy)
=(1-\gamma)^{-1}\E_{s\sim d_{\policy}^{\rho_0}(\cdot),a\sim\policy (\cdot|s)}\left[A_{\policyy}(s,a)\right].
\end{flalign}
Eq.(\ref{performance-difference-2002}) shows a difference between two arbitrary policies $\pi_{\bm\theta}$ and $\pi_{{\bm\theta}^{'}}$ with different parameters $\bm{\theta}$ and $\bm{\theta}^{'}$.
According to (\ref{performance-difference-2002}), we rewrite the policy optimization (\ref{def:problem-setting}) as follows
\begin{flalign}
\label{def:rewriting-problem-setting}
\pi_{\star}=\arg\max_{\policy\in\Pi_{\calC}} \E_{s\sim d_{\policy}^{\rho_0}(\cdot),a\sim\policy (\cdot|s)}\left[A_{\policyy}(s,a)\right].
\end{flalign}
However, Eq.(\ref{performance-difference-2002}) or (\ref{def:rewriting-problem-setting}) is very intractable for sampling-based policy optimization since it requires the data comes from the (unknown) policy $\policy$ that needed to be learned.
In this section, we provide a bound refines the result (\ref{performance-difference-2002}), which provide the sights for surrogate functions to solve problem (\ref{def:problem-setting}).

\subsection{Some Additional Notations}

We use a bold lowercase letter to denote a vector, e.g., $\ba=(a_1,a_2,\cdots,a_n)$, and its $i$-th element $\ba[i]=:a_{i}$.
Let $\varphi(\cdot):\calS\rightarrow\R$ be a function defined on $\calS$, $\delta_t^{\varphi}=r(s_{t+1}|s_t,a_t)+\gamma\varphi(s_{t+1})-\varphi(s_{t})$ is TD error with respect to $\varphi(\cdot)$. 
For two arbitrary policies $\pi_{\bm\theta}$ and $\pi_{{\bm\theta}^{'}}$, we denote $\delta^{\varphi}_{\policy,t}(s)$ as the expectation of TD error, and define $ \Delta_{t}^{\varphi}(\policy,\policyy,s)$ as the difference between $\delta^{\varphi}_{\policy,t}(s)$ and $\delta^{\varphi}_{\policyy,t}(s)$:
\begin{flalign}
\nonumber
\delta^{\varphi}_{\policy,t}(s)=\underset{\begin{subarray}{c} s_t \sim \Pro_{\policy}(\cdot|s)\\ a_{t}\sim{\policy}(\cdot|s_t)\\ s_{t+1}\sim\Pro(\cdot|s_t,a_t) \end{subarray}}\E\left[\delta_t^{\varphi}\right],
\Delta_{t}^{\varphi}(\policy,\policyy,s)=
\underset{\begin{subarray}{c} s_t \sim \Pro_{\policyy}(\cdot|s)\\ a_{t}\sim{\policyy}(\cdot|s_t)\\ s_{t+1}\sim\Pro(\cdot|s_t,a_t) \end{subarray}}
\E\left[\left(\dfrac{\policy(a_t|s_t)}{\policyy(a_t|s_t)}-1\right)\delta_t^{\varphi}\right].
\end{flalign}
Furthermore, we introduce two vectors $\bm{\delta}^{\varphi}_{\policy,t},\bm{\Delta}_{t}^{\varphi}(\policy,\policyy)\in\R^{|\calS|}$, and their components are:
\begin{flalign}
\bm{\delta}^{\varphi}_{\policy,t}[s]=\delta^{\varphi}_{\policy,t}(s),~~~\bm{\Delta}_{t}^{\varphi}(\policy,\policyy)[s]=\Delta_{t}^{\varphi}(\policy,\policyy,s).
\end{flalign}
Let matrix $\mathbf{P}^{(\lambda)}_{\pi_{\bm \theta}}=(1-\gamma\lambda)\sum_{{t}=0}^{\infty}(\gamma\lambda)^{{t}}\bP^{{t}+1}_{\policy}$, 
where $\lambda\in[0,1]$.
It is similar to the normalized discounted distribution $d_{\pi_{\bm {\theta}}}^{\rho_0}(s)$, we extend it to $\lambda$-version and denote it as ${d}_{\pi_{\bm {\theta}}}^{\lambda}(s)$:
\[
{d}_{\pi_{\bm {\theta}}}^{\lambda}(s)=\E_{s_0\sim\rho_{0}(\cdot)}
\left[
(1-\tilde\gamma)\sum_{t=0}^{\infty}{\tilde\gamma}^{t}{\mathbb{P}}^{(\lambda)}_{\pi_{\bm {\theta}}}(s_t=s|s_0)
\right],
\]
where $\tilde{\gamma}=\frac{\gamma(1-\lambda)}{1-\gamma\lambda}$, the probability $\mathbb{P}^{(\lambda)}_{\pi_{\bm {\theta}}}(s_t=s|s_0)$ is the $(s_0,s)$-th component of the matrix product $\left(\mathbf{P}^{(\lambda)}_{\pi_{\bm \theta}}\right)^{t}$.
Finally, we introduce a vector $\bd_{\pi_{\bm {\theta}}}^{\lambda}\in\R^{|\calS|}$, and its components are: $\bd_{\pi_{\bm {\theta}}}^{\lambda}[s]=d_{\pi_{\bm {\theta}}}^{\lambda}(s).$

\subsection{Main Results}

\begin{theorem}
[Generalized Policy Performance Difference]
\label{them:general-performance-difference}
For any function $\varphi(\cdot):\calS\rightarrow\R$, for two arbitrary policies $\pi_{\bm\theta}$ and $\pi_{{\bm\theta}^{'}}$,  
for any $p,q\in[1,\infty)$ such that $\frac{1}{p}+\frac{1}{q}=1$, we define two error terms:
\begin{flalign}
\label{error-01}
&\epsilon^{\varphi,(\lambda)}_{p,q,t}(\policy,\policyy)=:\|\bd_{\pi_{\bm {\theta}}}^{\lambda}-\bd_{\policyy}^{\lambda}\|_{p}\|{\bm{\delta}}^{\varphi}_{\policy,t}\|_{q},\\
\label{error-02}
L^{\varphi, \pm}_{p,q}(\policy,\policyy)&=:
\dfrac{1}{1-\tilde\gamma}
\sum_{t=0}^{\infty}\gamma^t\lambda^{t}\E_{s\sim{d}_{\policyy}^{\lambda}(\cdot)} \left[
\Delta_{t}^{\varphi}(\policy,\policyy,s) \pm\epsilon^{\varphi,(\lambda)}_{p,q,t}(\policy,\policyy)\right].
\end{flalign}
Then, the following bound with respect to policy performance difference $J(\pi_{\bm \theta})-J(\pi_{{\bm \theta}^{'}})$ holds:
\begin{flalign} 
\label{bound-diff-01}
L^{\varphi,-}_{p,q,}(\policy,\policyy)
\leq J(\pi_{\bm \theta})-J(\pi_{{\bm \theta}^{'}})
\leq
L^{\varphi,+}_{p,q,}(\policy,\policyy)
.
\end{flalign}
\end{theorem}
\begin{proof}
See Appendix \ref{sec:proof-them-01}.
\end{proof}
The bound (\ref{bound-diff-01}) is well-defined, i.e., if $\policy=\policyy$, all the three terms in Eq.(\ref{bound-diff-01}) are zero identically.
From Eq.(\ref{error-02}), we know the performance difference bound
$L^{\varphi, \pm}_{p,q}(\policy,\policyy)$ (\ref{bound-diff-01}) can be interpreted by two distinct difference parts:
\textbf{(i)} the first difference part, i.e., the expectation $\Delta_{t}^{\varphi}(\policy,\policyy,s)$, which is determined by the difference between TD errors of $\policy$ and $\policyy$;
\textbf{(ii)} the second difference part, i.e., the discounted distribution difference $\epsilon^{\varphi,(\lambda)}_{p,q,t}(\policy,\policyy)$, which is determined by the gap between the normalized discounted distribution of $\policy$ and $\policyy$.
Thus, the difference of both TD errors and discounted distribution determine the policy difference $J(\policy)-J(\policyy)$.

The different choices of $p$ and $q$ lead Eq.(\ref{bound-diff-01}) to be different bounds.
If $p=1,q=\infty$, we denote $\epsilon^{\varphi}_{\policy,t}=:\|{\bm{\delta}}^{\varphi}_{\policy,t}\|_{q}=\max_{s_{t}\in\calS}\E_{a_t\sim\policy(\cdot|s_t),s_{t+1}\sim\Pro(\cdot|s_t,a_t)}[|\delta_{t}^{\varphi}|]$,
then,
according to Lemma \ref{lem:difference-distri} (see Appendix \ref{sec:difference-distri}), when $p=1,q=\infty$, then error $\epsilon^{\varphi,(\lambda)}_{p,q,t}(\policy,\policyy)$ is reduced to:
\begin{flalign}
\nonumber
\epsilon^{\varphi,(\lambda)}_{p,q,t}(\policy,\policyy)\big|_{p=1,q=\infty}\leq
\dfrac{\tilde{\gamma}\left(\gamma\lambda(|\calS|-1)+1\right)\epsilon^{\varphi}_{\policy,t}}{(1-\tilde{\gamma})(1-\gamma\lambda)}\E_{s\sim d_{\pi_{\bm {\theta}}}^{\lambda}(\cdot)}\left[2D_{\mathrm{TV}}(\policyy,\policy)[s]\right],
\end{flalign}
where $D_{\mathrm{TV}}(\policyy,\policy)[s]$ is the total variational divergence between action distributions at state $s$, i.e.,
\[
2D_{\mathrm{TV}}(\policyy,\policy)[s]=\sum_{a\in\calA}\left|{{\policyy}}(a|s)-{{\policy}}(a|s)\right|.
\]
Finally, let $\varphi=V_{\policyy}$, the left side of (\ref{bound-diff-01}) in Theorem \ref{them:general-performance-difference} implies a lower bound of performance difference, which illustrates the worse case of approximation error, we present it in Proposition \ref{propo-01}.
\begin{proposition}
\label{propo-01}
For any two policies $\pi_{\bm\theta}$ and $\pi_{{\bm\theta}^{'}}$, let $\epsilon^{V}_{\policy}(\policyy)=:\sup_{t\in\N^{+}}\{\epsilon^{\varphi}_{\policy,t}: \varphi=V_{\policyy}\}$, then
\begin{flalign}
\nonumber
&J(\policy)-J(\policyy)
\ge\frac{1}{1-\tilde\gamma}\E_{s\sim{d}_{\policyy}^{\lambda}(\cdot),a\sim\policy(\cdot|s)}
\Bigg[A^{\mathtt{GAE}(\gamma,\lambda)}_{\policyy}(s,a)\\
\label{pro1-bound-01}
&~~~~~~~~~~~~~~~~~~~~~~~~~~~~~~~~~~~~~~~~~~~~~~~~~~~~~~~~~~-\dfrac{2\tilde{\gamma}\left(\gamma\lambda(|\calS|-1)+1\right)\epsilon^{V}_{\policy}(\policyy)}{(1-\tilde\gamma)(1-\gamma\lambda)}D_{\mathrm{TV}}(\policyy,\policy)[s]\Bigg].
\end{flalign}
\end{proposition}
The refined bound (\ref{pro1-bound-01}) contains $\mathtt{GAE}$ technique that significantly reduces variance while maintaining a tolerable level of bias empirically \citep{schulman2016high}, which implies using the bound (\ref{pro1-bound-01}) as a surrogate function could improve performance potentially for practice.
Although $\mathtt{GAE}$ has been empirically applied to extensive reinforcement learning tasks, to the best of our knowledge, the result (\ref{pro1-bound-01}) is the first to show a rigorous theoretical analysis to extend the surrogate functions to $\mathtt{GAE}$.
\begin{remark}[Unification of \citep{AchiamHTA17}]
If $\lambda\rightarrow0$, then the distribution ${d}_{\policyy}^{\lambda}(\cdot)$ is reduced to ${d}_{\policyy}^{\rho_0}(\cdot)$ and the bound (\ref{pro1-bound-01}) is reduced to 
\begin{flalign}
\label{bound-lam-0}
J(\policy)-J(\policyy)
\ge
\frac{1}{1-\gamma}\E_{s\sim{d}_{\policyy}^{\rho_0}(\cdot),a\sim\policy(\cdot|s)}
\left[
A_{\policyy}(s,a)
-2\dfrac{\gamma}{1-\gamma}\epsilon^{V}_{\policy}(\policyy)
D_{\mathrm{TV}}(\policyy,\policy)[s]
\right],
\end{flalign}
which matches the result of \citep[Corollary 1]{AchiamHTA17}. That is to say the proposed bound (\ref{pro1-bound-01}) unifies the classic bound (\ref{bound-lam-0})
\end{remark}
Let $\varphi=V^{c}_{\policyy}$, Theorem \ref{them:general-performance-difference} implies an upper bound of cost function as presented in the next Proposition \ref{pro-02}, we will use it to make guarantee for safe policy optimization.
\begin{proposition}
\label{pro-02}
For any two policies $\pi_{\bm\theta}$ and $\pi_{{\bm\theta}^{'}}$, let $\epsilon^{C}_{\policy}(\policyy)=:\sup_{t\in\N^{+}}\{\epsilon^{\varphi}_{\policy,t}: \varphi=V^{c}_{\policyy}\}$, then
\begin{flalign}
\nonumber
&J^{c}(\policy)-J^{c}(\policyy)\leq
\nonumber
\dfrac{1}{1-\tilde\gamma}\E_{s\sim{d}_{\policyy}^{\lambda}(\cdot),a\sim\policy(\cdot|s)}
\Bigg[
A^{\mathtt{GAE}(\gamma,\lambda)}_{\policyy,C}(s,a)\\
\label{pro2-bound-02}
&~~~~~~~~~~~~~~~~~~~~~~~~~~~~~~~~~~~~~~~~~~~~~~~~~~~~~~~~~~+\dfrac{2\tilde{\gamma}\left(\gamma\lambda(|\calS|-1)+1\right)\epsilon^{C}_{\policy}(\policyy)}{(1-\tilde\gamma)(1-\gamma\lambda)}D_{\mathrm{TV}}(\policyy,\policy)[s]\Bigg].
\end{flalign}
where we calculate $A^{\mathtt{GAE}(\gamma,\lambda)}_{\policyy,C}(s,a)$ according to the data sampled from $\policyy$ and the estimator (\ref{def:gae-cost}).
\end{proposition}


All above bound results (\ref{pro1-bound-01}) and (\ref{pro2-bound-02}) can be extended for a total variational divergence to KL-divergence between policies, which are desirable for policy optimization.
\begin{proposition}
\label{propo-03}
All the bounds in (\ref{pro1-bound-01}) and (\ref{pro2-bound-02}) hold if we make the following substitution:
\[
\E_{s\sim{d}_{\policyy}^{\lambda}(\cdot)}\left[D_{\mathrm{TV}}(\policyy,\policy)[s]\right]
\leftarrow
\sqrt{\frac{1}{2}\E_{s\sim{d}_{\policyy}^{\lambda}(\cdot)}\left[\emph{KL}(\policyy,\policy)[s]\right]},
\]
\end{proposition}
where $\mathrm{KL}(\cdot,\cdot)$ is KL-divergence, and $\mathrm{KL}(\policyy,\policy)[s]=\mathrm{KL}(\policyy(\cdot|s),\policy(\cdot|s))$.

\section{CUP: Constrained Update Projection}
\label{sec:algorithm}

It is challenging to implement CMDP (\ref{def:problem-setting}) directly since it requires us to judge whether a proposed policy $\policy$ is in the feasible region $\Pi_{\calC}$.
According to the bounds in Proposition \ref{propo-01}-\ref{propo-03}, we develop new surrogate functions to replace the objective and constraints.
We propose the CUP (constrained update projection) algorithm that is a two-step approach contains \emph{performance improvement} and \emph{projection}.
Due to the limitation of space, we present all the details of the implementation in Appendix \ref{sec-app-cpu} and Algorithm \ref{alg-app-cpu}.

\subsection{Algorithm}

\textbf{Step 1: Performance Improvement.}
According to Proposition \ref{propo-01} and Proposition \ref{propo-03}, for an appropriate coefficient $\alpha_k$, we update policy as:
\begin{flalign}
\label{performance-improvement}
\pi_{{\bm{\theta}}_{k+\frac{1}{2}}}&=\arg\max_{\pi_{{\bm{\theta}}}\in\Pi_{{\bm{\theta}}}}
\left\{
\underset{\begin{subarray}{c}s \sim{d}_{\pi_{\bm{\theta}_k}}^{\lambda}(\cdot)\\a\sim\pi_{\bm{\theta}_k}(\cdot|s)\end{subarray}}
\E\left[\frac{\pi_{\bm{\theta}}(a|s)}{\pi_{{\bm{\theta}}_k}(a|s)}
A^{\mathtt{GAE}(\gamma,\lambda)}_{{\pi_{\bm{\theta}_k}}}(s,a)\right]-\alpha_k
\sqrt{
\E_{s\sim{d}_{{\pi_{\bm{\theta}_k}}}^{\lambda}(\cdot)}\left[\mathrm{KL}(\pi_{\bm{\theta}_k},\policy)[s]\right]}
\right\}.
\end{flalign}
This step is a typical minimization-maximization (MM) algorithm \citep{hunter2004tutorial}, it includes return maximization and minimization the distance between old policy and new policy.
the expectation (\ref{performance-improvement}) by sample averages according to the trajectories collected by $\pi_{\bm{\theta}_k}$.

\textbf{Step 2: Projection.} According to Proposition \ref{pro-02} and Proposition \ref{propo-03}, for an appropriate coefficient $\beta_k$, we project the policy $\pi_{{\bm{\theta}}_{k+\frac{1}{2}}}$ onto the safe constraint set,
\begin{flalign}
\label{projection}
\pi_{{\bm{\theta}}_{k+1}}=\arg\min_{\pi_{{\bm{\theta}}}\in\Pi_{{\bm{\theta}}}}~D\left(\pi_{{\bm{\theta}}},\pi_{{\bm{\theta}}_{k+\frac{1}{2}}}\right),~\text{s.t.}~C_{\pi_{\bm{\theta}_k}}(\policy,\beta_k)\leq b,
\end{flalign}
where $D(\cdot,\cdot)$ (e.g., KL divergence or $\ell_2$-norm) measures distance between $\pi_{{\bm{\theta}}_{k+\frac{1}{2}}}$ and $\policy$,
\[C_{\pi_{\bm{\theta}_k}}(\policy,\beta)=J^{c}(\pi_{{\bm{\theta}}_k})+\frac{1}{1-\tilde\gamma}\E_{s\sim{d}_{\pi_{\bm{\theta}_k}}^{\lambda}(\cdot),a\sim\policy(\cdot|s)}
\left[
A^{\mathtt{GAE}(\gamma,\lambda)}_{{\pi_{\bm{\theta}}},C}(s,a)\right]+\beta
\sqrt{
\E_{s\sim{d}_{{\pi_{\bm{\theta}_k}}}^{\lambda}(\cdot)}\left[\mathrm{KL}(\pi_{\bm{\theta}_k},\policy)[s]\right]}.\]

Until now, the particular choice of surrogate function is heuristically motivated, we show the worst-case performance degradation guarantee and approximate satisfaction of safety constraints of CUP in Theorem \ref{them-re-cost}, and its proof is shown in Appendix \ref{sec:app-them2}.
\begin{theorem}
\label{them-re-cost}
Let $\chi_k=\E_{s\sim{d}_{{\pi_{\bm{\theta}_k}}}^{\lambda}(\cdot)}[\emph{KL}(\pi_{\bm{\theta}_k},\pi_{\bm{\theta}_{k+\frac{1}{2}}})[s]]$, $\iota=:\frac{\tilde{\gamma}\left(\gamma\lambda(|\calS|-1)+1\right)}{(1-\tilde{\gamma})(1-\gamma\lambda)}$.If $\pi_{\bm{\theta}_k}$ and 
$\pi_{\bm{\theta}_{k+1}}$ are generated according to (\ref{performance-improvement})-(\ref{projection}),
then the lower bound on policy improvement, and upper bound on constraint violation are
\begin{flalign}
\nonumber
J(\pi_{\bm{\theta}_{k+1}})-J(\pi_{\bm{\theta}_{k}})\ge&-\dfrac{\iota\alpha_{k}\sqrt{2\chi_{k}}}{1-\tilde\gamma}\epsilon^{V}_{\pi_{\bm{\theta}_{k+1}}}(\pi_{\bm{\theta}_k}),~~~
J^{c}(\pi_{\bm{\theta}_{k+1}})\leq b+\dfrac{\iota\beta_{k}\sqrt{2\chi_{k}}}{1-\tilde\gamma}\epsilon^{C}_{\pi_{\bm{\theta}_{k+1}}}(\pi_{\bm{\theta}_k}).
\end{flalign}
\end{theorem}
\begin{remark}[Asymptotic Safety Guarantee]
\label{remark:asymptotic-safety}
Let $\alpha_{k}\rightarrow 0,\beta_{k}\rightarrow 0$ as $k\rightarrow\infty$, Theorem \ref{them-re-cost} implies a monotonic policy improvement with an asymptotic safety guarantee.
\end{remark}

\subsection{Practical Implementation}
Now, we present our sample-based implementation for CUP (\ref{performance-improvement})-(\ref{projection}).
Our main idea is to estimate the objective and constraints in (\ref{performance-improvement})-(\ref{projection}) with samples collected by current policy $\pi_{\bm{\theta}_k}$, then solving its optimization problem via first-order optimizer.

Let $\{(s_t,a_t,r_{t+1},c_{t+1})\}_{t=1}^{T}\sim\pi_{\bm{\theta}_k}$, we denote the empirical KL-divergence with respect to $\policy$ and $\policyy$ as:
\[\hat{D}_{\mathrm{KL}}(\policy,\policyy)=\frac{1}{T}\sum_{t=1}^{T}\mathrm{KL}(\policy(a_t|s_t),\policyy(a_t|s_t)).\]
We update performance improvement (\ref{performance-improvement}) step as follows,
\begin{flalign}
\nonumber
\pi_{{\bm{\theta}}_{k+\frac{1}{2}}}&=\arg\max_{\pi_{\bm{\theta}}\in\Pi_{\theta}}\left\{\dfrac{1}{T}\sum_{t=1}^{T}\dfrac{\pi_{\bm{\theta}}(a_t|s_t)}{\pi_{{\bm{\theta}}_k}(a_t|s_t)}\hat{A}_t-\alpha_k\sqrt{\hat{D}_{\mathrm{KL}}(\pi_{{\bm{\theta}}_k},\pi_{{\bm{\theta}}})}\right\},
\end{flalign}
where $\hat{A}_t$ is an estimator of $A^{\mathtt{GAE}(\gamma,\lambda)}_{\pi_{{\bm{\theta}}_k}}(s,a)$.

Then we update the projection step by replacing the distance $D$ by KL-divergence, the next Theorem \ref{min-max-main} (for its proof, see Appendix \ref{app-sec-practical-ip}) provides a fundamental way for us to solve projection step (\ref{projection}).
\begin{theorem}
\label{min-max-main}
The constrained problem (\ref{eq:projection}) is equivalent to the following primal-dual problem:
\begin{flalign}
\nonumber
\max_{\nu\ge0}\min_{\pi_{\bm{\theta}}\in\Pi_{\bm{\theta}}}
\left\{
D\left(\pi_{{\bm{\theta}}},\pi_{{\bm{\theta}}_{k+\frac{1}{2}}}\right)+\nu
\left(
C_{\pi_{\bm{\theta}_k}}(\policy,\beta)-b
\right)
\right\}.
\end{flalign}
\end{theorem}
According to Theorem \ref{min-max-main}, we solve the constraint problem (\ref{projection}) by the following primal-dual approach, 
\begin{flalign}
\nonumber
(\pi_{{\bm{\theta}}_{k+1}},\nu_{k+1})=\arg\min_{\pi_{\bm{\theta}}\in\Pi_{\bm{\theta}}}\max_{\nu\ge0}
\left\{
\hat{D}_{\mathrm{KL}}(\pi_{{\bm{\theta}}_{k+\frac{1}{2}}},\pi_{{\bm{\theta}}})+\nu \hat{C}(\policy,\pi_{{\bm{\theta}}_k})
\right\}
\end{flalign}
where
$\hat{C}(\policy,\pi_{{\bm{\theta}}_k})=\hat{J}^{C}+\frac{1}{1-\tilde\gamma}\cdot\frac{1}{T}\sum_{t=1}^{T}\frac{\pi_{\bm{\theta}}(a_{t}|s_{t})}{\pi_{{\bm{\theta}}_k}(a_{t}|s_{t})}\hat{A}^{C}_{t}+\beta_k\sqrt{\hat{D}_{\mathrm{KL}}(\pi_{{\bm{\theta}}_k},\pi_{{\bm{\theta}}})}-b,$ $\hat{J}^{C}$and $\hat{A}^{C}_{t}$ are estimators for cost-return and cost-advantage.

Finally, let 
\begin{flalign}
\hat{\mathcal{L}}_{\text{c}}\left(\policy,\pi_{{\bm{\theta}}_k},{\bm{\theta}}_{k+\frac{1}{2}},\nu\right)=:\hat{D}_{\mathrm{KL}}(\pi_{{\bm{\theta}}_{k+\frac{1}{2}}},\pi_{{\bm{\theta}}})+\nu \hat{C}(\policy,\pi_{{\bm{\theta}}_k}),
\end{flalign}
we update the parameters $({\bm{\theta}}_{k+1},\nu_{k+1})$ as follows,
\begin{flalign}
\bm{\theta}_{k+1}&\leftarrow \bm{\theta}_{k}-\eta\dfrac{\partial}{\partial \bm{\theta}}\hat{\mathcal{L}}_{\text{c}}\left(\policy,\pi_{{\bm{\theta}}_k},{\bm{\theta}}_{k+\frac{1}{2}},\nu\right)\Big|_{\bm{\theta}=\bm{\theta}_{k},\nu=\nu_k},
\\
\nu_{k+1}&\leftarrow \left\{\nu_{k}+\eta\dfrac{\partial}{\partial \nu}\hat{\mathcal{L}}_{\text{c}}\left(\policy,\pi_{{\bm{\theta}}_k},{\bm{\theta}}_{k+\frac{1}{2}},\nu\right)\Big|_{\bm{\theta}=\bm{\theta}_{k},\nu=\nu_k}\right\}_{+},
\end{flalign}
where $\eta>0$ is step-size, $\{\cdot\}_{+}$ denotes the positive part, i.e., if $x\leq0$, $\{x\}_{+}=0$, else $\{x\}_{+}=x$. We have shown all the details of the implementation in Algorithm \ref{alg-app-cpu}.

\section{Related Work}

Due to the limitation of space, for more discussions and comparisons, see Appendix \ref{app-related-work} and Table \ref{app-table-com}.

\subsection{Local Policy Search and Lagrangian Approach}

A direct way to solve CMDP (\ref{def:problem-setting}) is to apply \emph{local policy search} \citep{peter2008Reinforcement,2013Safepolicy} over the policy space $\Pi_{\calC}$, i.e.,
\begin{flalign}
\label{local-policy-search}
\pi_{{{\bm{\theta}}}_{k+1}}=\arg\max_{\pi_{{\bm{\theta}}}\in\Pi_{{\bm{\theta}}}}J(\pi_{{\bm{\theta}}}),~\text{s.t}.~J^{c}(\pi_{{\bm{\theta}}})\leq b,~\text{and}~D(\pi_{{\bm{\theta}}},\pi_{{{\bm{\theta}}}_{k}})<\delta,
\end{flalign}
where $\delta$ is a positive scalar, $D(\cdot,\cdot)$ is some distance measure.
For practice, the local policy search (\ref{local-policy-search}) is challenging to implement because it requires evaluation of the constraint function $c$ to determine whether a proposed point $\pi$ is feasible \citep{zhang2020first}. 
Besides, \cite{li2019temporal,end2end2019,Liu_Ding_Liu_2020} provide a local policy search via the barrier function.
The key idea of the proposed CUP is parallel to Barrier functions.
When updating policy according to samples, local policy search (\ref{local-policy-search}) requires off-policy evaluation \citep{AchiamHTA17}, which is very challenging for high-dimension control problem \citep{duan2016benchmarking,yang-ijcai2018-414,yang2020convergence}.

A way to solve CMDP (\ref{def:problem-setting}) is Lagrangian approach that is also known as primal-dual problem:
\begin{flalign}
\label{min-max-search}
(\pi_{\star},\lambda_{\star})=\arg\min_{\lambda\ge0}\max_{\policy\in\Pi_{\bm{\theta}}}
\left\{
J(\policy)-\lambda(J^{c}(\policy)-b)
\right\}.
\end{flalign}
Although extensive canonical algorithms are proposed to solve problem (\ref{min-max-search}), e.g., \citep{liang2018accelerated,tessler2019reward,paternain2019constrained,le2019batch,russel2020robust,satija2020constrained,chen2021primal},
the policy updated by Lagrangian approach may be infeasible w.r.t. CMDP (\ref{def:problem-setting}).
This is hazardous in reinforcement learning when one needs to execute the intermediate policy (which may be unsafe) during training \citep{chow2018lyapunov}.

\textbf{Constrained Policy Optimization (CPO)}. Recently, \cite{AchiamHTA17} suggest to replace the cost constraint with a surrogate cost function which evaluates the constraint $J^{c}(\pi_{{\bm{\theta}}})$ according to the samples collected from the current policy $\pi_{{\bm{\theta}}_k}$, see Eq.(\ref{cpo-objective})-(\ref{trust-region}).
For a given policy $\pi_{{\bm{\theta}}_{k}}$, CPO \citep{AchiamHTA17} updates new policy $\pi_{{\bm{\theta}}_{k+1}}$ as follows:
\begin{flalign}
\label{cpo-objective}
&\pi_{{\bm{\theta}}_{k+1}}=\arg\max_{\pi_{{\bm{\theta}}}\in\Pi_{{\bm{\theta}}}}~~~~\E_{s\sim d^{\rho_0}_{\pi_{{\bm{\theta}}_k}}(\cdot),a\sim\pi_{{\bm{\theta}}}(\cdot|s)}\left[A_{\pi_{{\bm{\theta}}_k}}(s,a)\right]\\
\label{cost-constraint}
&~~~~~~~~\text{s.t.}~~J^{c}(\pi_{{\bm{\theta}}_k})+\frac{1}{1-\gamma}\E_{s\sim d^{\rho_0}_{\pi_{{\bm{\theta}}_k}}(\cdot),a\sim\pi_{{\bm{\theta}}}(\cdot|s)}\left[A^{c}_{\pi_{{\bm{\theta}}_k}}(s,a)\right]\leq b,\\
\label{trust-region}
&~~~~~~~~\bar{D}_{\text{KL}}(\pi_{{\bm{\theta}}},\pi_{{\bm{\theta}}_k})=\E_{s\sim d^{\rho_0}_{\pi_{{\bm{\theta}}_k}}(\cdot)}[\text{KL}(\pi_{{\bm{\theta}}},\pi_{{\bm{\theta}}_k})[s]]\leq\delta.
\end{flalign}
Existing recent works (e.g., \citep{AchiamHTA17,vuong2019supervised,yang2020projection,hanreinforcementl2020,ijcai2020-632,bharadhwaj2021conservative}) try to find some convex approximations to replace the term $A_{\pi_{{\bm{\theta}}_k}}(s,a)$ and $\bar{D}_{\text{KL}}(\pi_{{\bm{\theta}}},\pi_{{\bm{\theta}}_k})$ in Eq.(\ref{app-cpo-objective})-(\ref{app-trust-region}).
Such first-order and second-order approximations turn a non-convex problem (\ref{app-cpo-objective})-(\ref{app-trust-region}) to be a convex problem, 
it seems to make a simple solution, but this approach results in many error sources and troubles in practice.
Firstly, it still lacks a theory analysis to show the difference between the non-convex problem (\ref{app-cpo-objective})-(\ref{app-trust-region}) and its convex approximation.
Policy optimization is a typical non-convex problem \citep{yang2021sample}; its convex approximation may introduce some error for its original issue.
Secondly, CPO updates parameters according to conjugate gradient \citep{suli2003introduction}, and its solution involves the inverse Fisher information matrix,
which requires expensive computation for each update.

Instead of using a convex approximation for the objective function, the proposed CUP algorithm improves CPO and PCPO at least two aspects.
Firstly, the CUP directly optimizes the surrogate objective function via the first-order method, and it does not depend on any convex approximation.
Thus, the CUP effectively avoids the expensive computation for the inverse Fisher information matrix.
Secondly, CUP extends the surrogate objective function to GAE. Although \cite{zhang2020first} has used the GAE technique in experiments, to the best of our knowledge, it still lacks a rigorous theoretical analysis involved GAE before we propose CUP.

\section{Experiment}
In this section, we aim to answer the following three issues: 
\begin{itemize}
\item Does CUP satisfy the safety constraints in different environments? Does CUP performs well with different cost limits? 
\item How does CUP compare to the state-of-the-art safe RL algorithms?
\item Does CUP play a sensibility during the hyper-parameters in the tuning processing?
\end{itemize}

We train different robotic agents using five MuJoCo physical simulators \citep{todorov2012mujoco} which are open by OpenAI Gym API \citep{brockman2016openai}, and Safety Gym \citep{Ray2019}. For more details, see Appendix \ref{sec:app-ex-env}.
Baselines includes CPO \citep{AchiamHTA17}, PCPO \citep{yang2020projection}, TRPO Lagrangian (TRPO-L), PPO Lagrangian (PPO-L) and FOCOPS \citep{zhang2020first}. 
TRPO-L and PPO-L are improved by \citep{chow2018lyapunov,Ray2019}, which are based on TRPO \citep{schulman2015trust} and PPO \citep{schulman2017proximal}. 
These two algorithms use the Lagrangian method \citep{bertsekas1997nonlinear}, which applies adaptive penalty coefficients to satisfy the constraint.

\begin{figure*}[t]
  \centering
  \includegraphics[width=14cm, height=4.7cm]{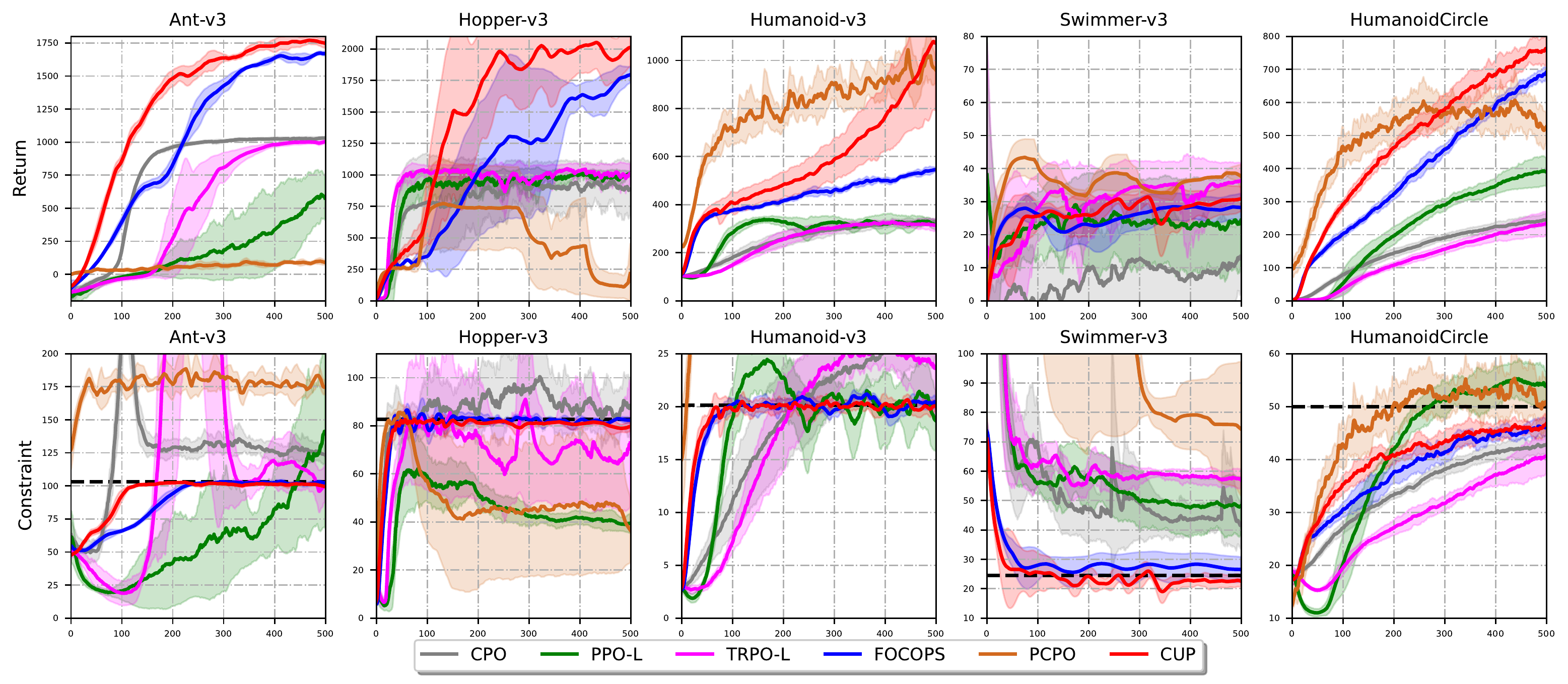}
  \caption{Comparison of CUP to baseline algorithms over 10 seeds on Mujoco.} 
  \label{fig:comparison}
\end{figure*}
\begin{figure*}[t!]
  \centering
  \includegraphics[width=14cm, height=4.7cm]{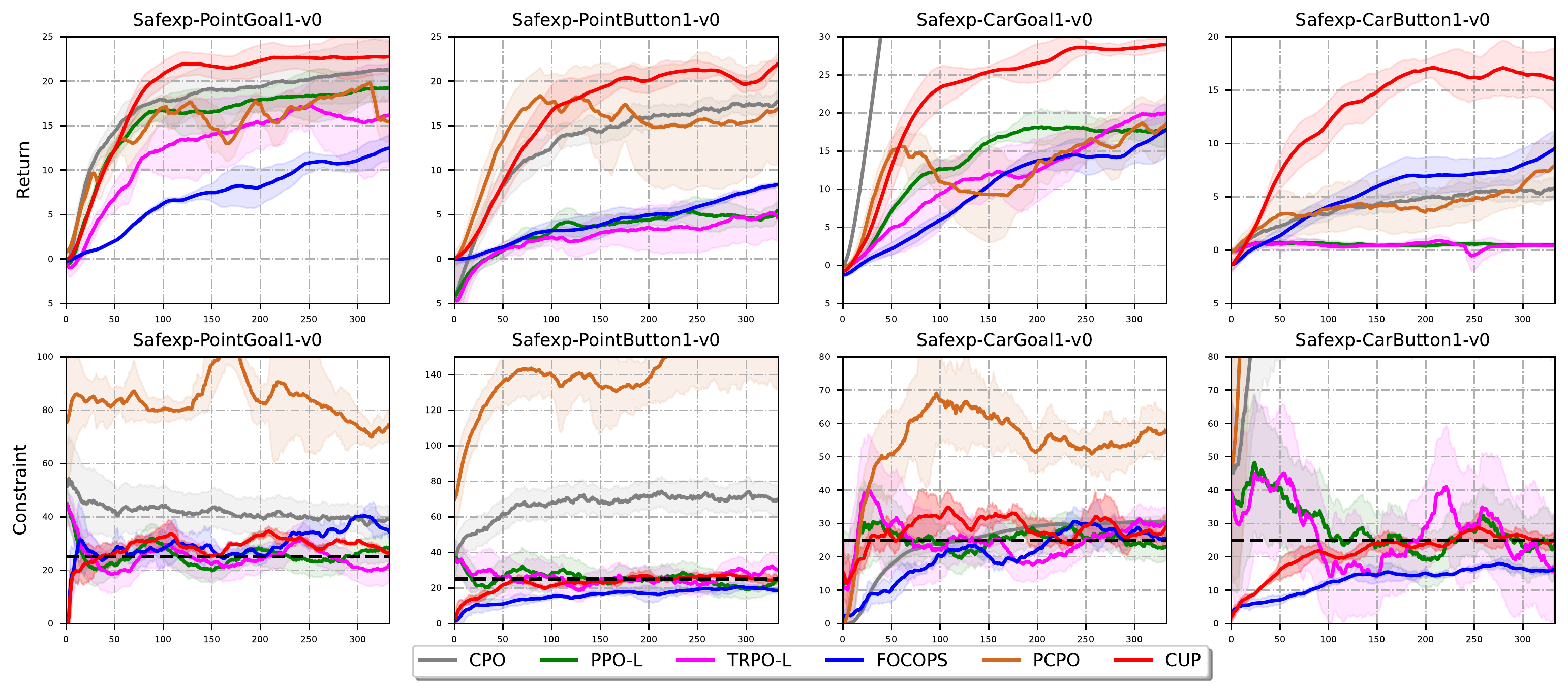}
  \caption{Comparison of CUP to to baseline algorithms over 3 seeds on Satety-Gym.} 
  \label{fig:safety-gym-comparison}
\end{figure*}

\subsection{Evaluation CUP and Comparison Analysis}

We have shown the Learning curves for CUP, and other baselines in Figure \ref{fig:comparison}-\ref{fig:safety-gym-comparison}, and Table \ref{tab:mujoco} summarizes the performance of all algorithms.
Results show that CUP quickly stabilizes the constraint return around the limit value while converging the objective return to higher values faster.
In most cases, the traces of constraint from CUP almost coincide with the dashed black line of the limit. 
By contrast, the baseline algorithms frequently suffer from over or under the correction. 

From Figure \ref{fig:comparison}, we know initial policies of the baseline algorithms are not guaranteed to be feasible, such as in Swimmer-v3, while CUP performs the best and keeps safety learning in Swimmer-v3 tasks. In the HumanoidCircle task, all the algorithms learn steadily to obtain a safe policy, except PPO-L.
Additionally, we observed that CUP brings the policy back to the feasible range faster than other baselines in the HumanoidCircle task.
In the Ant-v3 task, only the FOCOPS and the proposed CUP learn safely, and both CPO and TRPO-L violate the safety constraints significantly. Besides, although FOCOPS and CUP converge to a safe policy, CUP obtains a better reward performance than FOCOPS in the Ant-v3 task.
 The result of Figure \ref{fig:safety-gym-comparison} is relatively complex, the initial policies of the CPO and PCPO are not guaranteed to be feasible on both Safexp-PointGoal1-v0 and Safexp-PointButton1-v0. 
We think it is not accidental, but it partially provides corroboration of the previous discussions in Appendix \ref{app-related-work}.
Both CPO and PCPO use first-order and second-order approximation to approximate a non-convex problem as a convex problem, 
which inevitably produces a significant deviation from the original RL problem, and it is more serious in large-scale and complex control systems.

From Table \ref{tab:mujoco}, we know although PPO-L achieves a reward of $35.58\pm 5.68$ outperforms CUP in Swimmer-v3, PPO-L obtains a cost of $54.91 \pm 3.93$ that violates the cost limit of $24.5$ significantly, which implies PPO-L learns a dangerous policy under this setting. 
On the other hand, Figure \ref{fig:comparison} has shown that CUP generally gains higher returns than different baselines while enforcing the cost constraint.
Mainly, CUP achieves a reward performance of ${2025.56\pm122.35}$ that significantly outperforms all the baseline algorithms.
Additionally, after equal iterations, CUP performs a greater speed of stabilizing the constraint return around the limit value and is quicker to find feasible policies to gain a more significant objective return.

\begin{table*}[t!]
 \centering
  
 \vskip 0.1in
 \begin{adjustbox}{width=1\textwidth}
 \begin{tabular}{cc|c|c|c|c|c|c}
 \toprule
 Environment     & {}       & CPO & TRPO-L & PPO-L &PCPO &FOCOPS & \textbf{CUP} \\
  \hline 
 Ant-v3          &Return   & $1030.17\pm8.15$ & $480.86\pm161.05$ & $1012.02\pm17.26$ & $90.83\pm17.66$ & $1662.53\pm17.40$ & $\boldsymbol{1743.66\pm40.5}$ \\
 cost limit: 103        &Constraint     & ${\color{black}120.76\pm4.80}$ & ${\color{black}131.07\pm67.9}$ & ${\color{black}112.45\pm15.48}$&${\color{black}174.80 \pm 5.53}$ & $101.31\pm0.41$ & $99.11\pm0.93$ \\
 \hline 
 Hopper-v3       &Return   & $875.89\pm285.17$ & $1025.49\pm10.68$ & $1010.2\pm61.48$& $214.90\pm101.22$ & $1687.72\pm24.38$ & $\boldsymbol{2025.56\pm122.35}$ \\
 cost limit: 83         &Constraint     & $76.6\pm10.62$ & $40.36\pm4.75$ & ${\color{black}83.28\pm31.19}$ & $36.63 \pm 12.54$& ${\color{black}102.3\pm1.455}$& $79.98\pm2.306$ \\
 \hline 
  Swimmer-v3     &Return  & $18.77\pm6.56$ & $27.35\pm10.07$ & $\boldsymbol{35.58\pm5.68}$& $37.73\pm3.56$ & $28.15\pm4.30$ & $33.38\pm0.54$ \\
 cost limit: 24.5         &Constraint     & ${\color{black}42.07\pm3.31}$ & ${\color{black}49.58\pm7.46}$ & ${\color{black}{54.91\pm3.93}}$ & ${\color{black}{74.39\pm22.71}}$& ${\color{black}26.54\pm4.16}$& $23.31\pm0.052$ \\
 \hline 
 Humanoid-v3    &Return   & $326.95\pm16.00$ & $307.71\pm24.71$& $322.11\pm25.54$ & $962.13\pm57.94$& $542.5\pm4.76$ & $\boldsymbol{1066.83\pm266.12}$ \\
 cost limit: 20.0         &Constraint     & ${\color{black}26.13\pm2.13}$ & $18.22\pm3.04$ & ${\color{black}22.94\pm4.54}$ & ${\color{black}48.66\pm3.52}$ & $20.04\pm0.19$& $19.91\pm0.36$ \\
 \hline 
 Humanoid-Circle  &Return   & $237.54\pm23.20$ & $384.45\pm47.66$ & $243.35\pm37.90$& $525.23 \pm48.32$& $713.04\pm9.25$ & $\boldsymbol{768.65\pm63.70}$ \\
  cost limit: 50.0          &Constraint     & $43.64\pm1.91$ & ${\color{black}53.77\pm1.48}$ & $41.17\pm3.98$ & $50.80\pm4.57$&$47.73\pm0.64$& $48.23\pm0.65$ \\
 \bottomrule
 \end{tabular}
 \end{adjustbox}
  
  \caption{Average results for baseline algorithms and CUP over 10 seeds the last 500 iterations.}
 \label{tab:mujoco}
\end{table*}
 \begin{figure*}[t!]
 \centering
 \subfigure[Cost constraint with respect to hyper-parameter $\nu$ (defined in Projection step).]{\includegraphics[width=13.5cm,height=2.5cm]{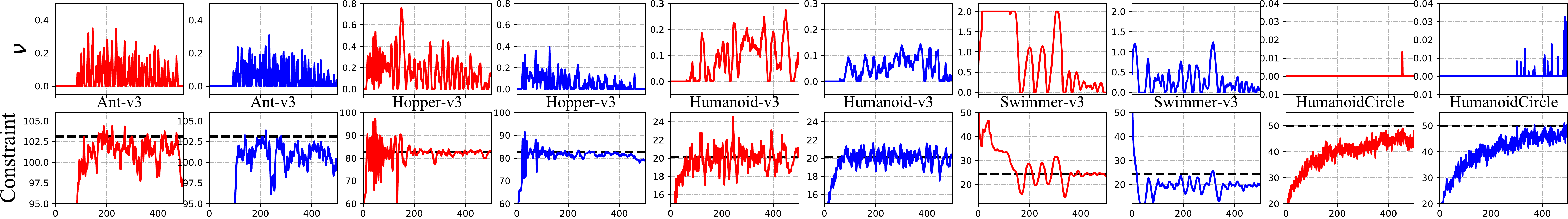}} 
 \subfigure[Performance w.r.t. penalty factor $\alpha$.]{\includegraphics[width=6.8cm,height=2.3cm]{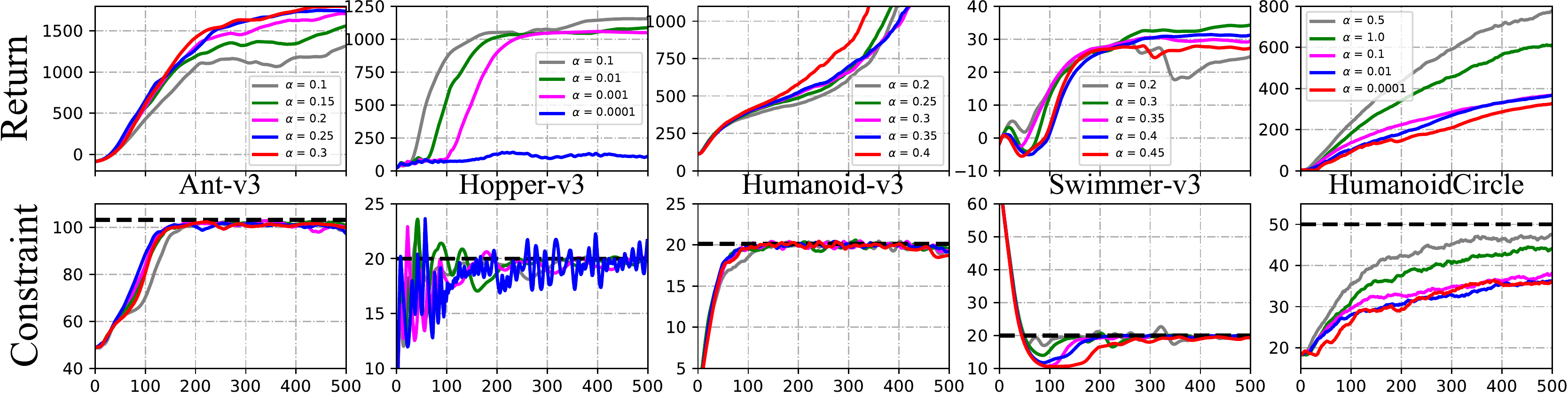}} 
  \subfigure[Performance w.r.t. cost limit.]{\includegraphics[width=6.8cm,height=2.3cm]{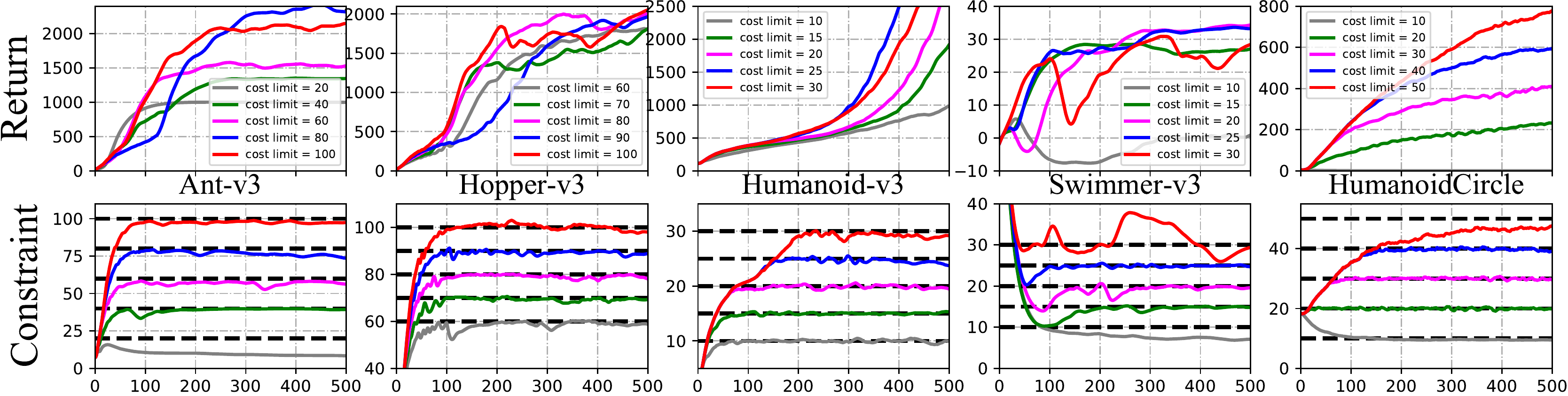}} 
 \caption{Sensitivity analysis for hyper-parameters tuning with respect to $\nu$, $\alpha$ and cost limit.}
 \label{fig-sensitivity-analysis}
\end{figure*}

\subsection{Sensitivity Analysis for Hyper-Parameters Tuning}
Hyper-parameters tuning is necessary to achieve efficient policy improvement and enforce constraints. 
We investigate the performance with respect to the parameters: $\nu$, step-size $\alpha$, and cost limit $b$.

From Figure \ref{fig-sensitivity-analysis} (a), we know if the estimated cost under the target threshold $b$, then $\nu$ keeps calm, which implies $\nu$ is not activated.
Such an empirical phenomenon gives significant expression to the Ant-v3, Humanoid-v3, and Hopper-v3 environments.
While if the estimated cost exceeds the target threshold $b$, $\nu$ will be activated, which requires the agent to play a policy on the safe region.
Those empirical results are consistent with the update rule of $\nu$:
\[
 \nu_{k+1}=\{\nu_{k}+\eta(\hat{J}_{k}^{C}-b)\}_{+},
\]
 which implies the projection of CUP plays an important role for the agent to learn a safe policy.
Additionally, Figure \ref{fig-sensitivity-analysis} (a) provides a visualization way to show the difficulty of different tasks, where the task actives much quantification of $\nu$, such a task is more challenging to obtain a safe policy.
Furthermore, Figure \ref{fig-sensitivity-analysis} (b) shows that the performance of CUP is still very stable for different settings of $\alpha$, and the constraint value of CUP also still fluctuates around the target value. The different value achieved by CUP in different setting $\alpha$ is affected by the simulated environment and constraint thresholds, which are easy to control.
Finally, Figure \ref{fig-sensitivity-analysis} (c) shows that CUP learns safe policies stably under the cost limit thresholds.
We compare policy performance and cost under different cost limit settings. For example, in the Swimmer-v3, we set cost limit $b$ among $\{10,15,20,25,30\}$.
Different cost limit setting implies different difficulty for learning, results show that CUP is scalable to various complex tasks, which means CUP is robust to different cost limit settings for various safe RL tasks.

\section{Conclusion}
This paper proposes the CUP algorithm with a theoretical safety guarantee.
We derive the CUP based on the newly proposed surrogate functions with respect to objectives and constraints and provide a practical implementation of CUP that does not depend on any convex approximation.
We compared CUP against a comprehensive list of safe RL baselines on a wide range of tasks, which shows the effectiveness of CUP where the agent satisfies safe constraints.

\clearpage

\bibliographystyle{named}
\bibliography{example_paper}

\clearpage
\appendix

\section*{Checklist}

The checklist follows the references.  Please
read the checklist guidelines carefully for information on how to answer these
questions.  For each question, change the default \answerTODO{} to \answerYes{},
\answerNo{}, or \answerNA{}.  You are strongly encouraged to include a {\bf
justification to your answer}, either by referencing the appropriate section of
your paper or providing a brief inline description.  For example:
\begin{itemize}
  \item Did you include the license to the code and datasets? \answerYes{See Section~\ref{gen_inst}.}
  \item Did you include the license to the code and datasets? \answerNo{The code and the data are proprietary.}
  \item Did you include the license to the code and datasets? \answerNA{}
\end{itemize}
Please do not modify the questions and only use the provided macros for your
answers.  Note that the Checklist section does not count towards the page
limit.  In your paper, please delete this instructions block and only keep the
Checklist section heading above along with the questions/answers below.

\begin{enumerate}

\item For all authors...
\begin{enumerate}
  \item Do the main claims made in the abstract and introduction accurately reflect the paper's contributions and scope?
    \answerYes{See Abstract and Section \ref{sec:intro}.}
  \item Did you describe the limitations of your work?
    \answerNo{}
  \item Did you discuss any potential negative societal impacts of your work?
    \answerNo{}
  \item Have you read the ethics review guidelines and ensured that your paper conforms to them?
    \answerYes{We ensure our paper to conform to the ethics review guidelines.}
\end{enumerate}

\item If you are including theoretical results...
\begin{enumerate}
  \item Did you state the full set of assumptions of all theoretical results?
    \answerYes{See Section \ref{sec:generalized-bound} and Section \ref{sec:algorithm}.}
        \item Did you include complete proofs of all theoretical results?
    \answerYes{See Appendix \ref{sec:proof-them-01} and Appendix \ref{sec:app-them2}.}
\end{enumerate}

\item If you ran experiments...
\begin{enumerate}
  \item Did you include the code, data, and instructions needed to reproduce the main experimental results (either in the supplemental material or as a URL)?
    \answerYes{See the URL in the supplementary material for the code, and see \ref{sec:app-ex} for environments of experiments.}
  \item Did you specify all the training details (e.g., data splits, hyperparameters, how they were chosen)?
    \answerYes{See Appendix \ref{sec:app-ex}}
        \item Did you report error bars (e.g., with respect to the random seed after running experiments multiple times)?
     \answerYes{See Appendix \ref{sec:app-ex}}
        \item Did you include the total amount of compute and the type of resources used (e.g., type of GPUs, internal cluster, or cloud provider)?
   \answerYes{See Appendix \ref{sec:app-ex}}
\end{enumerate}

\item If you are using existing assets (e.g., code, data, models) or curating/releasing new assets...
\begin{enumerate}
  \item If your work uses existing assets, did you cite the creators?
    \answerYes{}
  \item Did you mention the license of the assets?
    \answerYes{}
  \item Did you include any new assets either in the supplemental material or as a URL?
    \answerYes{We provide the code for our implementation of CUP in the supplemental material.}
  \item Did you discuss whether and how consent was obtained from people whose data you're using/curating?
    \answerYes{We use open source safe reinforcement learning environments, see Appendix \ref{sec:app-ex-env}}
  \item Did you discuss whether the data you are using/curating contains personally identifiable information or offensive content?
    \answerYes{Our data does not contain any personally identifiable information or offensive content.}
\end{enumerate}

\item If you used crowdsourcing or conducted research with human subjects...
\begin{enumerate}
  \item Did you include the full text of instructions given to participants and screenshots, if applicable?
    \answerNA{}
  \item Did you describe any potential participant risks, with links to Institutional Review Board (IRB) approvals, if applicable?
    \answerNA{}
  \item Did you include the estimated hourly wage paid to participants and the total amount spent on participant compensation?
    \answerNA{}
\end{enumerate}

\end{enumerate}

\section{Notations}

\subsection{Matrix Index}
In this paper,
we use a bold capital letter to denote matrix, e.g., $\bA=(a_{i,j})\in\R^{m\times n}$, and its $(i,j)$-th element denoted as \[\bA[i,j]=:a_{i,j},\] where $1\leq i\leq m,1\leq j\leq n$.
Similarly, a bold  lowercase letter denotes a vector, e.g., $\ba=(a_1,a_2,\cdots,a_n)\in\R^{n}$, and  its $i$-th element denoted as \[\ba[i]=:a_{i},\] where $1\leq i\leq n$.

\subsection{Key Notations of Reinforcement Learning}

For convenience of reference, we list key notations that have be used in this paper.

\subsubsection{Value Function and Dynamic System of MDP.}
\begin{tabular}{|r c p{10.5cm}|}
   \hline
   $\mathbf{r}_{\pi_{\bm \theta}},~R_{\policy}(s),$ &: & $\mathbf{r}_{\pi_{\bm \theta}}\in\R^{|\calS|}$ is the expected vector reward according to $\pi_{\bm \theta}$, i.e., their components are: $\mathbf{r}_{\pi_{\bm \theta}}[s] =\sum_{a\in\mathcal{A}}\sum_{s^{'}\in\mathcal{S}}\pi_{\bm{\theta}}(a|s)r(s'|s,a)=:R_{\policy}(s),~s\in\calS.$\\
   $\bv_{\pi_{\bm{\theta}}},~V_{\policy}(s),$ &: & 
   $\bv_{\pi_{\bm{\theta}}}\in\R^{|\calS|}$ is the vector that stores all the state value functions, and its components are:
$
\bv_{\pi_{\bm{\theta}}}[s]=V_{\pi_{\bm{\theta}}}(s),~s\in\calS.
$
\\
   \hline
$\rho(\cdot),\bm{\rho}$ &: & $\rho(s)$: the initial state distribution of state $s$; $\bm{\rho}\in\R^{|\calS|}$, and $\bm{\rho}[s]=\rho(s)$. \\
   \hline
        $\mathbf{P}_{\pi_{\bm \theta}}$ &: & Single-step state transition matrix by executing $\policy$. \\
        $\Pro_{\policy}(s^{'}|s)$ &: & Single-step state transition probability from $s$ to $s^{'}$ by executing $\policy$, and it is the $(s,s^{'})$-th component of the matrix $\mathbf{P}_{\pi_{\bm \theta}}$, i.e., $\mathbf{P}_{\pi_{\bm \theta}}[s,s^{'}]=\Pro_{\policy}(s^{'}|s)$.\\
        $\mathbb{P}_{\policy}(s_t=s^{'}|s)$&: & The probability of visiting the state $s^{'}$ after $t$
time steps from the state $s$ by executing $\policy$, and it is the $(s,s^{'})$-th component of the matrix $\mathbf{P}_{\pi_{\bm \theta}}$, i.e., $\mathbf{P}^{t}_{\pi_{\bm \theta}}[s,s^{'}]=\mathbb{P}_{\policy}(s_t=s^{'}|s)$.\\
  \hline
$d_{\pi_{\bm {\theta}}}^{s_0}(s),~d_{\pi_{\bm {\theta}}}^{\rho_0}(s)$  &: & The normalized discounted distribution of the future state $s$ encountered starting at $s_0$ by executing $\pi_{\bm {\theta}}$:
$d_{\pi_{\bm {\theta}}}^{s_0}(s)=:(1-\gamma)\sum_{t=0}^{\infty}\gamma^{t}\mathbb{P}_{\pi_{\bm {\theta}}}(s_t=s|s_0).$ Since $s_0\sim\rho(\cdot)$, we define $d_{\pi_{\bm {\theta}}}^{\rho_0}(s)=:\E_{s_0\sim\rho(\cdot)}[d_{\pi_{\bm {\theta}}}^{s_0}(s)]$.\\
$\bd_{\pi_{\bm {\theta}}}^{\rho_0}$&: &  It stores all the normalized discounted state distributions $d_{\pi_{\bm {\theta}}}^{\rho_0}(s)$, $\in\calS$, i.e., $\bd_{\pi_{\bm {\theta}}}^{\rho_0}\in\R^{|\calS|}$, and its components are:
$
\bd_{\pi_{\bm {\theta}}}^{\rho_0}[s]=d_{\pi_{\bm {\theta}}}^{\rho_0}(s).
$\\
   \hline
\end{tabular}

\subsubsection{Extend them to $\lambda$-version. }

\begin{tabular}{|r c p{10.5cm}|}
 \hline
    $\mathbf{P}^{(\lambda)}_{\pi_{\bm \theta}}$&: &$\mathbf{P}^{(\lambda)}_{\pi_{\bm \theta}}=(1-\gamma\lambda)\sum_{{t}=0}^{\infty}(\gamma\lambda)^{{t}}\bP^{{t}+1}_{\policy}.$
    \\
    $\Pro_{\policy}^{(\lambda)}(s^{'}|s)$&: &$\Pro_{\policy}^{(\lambda)}(s^{'}|s)=:\mathbf{P}^{(\lambda)}_{\pi_{\bm \theta}}[s,s^{'}]=(1-\gamma\lambda)\sum_{{t}=0}^{\infty}(\gamma\lambda)^{{t}}\Pro_{\policy}(s_{t+1}=s^{'}|s)$.
    \\
       \hline
   $\mathbf{r}^{(\lambda)}_{\pi_{\bm \theta}},~R^{(\lambda)}_{\pi_{\bm \theta}}(s)$ &: &$\mathbf{r}^{(\lambda)}_{\pi_{\bm \theta}}=\sum_{{t}=0}^{\infty}(\gamma\lambda\bP_{\policy})^{{t}}\mathbf{r}_{\pi_{\bm \theta}};~ R^{(\lambda)}_{\pi_{\bm \theta}}(s)=:
  \mathbf{r}^{(\lambda)}_{\pi_{\bm \theta}}[s].$\\
     \hline
   $\tilde{\gamma}$&: &$\tilde{\gamma}=\dfrac{\gamma(1-\lambda)}{1-\gamma\lambda}$.\\
      \hline
       $d_{\pi_{\bm {\theta}}}^{s_0,\lambda}(s)$&: &$d_{\pi_{\bm {\theta}}}^{s_0,\lambda}(s)=(1-\tilde \gamma)\sum_{t=0}^{\infty}\tilde{\gamma}^{t}\mathbb{P}^{(\lambda)}_{\pi_{\bm {\theta}}}(s_t=s|s_0)$.\\
          $d_{\pi_{\bm {\theta}}}^{\lambda}(s),~\bd_{\pi_{\bm {\theta}}}^{\lambda}$&: & $d_{\pi_{\bm {\theta}}}^{\lambda}(s)=\E_{s_0\sim\rho_{0}(\cdot)}\left[d_{\pi_{\bm {\theta}}}^{s_0,\lambda}(s)\right],~\bd_{\pi_{\bm {\theta}}}^{\lambda}[s]=d_{\pi_{\bm {\theta}}}^{\lambda}(s)$.      
         \\
      \hline
\end{tabular}

\subsubsection{
 TD error w.r.t. any function $\varphi(\cdot)$.
 }

\begin{tabular}{|r c p{10.5cm}|}
 \hline
     $\delta_t^{\varphi}$&: &$ \delta_t^{\varphi}=r(s_{t+1}|s_t,a_t)+\gamma\varphi(s_{t+1})-\varphi(s_{t}).$
     \\
    $ \delta^{\varphi}_{\policy,t}(s)$&: &$ \delta^{\varphi}_{\policy,t}(s)=\E_{s_{t}\sim\Pro_{\policy}(\cdot|s),a_{t}\sim{\policy}(\cdot|s_t),s_{t+1}\sim\Pro(\cdot|s_t,a_t)}\left[\delta_t^{\varphi}\right]$.\\
      $\bm{\delta}^{\varphi}_{\policy,t}$&: & $\bm{\delta}^{\varphi}_{\policy,t}[s]={{\delta}}^{\varphi}_{\policy,t}(s).$
      \\
      $\Delta_{t}^{\varphi}(\policy,\policyy,s)$&: & $\E_{s_{t}\sim\Pro_{\policyy}(\cdot|s),a_{t}\sim{\policyy}(\cdot|s_t),s_{t+1}\sim\Pro(\cdot|s_t,a_t)}\left[\left(\dfrac{\policy(a_t|s_t)}{\policyy(a_t|s_t)}-1\right)\delta_t^{\varphi}\right]$.\\
      $\bm{\Delta}_{t}^{\varphi}(\policy,\policyy)$&: & $\bm{\Delta}_{t}^{\varphi}(\policy,\policyy)[s]=\Delta_{t}^{\varphi}(\policy,\policyy,s)$.
      \\
          \hline
\end{tabular}

\clearpage

\section{Additional Discussion about Related Work}
\label{app-related-work}

This section reviews three typical safe reinforcement learning algorithms: CPO \citep{AchiamHTA17}, PCPO \citep{yang2020projection} and FOCOPS \citep{zhang2020first}.
Those algorithms also use new surrogate functions to replace the objective and constraints, which resembles the proposed CUP algorithm.
The goal is to present the contribution of our work.

\subsection{CPO \citep{AchiamHTA17}}

\label{app-sec:cpo}
For a given policy $\pi_{{\bm{\theta}}_{k}}$, CPO updates new policy $\pi_{{\bm{\theta}}_{k+1}}$ as follows:
\begin{flalign}
\label{app-cpo-objective}
&\pi_{{\bm{\theta}}_{k+1}}=\arg\max_{\pi_{{\bm{\theta}}}\in\Pi_{{\bm{\theta}}}}~~~~\E_{s\sim d^{\rho_0}_{\pi_{{\bm{\theta}}_k}}(\cdot),a\sim\pi_{{\bm{\theta}}}(\cdot|s)}\left[A_{\pi_{{\bm{\theta}}_k}}(s,a)\right]\\
\label{app-cost-constraint}
&~~~~~~~~\text{s.t.}~~J^{c}(\pi_{{\bm{\theta}}_k})+\frac{1}{1-\gamma}\E_{s\sim d^{\rho_0}_{\pi_{{\bm{\theta}}_k}}(\cdot),a\sim\pi_{{\bm{\theta}}}(\cdot|s)}\left[A^{c}_{\pi_{{\bm{\theta}}_k}}(s,a)\right]\leq b,\\
\label{app-trust-region}
&~~~~~~~~\bar{D}_{\mathrm{KL}}(\pi_{{\bm{\theta}}},\pi_{{\bm{\theta}}_k})=\E_{s\sim d^{\rho_0}_{\pi_{{\bm{\theta}}_k}}(\cdot)}[\mathrm{KL}(\pi_{{\bm{\theta}}},\pi_{{\bm{\theta}}_k})[s]]\leq\delta.
\end{flalign}
It is impractical to solve the problem (\ref{app-cpo-objective}) directly due to the computational cost.
\citep{AchiamHTA17} suggest to find some convex approximations to replace the term $A_{\pi_{{\bm{\theta}}_k}}(s,a)$ and $\bar{D}_{\mathrm{KL}}(\pi_{{\bm{\theta}}},\pi_{{\bm{\theta}}_k})$ Eq.(\ref{app-cpo-objective})-(\ref{app-trust-region}).

Concretely, according to (\ref{performance-difference-2002}), \cite{AchiamHTA17} suggest to use first-order Taylor expansion of $J(\pi_{\bm{\theta}})$ to replace the objective (\ref{app-cpo-objective}) as follows,
\begin{flalign}
\nonumber
\frac{1}{1-\gamma}\E_{s\sim d_{\pi_{{\bm{\theta}}_k}}^{\rho_0}(\cdot),a\sim\pi_{{\bm{\theta}}_k}(\cdot|s)}\left[\dfrac{\pi_{{\bm{\theta}}}(a|s)}{\pi_{{\bm{\theta}}_k}(a|s)}A_{\pi_{{\bm{\theta}}_k}}(s,a)\right]=J(\pi_{{\bm{\theta}}})-J(\pi_{{\bm{\theta}}_k})\approx({\bm{\theta}}-{\bm{\theta}}_k)^{\top}\nabla_{{\bm{\theta}}} J(\pi_{\bm{\theta}}).
\end{flalign}
Similarly, \cite{AchiamHTA17} use the following approximations to turn the constrained policy optimization (\ref{app-cpo-objective})-(\ref{app-trust-region}) to be a convex problem,
\begin{flalign}
\label{app-contrained-01}
\frac{1}{1-\gamma}\E_{s\sim d_{\pi_{{\bm{\theta}}_k}}^{\rho_0}(\cdot),a\sim\pi_{{\bm{\theta}}_k}(\cdot|s)}&\left[\frac{\pi_{{\bm{\theta}}}(a|s)}{\pi_{{\bm{\theta}}_k}(a|s)}A^{c}_{\pi_{{\bm{\theta}}_k}}(s,a)\right]\approx({\bm{\theta}}-{\bm{\theta}}_k)^{\top}\nabla_{{\bm{\theta}}} J^{c}(\pi_{\bm{\theta}}),\\
\label{app-contrained-02}
\bar{D}_{\mathrm{KL}}(\pi_{{\bm{\theta}}},\pi_{{\bm{\theta}}_k})&\approx({\bm{\theta}}-{\bm{\theta}}_k)^{\top}\mathbf{H}({\bm{\theta}}-{\bm{\theta}}_k),
\end{flalign}
where $\mathbf{H}$ is Hessian matrix of $\bar{D}_{\mathrm{KL}}(\pi_{{\bm{\theta}}},\pi_{{\bm{\theta}}_k})$, i.e., 
\[\mathbf{H}[i,j]=:\dfrac{\partial^2}{\partial {\bm{\theta}}_i \partial {\bm{\theta}}_j}\E_{s\sim d^{\rho_0}_{\pi_{{\bm{\theta}}_k}}(\cdot)}\left[\mathrm{KL}(\pi_{{\bm{\theta}}},\pi_{{\bm{\theta}}_k})[s]\right],\]
Eq.(\ref{app-contrained-02}) is the second-oder approximation of (\ref{app-trust-region}).

Let $ \lambda_{\star},\nu_{\star}$ is the dual solution of the following problem
\[
\lambda_{\star},\nu_{\star}=\arg\max_{\lambda\ge0,\nu\ge0}
\left\{
\dfrac{-1}{2\lambda}
\left(
\bg^{\top}\bH^{-1}\bg-2\nu r+sv^2
\right)
+\nu c-\dfrac{\lambda\delta}{2}
\right\}
;
\]
where
$\bg=\nabla_{\bm{\theta}}\E_{s\sim d^{\rho_0}_{\pi_{{\bm{\theta}}_k}}(\cdot),a\sim\pi_{{\bm{\theta}}}(\cdot|s)}\left[A_{\pi_{{\bm{\theta}}_k}}(s,a)\right]$, $\ba=\nabla_{\bm{\theta}}\E_{s\sim d^{\rho_0}_{\pi_{{\bm{\theta}}_k}}(\cdot),a\sim\pi_{{\bm{\theta}}}(\cdot|s)}\left[A^{c}_{\pi_{{\bm{\theta}}_k}}(s,a)\right]$, $r=\bg^{\top}\bH\ba,s=\ba^{\top}\bH^{-1}\ba$, and $c=J^{c}(\pi_{\bm{\theta}_k})-b$.

Finally, CPO updates parameters according to conjugate gradient as follows:
if approximation to CPO is feasible:
\begin{flalign}
\nonumber
\bm{\theta}_{k+1}=\bm{\theta}_{k}+\frac{1}{\lambda_{\star}}\bH^{-1}(\bg-\nu_{\star}\ba),
\end{flalign}
else,
\[
\bm{\theta}_{k+1}=\bm{\theta}_{k}-\sqrt{\dfrac{2\delta}{\ba^{\top}\bH^{-1}\ba}}\bH^{-1} \ba.
\]

 \subsection{PCPO \citep{yang2020projection}}

Projection-Based Constrained Policy Optimization (PCPO) is an iterative method for optimizing policies in a two-step process: the first step performs a local reward improvement update, while the second step reconciles any constraint violation by projecting the policy back onto the constraint set.

 \textbf{Reward Improvement:}
 \begin{flalign}
 \nonumber
 \pi_{\bm{\theta}_{k+\frac{1}{2}}}=\arg\max_{\pi_{{\bm{\theta}}}\in\Pi_{{\bm{\theta}}}}\E_{s\sim d^{\rho_0}_{\pi_{{\bm{\theta}}_k}}(\cdot),a\sim\pi_{{\bm{\theta}}}(\cdot|s)}\left[A_{\pi_{{\bm{\theta}}_k}}(s,a)\right],\\
 \nonumber
  \text{ s.t.} \bar{D}_{\mathrm{KL}}(\pi_{{\bm{\theta}}},\pi_{{\bm{\theta}}_k})=\E_{s\sim d^{\rho_0}_{\pi_{{\bm{\theta}}_k}}(\cdot)}[\mathrm{KL}(\pi_{{\bm{\theta}}},\pi_{{\bm{\theta}}_k})[s]]\leq\delta;
\end{flalign}
 \textbf{Projection:}
\begin{flalign}
\nonumber
    & \pi_{{\bm{\theta}}_{k+1}}=\arg\min_{\pi_{{\bm{\theta}}}\in\Pi_{{\bm{\theta}}}}~D\left(\pi_{{\bm{\theta}}},\pi_{{\bm{\theta}}_{k+\frac{1}{2}}}\right),\\
 \nonumber    
    \text{s.t.}~ J^{c}(\pi_{{\bm{\theta}}_k})&+\dfrac{1}{1-\gamma}\E_{s\sim d^{\rho_0}_{\pi_{{\bm{\theta}}_k}}(\cdot),a\sim\pi_{{\bm{\theta}}}(\cdot|s)}\left[A^{c}_{\pi_{{\bm{\theta}}_k}}(s,a)\right]\leq b.
\end{flalign}

Then, \cite{yang2020projection} follows CPO \citep{AchiamHTA17} uses convex approximation to original problem, and calculate the update rule as follows,
\[
\bm{\theta}_{k+1}=\bm{\theta}_{k}-\sqrt{\dfrac{2\delta}{\bg^{\top}\bH^{-1}\bg}}\bH^{-1} \bg
-\max
\left(0,
\dfrac
{
\sqrt{\dfrac{2\delta}{\bg^{\top}\bH^{-1}\bg}}\ba^{\top}\bH^{-1} \bg+c
}
{\ba^{\top}\bL^{-1}\ba}
\right)\bL^{-1}\ba,
\]
where $\bL=\bI$ if $D$ is $\ell_2$-norm, and $\bL=\bH$ if $D$ is KL-divergence.

\subsection{FOCOPS \citep{zhang2020first}}

\cite{zhang2020first} propose the First Order Constrained Optimization in Policy Space (FOCOPS) that is a two-step approach. We present it as follows.

\textbf{Step1: Finding the optimal update policy.}
Firstly, for a given policy $\pi_{\bm{\theta}k}$, we find an optimal update policy $\pi^{\star}$ by solving the optimization problem (\ref{app-cpo-objective})-(\ref{app-trust-region}) in the non-parameterized policy space.
\begin{flalign}
\label{app-non-parameterized-01}
&\pi^{\star}=\arg\max_{\pi\in\Pi}~~~~\E_{s\sim d^{\rho_0}_{\pi_{{\bm{\theta}}_k}}(\cdot),a\sim\pi(\cdot|s)}\left[A_{\pi_{{\bm{\theta}}_k}}(s,a)\right]\\
\label{app-non-parameterized-02}
&~~~~~~~~\text{s.t.}~~J^{c}(\pi_{{\bm{\theta}}_k})+\frac{1}{1-\gamma}\E_{s\sim d^{\rho_0}_{\pi_{{\bm{\theta}}_k}}(\cdot),a\sim\pi(\cdot|s)}\left[A^{c}_{\pi_{{\bm{\theta}}_k}}(s,a)\right]\leq b,\\
\label{app-non-parameterized-03}
&~~~~~~~~\bar{D}_{\mathrm{KL}}(\pi_{{\bm{\theta}}},\pi_{{\bm{\theta}}_k})=\E_{s\sim d^{\rho_0}_{\pi_{{\bm{\theta}}_k}}(\cdot)}[\mathrm{KL}(\pi,\pi_{{\bm{\theta}}_k})[s]]\leq\delta.
\end{flalign}
If $\pi_{\bm{\theta}_k}$ is feasible, then the optimal policy for (\ref{app-non-parameterized-01})-(\ref{app-non-parameterized-03}) takes the following form:
\begin{flalign}
\label{optimal-solu}
\pi^{\star}(a|s)=\dfrac{\pi_{\bm{\theta}_k}(a|s)}{Z_{\lambda,\nu}(s)}\exp
\left(\dfrac{1}{\lambda}
\left(
A_{\pi_{\bm{\theta}_k}}(s,a)-\nu A^{c}_{\pi_{\bm{\theta}_k}}(s,a)
\right)
\right),
\end{flalign}
where $Z_{\lambda,\nu}(s)$ is the partition function which ensures (\ref{optimal-solu}) is a valid probability distribution, $\lambda$ and $\nu$ are solutions to the optimization problem:
\[
\min_{\lambda,\nu\ge0}\lambda\nu+\nu \tilde{b}+\lambda \E_{s\sim d^{\rho_0}_{\pi_{{\bm{\theta}}_k}}(\cdot),a\sim\pi^{\star}(\cdot|s)}\left[Z_{\lambda,\nu}(s)\right],
\]
the term $\tilde{b}=(1-\gamma)(b-J^{c}(\pi_{\bm{\theta}_k}))$.

\textbf{Step 2: Projection}
Then, we project the policy found in the previous step back into the parameterized policy space $\Pi_{\bm{\theta}}$ by solving for the closest policy $\policy\in\Pi_{\bm{\theta}}$ to $\pi^{\star}$ in order to obtain $\pi_{\bm{\theta}_{k+1}}$:
\[
\bm{\theta}_{k+1}=\arg\min_{\bm\theta} \E_{s\sim d^{\rho_0}_{\pi_{{\bm{\theta}}_k}}(\cdot)}[\mathrm{KL}(\pi_{{\bm{\theta}}},\pi^{\star})[s]].
\]

\begin{landscape}
\renewcommand{\arraystretch}{1.5}
\begin{table}[t]
 \caption{Comparison of some safe reinforcement algorithms.}
 \label{tab:comp}
 \centering
 \small
 \setlength{\extrarowheight}{0.04cm}
 \begin{tabularx}{\linewidth}{m{2.5cm}|>{\centering}m{8cm}|>{\centering}m{7.5cm}|m{3cm}}
  \toprule
  Algorithm & Optimization problem & Implementation & Remark \\
  \midrule
  \makecell{CPO\\ \citep{AchiamHTA17}}
   & $\pi_{\bm{\theta}_{k+1}}=\arg\max_{\pi_{{\bm{\theta}}}\in\Pi_{{\bm{\theta}}}}\E_{s\sim d^{\rho_0}_{\pi_{{\bm{\theta}}_k}}(\cdot),a\sim\pi_{{\bm{\theta}}}(\cdot|s)}\left[A_{\pi_{{\bm{\theta}}_k}}(s,a)\right]$,
   \newline
    s.t. $J^{c}(\pi_{{\bm{\theta}}_k})+\E_{s\sim d^{\rho_0}_{\pi_{{\bm{\theta}}_k}}(\cdot),a\sim\pi_{{\bm{\theta}}}(\cdot|s)}\left[A^{c}_{\pi_{{\bm{\theta}}_k}}(s,a)\right]\leq b $,
    \newline
   $ \bar{D}_{\mathrm{KL}}(\pi_{{\bm{\theta}}},\pi_{{\bm{\theta}}_k})=\E_{s\sim d^{\rho_0}_{\pi_{{\bm{\theta}}_k}}(\cdot)}[\mathrm{KL}(\pi_{{\bm{\theta}}},\pi_{{\bm{\theta}}_k})[s]]\leq\delta$.
    & $\bm{\theta}_{k+1}=\arg\max_{\bm\theta}~\bg^{\top}(\bm{\theta}-\bm{\theta}_{k})$, 
    \newline
     s.t. $c+\mathbf{b}^{\top}(\bm{\theta}-\bm{\theta}_k)\leq0$,
     \newline
   $ \dfrac{1}{2}(\bm{\theta}-\bm{\theta}_k)^{\top}\mathbf{H}(\bm{\theta}-\bm{\theta}_k)\leq\delta$.
    &Convex Implementation
    \\
   \hline 
    \makecell{PCPO\\ \citep{yang2020projection}}
   &Reward Improvement $\pi_{\bm{\theta}_{k+\frac{1}{2}}}=\arg\max_{\pi_{{\bm{\theta}}}\in\Pi_{{\bm{\theta}}}}\E_{s\sim d^{\rho_0}_{\pi_{{\bm{\theta}}_k}}(\cdot),a\sim\pi_{{\bm{\theta}}}(\cdot|s)}\left[A_{\pi_{{\bm{\theta}}_k}}(s,a)\right]$,
    \newline
   s.t. $\bar{D}_{\mathrm{KL}}(\pi_{{\bm{\theta}}},\pi_{{\bm{\theta}}_k})=\E_{s\sim d^{\rho_0}_{\pi_{{\bm{\theta}}_k}}(\cdot)}[\mathrm{KL}(\pi_{{\bm{\theta}}},\pi_{{\bm{\theta}}_k})[s]]\leq\delta$;
    \newline
    Projection
     \newline
    $\pi_{{\bm{\theta}}_{k+1}}=\arg\min_{\pi_{{\bm{\theta}}}\in\Pi_{{\bm{\theta}}}}~D\left(\pi_{{\bm{\theta}}},\pi_{{\bm{\theta}}_{k+\frac{1}{2}}}\right)$,
    \newline
    s.t. $J^{c}(\pi_{{\bm{\theta}}_k})+\dfrac{1}{1-\gamma}\E_{s\sim d^{\rho_0}_{\pi_{{\bm{\theta}}_k}}(\cdot),a\sim\pi_{{\bm{\theta}}}(\cdot|s)}\left[A^{c}_{\pi_{{\bm{\theta}}_k}}(s,a)\right]\leq b $.
    & Reward Improvement
     \newline
    $\bm{\theta}_{k+\frac{1}{2}}=\arg\max_{\bm\theta}~\bg^{\top}(\bm{\theta}-\bm{\theta}_{k})$, 
     \newline
     s.t.$\dfrac{1}{2}(\bm{\theta}-\bm{\theta}_k)^{\top}\mathbf{H}(\bm{\theta}-\bm{\theta}_k)\leq\delta$;
    \newline
    Projection
     \newline
      ${{\bm{\theta}}_{k+1}}=\arg\min_{\bm\theta}\dfrac{1}{2}(\bm{\theta}-\bm{\theta}_k)^{\top}\mathbf{L}(\bm{\theta}-\bm{\theta}_k)$,
       \newline 
     s.t. $c+\mathbf{b}^{\top}(\bm{\theta}-\bm{\theta}_k)\leq0$.
    &Convex Implementation\\  
   \hline 
    \makecell{FOCOPS\\ \citep{zhang2020first}}
   &Optimal update policy
    \newline
    $\pi^{\star}=\arg\max_{\pi\in\Pi}\E_{s\sim d^{\rho_0}_{\pi_{{\bm{\theta}}_k}}(\cdot),a\sim\pi(\cdot|s)}\left[A_{\pi_{{\bm{\theta}}_k}}(s,a)\right]$,
   \newline
    s.t. $J^{c}(\pi_{{\bm{\theta}}_k})+\E_{s\sim d^{\rho_0}_{\pi_{{\bm{\theta}}_k}}(\cdot),a\sim\pi(\cdot|s)}\left[A^{c}_{\pi_{{\bm{\theta}}_k}}(s,a)\right]\leq b $,
    \newline
   $ \bar{D}_{\mathrm{KL}}(\pi_{{\bm{\theta}}},\pi_{{\bm{\theta}}_k})=\E_{s\sim d^{\rho_0}_{\pi_{{\bm{\theta}}_k}}(\cdot)}[\mathrm{KL}(\pi,\pi_{{\bm{\theta}}_k})[s]]\leq\delta$;
   \newline
   Projection
    \newline
    $\pi_{{\bm{\theta}}_{k+1}}=\arg\min_{\pi_{{\bm{\theta}}}\in\Pi_{{\bm{\theta}}}}~\E_{s\sim d^{\rho_0}_{\pi_{{\bm{\theta}}_k}}(\cdot)}[\mathrm{KL}(\pi_{{\bm{\theta}}},\pi^{\star})[s]]$.
    & Optimal update policy
    \newline
     \newline
    $\pi^{\star}(a|s)=\frac{\pi_{\bm{\theta}_k}(a|s)}{Z_{\lambda,\nu}(s)}\exp\left(\frac{1}{\lambda}\left(A_{\pi_{\bm{\theta}_k}}(s,a)-\nu A^{c}_{\pi_{\bm{\theta}_k}}(s,a)\right)\right)$;
     \newline
      \newline
      Projection
       \newline
    $\bm{\theta}_{k+1}=\arg\min_{\bm\theta} \E_{s\sim d^{\rho_0}_{\pi_{{\bm{\theta}}_k}}(\cdot)}[\mathrm{KL}(\pi_{{\bm{\theta}}},\pi^{\star})[s]].$
    & Non-Convex Implementation
    \\
     \hline 
    \makecell{CUP\\ (Our Work)}
    &
    Policy Improvement
    \begin{flalign} 
    \nonumber
    \pi_{{\bm{\theta}}_{k+\frac{1}{2}}}=\arg\max_{\pi_{{\bm{\theta}}}\in\Pi_{{\bm{\theta}}}}
\bigg\{
\E_{s\sim{d}_{\pi_{\bm{\theta}_k}}^{\lambda}(\cdot),a\sim\policy(\cdot|s)}
\left[
A^{\mathrm{GAE}(\gamma,\lambda)}_{{\pi_{\bm{\theta}_k}}}(s,a)\right]
\\
 \nonumber
-\alpha_k
\sqrt{
\E_{s\sim{d}_{{\pi_{\bm{\theta}_k}}}^{\lambda}(\cdot)}\left[\mathrm{KL}(\pi_{\bm{\theta}_k},\policy)[s]\right]}
\bigg\},
\end{flalign}
Projection
 \begin{flalign}
\nonumber
\pi_{{\bm{\theta}}_{k+1}}=\arg\min_{\pi_{{\bm{\theta}}}\in\Pi_{{\bm{\theta}}}}~D\Big(\pi_{{\bm{\theta}}},\pi_{{\bm{\theta}}_{k+\frac{1}{2}}}\Big),~~~~~~~~~~~~~~~~~\\
\nonumber
 \text{s.t.} J^{c}(\pi_{{\bm{\theta}}_k})+\dfrac{1}{1-\tilde\gamma}\E_{s\sim{d}_{\pi_{\bm{\theta}_k}}^{\lambda}(\cdot),a\sim\policy(\cdot|s)}
\left[A^{\mathrm{GAE}(\gamma,\lambda)}_{{\pi_{\bm{\theta}_k}},C}(s,a)\right]\\
\nonumber
+\beta_k\sqrt{\E_{s\sim{d}_{{\pi_{\bm{\theta}_k}}}^{\lambda}(\cdot)}\left[\mathrm{KL}(\pi_{\bm{\theta}_k},\policy)[s]\right]}\leq b. 
\end{flalign}
    &
    Policy Improvement
     \newline
    \begin{flalign}
    \nonumber
{\bm{\theta}}_{k+\frac{1}{2}}=\arg\max_{{\bm{\theta}}}\Bigg\{\frac{1}{T}\sum_{t=1}^{T}\frac{\pi_{\bm{\theta}}(a_{t}|s_{t})}{{\pi_{\bm{\theta}_{k}}}(a_{t}|s_{t})}\hat{A}_{t}~~~~~~~~~~~~~~~~~~~~~~~~~~~~~\\
\nonumber
-\alpha\sqrt{\frac{1}{T}\sum_{t=1}^{T}\mathrm{KL}(\pi_{\bm{\theta}_{k}}(\cdot|s_{t}),\pi_{{\bm{\theta}}}(\cdot|s_{t}))}\Bigg\};
\end{flalign}
Projection
 \newline
\begin{flalign}
\nonumber
{\bm{\theta}}_{k+1}=\arg\min_{{\bm{\theta}}} \dfrac{1}{T}\sum_{t=1}^{T}
\bigg\{
\mathrm{KL}\left({\pi_{\bm{\theta}_{k+\frac{1}{2}}}}(\cdot|s_{t}),\pi_{{\bm{\theta}}}(\cdot|s_{t})\right)\\
\nonumber
+\nu_k
\dfrac{1-\gamma\lambda}{1-\gamma}\dfrac{\pi_{\bm{\theta}}(a_{t}|s_{t})}{\pi_{{\bm{\theta}}_k}(a_{t}|s_{t})}\hat{A}^{C}_{t}
\bigg\}.
\end{flalign}
    &
    Non-Convex Implementation
    \\
    \bottomrule
 \end{tabularx}
 \label{app-table-com}
\end{table}
\end{landscape}

\clearpage

\section{Constrained Update Projection Algorithm}
\label{sec-app-cpu}

\subsection{Practical Implementation of Performance Improvement}

\subsubsection{Sample-based Performance Improvement}

Let the trajectory $\{(s_t,a_t,r_{t+1},c_{t+1})\}_{t=1}^{T}$ be sampled according to $\pi_{\bm{\theta}_k}$, then we denote the empirical KL-divergence with respect to $\policy$ and $\policyy$ as follows,
\[\hat{D}_{\mathtt{KL}}(\policy,\policyy)=\dfrac{1}{T}\sum_{t=1}^{T}\mathtt{KL}(\policy(a_t|s_t),\policyy(a_t|s_t)).\]

We defined the following $\hat{\mathcal{L}}_{\text{R}}(\policy,\pi_{{\bm{\theta}}_k})$,
\begin{flalign}
\label{loss-performance-improvement}
\hat{\mathcal{L}}_{\text{R}}(\policy,\pi_{{\bm{\theta}}_k})&=\dfrac{1}{T}\sum_{t=1}^{T}\dfrac{\pi_{\bm{\theta}}(a_t|s_t)}{\pi_{{\bm{\theta}}_k}(a_t|s_t)}\hat{A}_t-\alpha_k\sqrt{\hat{D}_{\mathtt{KL}}(\pi_{{\bm{\theta}}_k},\pi_{{\bm{\theta}}})},
\end{flalign}
where $\hat{A}_t$ is an estimator of $A^{\mathtt{GAE}(\gamma,\lambda)}_{\pi_{{\bm{\theta}}_k}}(s,a)$.
The term $\hat{\mathcal{L}}_{\text{R}}(\policy,\pi_{{\bm{\theta}}_k})$ (\ref{loss-performance-improvement}) is an estimator of the next expectation that appears in (\ref{performance-improvement})
\[
\E_{s \sim{d}_{\pi_{\bm{\theta}_k}}^{\lambda}(\cdot),~a\sim\pi_{\bm{\theta}_k}(\cdot|s)}\left[\frac{\pi_{\bm{\theta}}(a|s)}{\pi_{{\bm{\theta}}_k}(a|s)}
A^{\mathtt{GAE}(\gamma,\lambda)}_{{\pi_{\bm{\theta}_k}}}(s,a)\right]
-\alpha_k
\sqrt{
\E_{s\sim{d}_{{\pi_{\bm{\theta}_k}}}^{\lambda}(\cdot)}\left[\mathtt{KL}(\pi_{\bm{\theta}_k},\policy)[s]\right]}.
\]
Then we implement the performance improvement as follows,
\begin{flalign}
\label{app-perfor-practical-01}
\pi_{{\bm{\theta}}_{k+\frac{1}{2}}}&=\arg\max_{\pi_{\bm{\theta}}\in\Pi_{\theta}}\left\{\hat{\mathcal{L}}_{\text{R}}(\policy,\pi_{{\bm{\theta}}_k})\right\}.
\end{flalign}

\subsubsection{Clipped Surrogate Objective}
\label{sec-clip}

How can the implementation (\ref{app-perfor-practical-01}) take the biggest possible improvement step on a policy using the data we currently have, without stepping so far that we accidentally cause performance collapse? Now, we present a clip implementation for policy improvement, which is very efficient in practice.

Instead of the previous policy improvement (\ref{performance-improvement}), according to PPO \cite{schulman2017proximal}, we update the policy as follows,
\begin{flalign}
\nonumber
\pi_{{\bm{\theta}}_{k+\frac{1}{2}}}&=\arg\max_{\pi_{{\bm{\theta}}}\in\Pi_{{\bm{\theta}}}}
\left\{
\E_{s\sim{d}_{\pi_{\bm{\theta}_k}}^{\lambda}(\cdot),~a\sim\policy(\cdot|s)}\left[\calL_{\text{clip}}\left(s,a,\policy,\pi_{\bm{\theta}_{k}},\epsilon\right)\right]
\right\},
\end{flalign}
where the the objective $\calL_{\text{clip}}$ is defined as follows,
\begin{flalign}
\label{clip-objective}
\calL_{\text{clip}}\left(s,a,\policy,\pi_{\bm{\theta}_{k}},\epsilon\right)
=&\min\left\{
\dfrac{\pi_{\bm{\theta}}(a|s)}{\pi_{{\bm{\theta}}_k}(a|s)}A^{\mathtt{GAE}(\gamma,\lambda)}_{{\pi_{\bm{\theta}_k}}}(s,a),~
\text{clip}\left(
\dfrac{\pi_{\bm{\theta}}(a|s)}{\pi_{{\bm{\theta}}_k}(a|s)},1-\epsilon,1+\epsilon
\right)A^{\mathtt{GAE}(\gamma,\lambda)}_{{\pi_{\bm{\theta}_k}}}(s,a)
\right\},
\end{flalign}
$\epsilon$ is a hyperparameter which roughly says how far away the policy $\pi_{{\bm{\theta}}_{k+\frac{1}{2}}}$ is allowed to go from the current policy $\pi_{{\bm{\theta}}_{k}}$.
The objective $\calL_{\text{clip}}\left(s,a,\policy,\pi_{\bm{\theta}_{k}},\epsilon\right)$ is complex, we present the insights of this clip mechanism \cite{schulman2017proximal} to make CUP learn stably.

\textbf{Positive GAE: $A^{\mathtt{GAE}(\gamma,\lambda)}_{{\pi_{\bm{\theta}_k}}}(s,a)>0$}. Firstly, we consider the positive advantage, which implies the objective $\calL_{\text{clip}}\left(s,a,\policy,\pi_{\bm{\theta}_{k}},\epsilon\right)$ reduces to 
\begin{flalign}
\calL_{\text{clip}}\left(s,a,\policy,\pi_{\bm{\theta}_{k}},\epsilon\right)
=\min\left\{
\dfrac{\pi_{\bm{\theta}}(a|s)}{\pi_{{\bm{\theta}}_k}(a|s)}, 1+\epsilon
\right\}A^{\mathtt{GAE}(\gamma,\lambda)}_{{\pi_{\bm{\theta}_k}}}(s,a).
\end{flalign}
Since $A^{\mathtt{GAE}(\gamma,\lambda)}_{{\pi_{\bm{\theta}_k}}}(s,a)>0$, to improve the performance, we need to increase $\policy$.
The $\min\{\cdot\}$ operator determines the quantization how much the CUP improves. If the policy improves too much such that \[\policy(a|s)>(1+\epsilon)\pi_{\bm{\theta}_{k}}(a|s),\]
The $\min\{\cdot\}$ operator hit the objective with a ceiling of $(1+\epsilon)A^{\mathtt{GAE}(\gamma,\lambda)}_{{\pi_{\bm{\theta}_k}}}(s,a).$
The clip technique requires CUP learns a policy $\pi_{{\bm{\theta}}_{k+\frac{1}{2}}}$ does not benefit by going far away from the current policy $\pi_{{\bm{\theta}}_{k}}$.

\textbf{Negative GAE: $A^{\mathtt{GAE}(\gamma,\lambda)}_{{\pi_{\bm{\theta}_k}}}(s,a)<0$}. Let us consider the negative advantage, which implies the objective $\calL_{\text{clip}}\left(s,a,\policy,\pi_{\bm{\theta}_{k}},\epsilon\right)$ reduces to 
\begin{flalign}
\calL_{\text{clip}}\left(s,a,\policy,\pi_{\bm{\theta}_{k}},\epsilon\right)
=\max\left\{
\dfrac{\pi_{\bm{\theta}}(a|s)}{\pi_{{\bm{\theta}}_k}(a|s)}, 1-\epsilon
\right\}A^{\mathtt{GAE}(\gamma,\lambda)}_{{\pi_{\bm{\theta}_k}}}(s,a).
\end{flalign}
Since $A^{\mathtt{GAE}(\gamma,\lambda)}_{{\pi_{\bm{\theta}_k}}}(s,a)<0$, to improve the performance, we need to decrease the policy $\policy(a|s)$.
The $\max\{\cdot\}$ operator determines the quantization how much the CUP improves. If the policy decrease too much such that \[\policy(a|s)<(1-\epsilon)\pi_{\bm{\theta}_{k}}(a|s),\]
The $\max\{\cdot\}$ operator hit the objective with a ceiling of $(1-\epsilon)A^{\mathtt{GAE}(\gamma,\lambda)}_{{\pi_{\bm{\theta}_k}}}(s,a).$
Thus, similar to the positive GAE the clip technique requires CUP learns a policy $\pi_{{\bm{\theta}}_{k+\frac{1}{2}}}$ does not benefit by going far away from the current policy $\pi_{{\bm{\theta}}_{k}}$.

\subsubsection{Learning from Sampling}

To short the expression, we introduce a function $g(\epsilon,A)$ as follows,
\begin{flalign}
\nonumber
g(\epsilon,A)
=
\begin{cases}
(1+\epsilon)A~~A\ge0
\\
(1+\epsilon)A~~A<0.
\end{cases}
\end{flalign}
Then we rewrite the objective (\ref{clip-objective}) as follows,
\begin{flalign}
\label{clip-objective-01}
\calL_{\text{clip}}\left(s,a,\policy,\pi_{\bm{\theta}_{k}},\epsilon\right)=\min\left\{
\dfrac{\pi_{\bm{\theta}}(a|s)}{\pi_{{\bm{\theta}}_k}(a|s)}A^{\mathtt{GAE}(\gamma,\lambda)}_{{\pi_{\bm{\theta}_k}}}(s,a),
g\left(\epsilon,A^{\mathtt{GAE}(\gamma,\lambda)}_{{\pi_{\bm{\theta}_k}}}(s,a)\right)
\right\}.
\end{flalign}

Recall the trajectory $\{(s_t,a_t,r_{t+1},c_{t+1})\}_{t=1}^{T}$ be sampled according to $\pi_{\bm{\theta}_k}$,
we defined the following $\hat{\mathcal{L}}_{\text{clip}}(\policy,\pi_{{\bm{\theta}}_k})$,
\begin{flalign}
\label{loss-clip-performance-improvement}
\hat{\mathcal{L}}_{\text{clip}}(\policy,\pi_{{\bm{\theta}}_k},\epsilon)&=
\min\left\{
\dfrac{1}{T}\sum_{t=1}^{T}\dfrac{\pi_{\bm{\theta}}(a_t|s_t)}{\pi_{{\bm{\theta}}_k}(a_t|s_t)}\hat{A}_t-
g\left(\epsilon,\dfrac{1}{T}\sum_{t=1}^{T}\hat{A}_t\right)
\right\}
,
\end{flalign}
where $\hat{A}_t$ is an estimator of $A^{\mathtt{GAE}(\gamma,\lambda)}_{\pi_{{\bm{\theta}}_k}}(s,a)$.
The term $\hat{\mathcal{L}}_{\text{clip}}(\policy,\pi_{{\bm{\theta}}_k},\epsilon)$ (\ref{loss-clip-performance-improvement}) is an estimator of the next expectation that appears in (\ref{clip-objective-01}).

Then we implement the performance improvement as follows,
\begin{flalign}
\nonumber
\pi_{{\bm{\theta}}_{k+\frac{1}{2}}}&=\arg\max_{\pi_{\bm{\theta}}\in\Pi_{\theta}}\left\{\hat{\mathcal{L}}_{\text{clip}}(\policy,\pi_{{\bm{\theta}}_k},\epsilon)\right\},
\end{flalign}
i.e., we obtain the parameter ${\bm{\theta}}_{k+\frac{1}{2}}$ according to 
\begin{flalign}
\nonumber
{\bm{\theta}}_{k+\frac{1}{2}}&={\bm{\theta}}_{k}+\eta_1\dfrac{\partial}{\partial{\bm{\theta}}}\hat{\mathcal{L}}_{\text{clip}}(\policy,\pi_{{\bm{\theta}}_k},\epsilon)\Big|_{\bm{\theta}=\bm{\theta}_{k}},
\end{flalign}
where $\eta_1$ is step-size.

\subsection{Practical Implementation of Projection}
\label{app-sec-practical-ip}

Recall (\ref{projection}), we introduce the new surrogate function with respected to cost function as follows,
\[
C_{\policyy}(\policy,\beta)=J^{c}(\policyy)+
\dfrac{1}{1-\tilde\gamma}\E_{s\sim{d}_{\policyy}^{\lambda}(\cdot),a\sim\policy(\cdot|s)}
\left[
A^{\mathtt{GAE}(\gamma,\lambda)}_{\policyy,C}(s,a)+\beta
\sqrt{\E_{s\sim{d}_{\policyy}^{\lambda}(\cdot)}\left[\mathtt{KL}(\policyy,\policy)[s]\right]}
\right],
\]
where $\beta$ is adaptive to the term $\frac{\sqrt{2}\tilde{\gamma}\left(\gamma\lambda(|\calS|-1)+1\right)\epsilon^{V}_{\policy}(\policyy)}{(1-\gamma\lambda)}$.
Now, we rewrite the projection step (\ref{projection}) as follows,
\begin{flalign}
\label{eq:projection}
\pi_{{\bm{\theta}}_{k+1}}=&\arg\min_{\pi_{{\bm{\theta}}}\in\Pi_{{\bm{\theta}}}}~D\left(\pi_{{\bm{\theta}}},\pi_{{\bm{\theta}}_{k+\frac{1}{2}}}\right),
~~\text{s.t.}~C_{\pi_{\bm{\theta}_k}}(\policy,\beta)\leq b.
\end{flalign}
We update the projection step (\ref{projection}) by replacing the distance function $D(\cdot,\cdot)$ by KL-divergence, and we solve the constraint problem (\ref{projection}) by the primal-dual approach.  
\begin{theorem}
\label{min-max-alg}
The constrained problem (\ref{eq:projection}) is equivalent to the following primal-dual problem:
\begin{flalign}
\nonumber
\max_{\nu\ge0}\min_{\pi_{\bm{\theta}}\in\Pi_{\bm{\theta}}}
\left\{
D\left(\pi_{{\bm{\theta}}},\pi_{{\bm{\theta}}_{k+\frac{1}{2}}}\right)+\nu
\left(
C_{\pi_{\bm{\theta}_k}}(\policy,\beta)-b
\right)
\right\}.
\end{flalign}
\end{theorem}
\begin{proof}
This result is a direct application of \cite[Chapter 5.9]{boyd2004convex}, and we also present it in \ref{app-sec-slater-condition}.
Firstly, we notice if $D\left(\cdot,\pi_{{\bm{\theta}}_{k+\frac{1}{2}}}\right)$ is KL divergence or $\ell_2$-norm, then
the constrained problem (\ref{eq:projection}) is a convex problem \footnote{It is worth noting that $\min_{\pi_{{\bm{\theta}}}\in\Pi_{{\bm{\theta}}}}~D\left(\pi_{{\bm{\theta}}},\pi_{{\bm{\theta}}_{k+\frac{1}{2}}}\right)$ is a convex problem, while $\min_{\bm{\theta}\in\R^{p}}~D\left(\pi_{{\bm{\theta}}},\pi_{{\bm{\theta}}_{k+\frac{1}{2}}}\right)$ can be a non-convex problem.}.
In fact, for a given policy $\pi_{{\bm{\theta}}_{k+\frac{1}{2}}}$, $D\left(\cdot,\pi_{{\bm{\theta}}_{k+\frac{1}{2}}}\right)$ is convex over the policy $\Pi_{\bm{\theta}}$, and $C_{\policyy}(\cdot,\beta)$ is also convex over the policy $\Pi_{\bm{\theta}}$. Additionally, Slater’s condition alway holds since $C_{\policyy}(\policyy,\beta)=0$.
\end{proof}

According to Theorem \ref{min-max-alg}, we turn the projection step (\ref{eq:projection}) as the following unconstrained problem,
\begin{flalign}
\label{app-proj-01}
\max_{\nu\ge0}\min_{\pi_{\bm{\theta}}\in\Pi_{\bm{\theta}}}
\left\{
D\left(\pi_{{\bm{\theta}}},\pi_{{\bm{\theta}}_{k+\frac{1}{2}}}\right)+\nu
\left(
C_{\pi_{\bm{\theta}_k}}(\policy,\beta)-b
\right)
\right\}.
\end{flalign}
In our implementation, we use KL-divergence as the distance $D(\cdot,\cdot)$ to measure the difference between two policies, then
 \begin{flalign}
\label{app-d-measure}
 D\left(\pi_{{\bm{\theta}}},\pi_{{\bm{\theta}}_{k+\frac{1}{2}}}\right)=\E_{s\sim{d}_{{\pi_{\bm{\theta}_k}}}^{\lambda}(\cdot)}\left[\mathtt{KL}\left(\pi_{\bm{\theta}_{k+\frac{1}{2}}},\policy\right)[s]\right],
 \end{flalign}
 which implies we can rewrite the problem (\ref{app-proj-01}) as follows,
 \begin{flalign}
\label{app-proj-02}
\max_{\nu\ge0}\min_{\pi_{\bm{\theta}}}
\left\{
\E_{s\sim{d}_{{\pi_{\bm{\theta}_k}}}^{\lambda}(\cdot)}\left[\mathtt{KL}\left(\pi_{\bm{\theta}_{k+\frac{1}{2}}},\policy\right)[s]\right]+\nu
\left(
C_{\pi_{\bm{\theta}_k}}(\policy,\beta)-b
\right)
\right\}.
\end{flalign}

Furthermore, we update the projection step as follows,
\begin{flalign}
\nonumber
(\pi_{{\bm{\theta}}_{k+1}},\nu_{k+1})=\arg\min_{\pi_{\bm{\theta}}\in\Pi_{\bm{\theta}}}\max_{\nu\ge0}
\left\{
\hat{\mathcal{L}}_{\text{c}}\left(\policy,\pi_{{\bm{\theta}}_k},{\bm{\theta}}_{k+\frac{1}{2}},\nu\right)
\right\},
\end{flalign}
where 
\begin{flalign}
\nonumber
&\hat{\mathcal{L}}_{\text{c}}\left(\policy,\pi_{{\bm{\theta}}_k},{\bm{\theta}}_{k+\frac{1}{2}},\nu\right)=
\hat{D}_{\mathtt{KL}}(\pi_{{\bm{\theta}}_{k+\frac{1}{2}}},\pi_{{\bm{\theta}}})
+\nu \hat{C}(\policy,\pi_{{\bm{\theta}}_k}),\\
\nonumber
\hat{C}(\policy,\pi_{{\bm{\theta}}_k})&=\hat{J}^{C}+\frac{1}{1-\tilde\gamma}\cdot\frac{1}{T}\sum_{t=1}^{T}\frac{\pi_{\bm{\theta}}(a_{t}|s_{t})}{\pi_{{\bm{\theta}}_k}(a_{t}|s_{t})}\hat{A}^{C}_{t}+\beta_k\sqrt{\hat{D}_{\mathtt{KL}}(\pi_{{\bm{\theta}}_k},\pi_{{\bm{\theta}}})}-b,
\end{flalign}
$\hat{J}^{C}$and $\hat{A}^{C}_{t}$ are estimators for cost-return and cost-advantage.
\begin{remark}[Track for Learning $\nu$]
Particularly, after some simple algebra, we obtain the derivation of $\hat{\mathcal{L}}_{\text{c}}(\cdot)$ with respect to $\nu$ as follows,
\begin{flalign}
\dfrac{\partial \hat{\mathcal{L}}_{\text{c}}\left(\policy,\pi_{{\bm{\theta}}_k},{\bm{\theta}}_{k+\frac{1}{2}},\nu\right)}{\partial \nu}=J^{c}(\pi_{{\bm{\theta}}_k})+\dfrac{1}{1-\tilde\gamma}\E_{s\sim{d}_{\pi_{\bm{\theta}_k}}^{\lambda}(\cdot),a\sim\policy(\cdot|s)}
\left[
A^{\mathtt{GAE}(\gamma,\lambda)}_{{\pi_{\bm{\theta}_k}},C}(s,a)\right]-b.
\end{flalign}
But recall (\ref{performance-improvement}) is a minimization-maximization iteration, i.e., 
we require to minimize the distance $\E_{s\sim{d}_{\pi_{\bm{\theta}_k}}^{\lambda}(\cdot)}\mathrm{KL}\left(\pi_{{\bm{\theta}}},\pi_{{\bm{\theta}}_{k}}\right)[s]$, which implies $\policy$ is close to $\pi_{\bm{\theta}_k}$.
Thus it is reasonable to consider \[\E_{s\sim{d}_{\pi_{\bm{\theta}_k}}^{\lambda}(\cdot),a\sim\policy(\cdot|s)}
\left[
A^{\mathtt{GAE}(\gamma,\lambda)}_{{\pi_{\bm{\theta}_k}},C}(s,a)\right]\approx 0.\] 
Thus, in practice, we update $\nu$ following a simple way
\[
\nu\leftarrow\left\{\nu+\eta(J^{c}(\pi_{{\bm{\theta}}_k})-b)\right\}_{+}.
\]
\end{remark}
Finally, we obtain the parameters $({\bm{\theta}}_{k+1},\nu_{k+1})$ as follows,
\begin{flalign}
\label{learning-para-01}
\bm{\theta}_{k+1}&\leftarrow \bm{\theta}_{k}-\eta_2\dfrac{\partial}{\partial \bm{\theta}}\hat{\mathcal{L}}_{\text{c}}\left(\policy,\pi_{{\bm{\theta}}_k},{\bm{\theta}}_{k+\frac{1}{2}},\nu\right)\Big|_{\bm{\theta}=\bm{\theta}_{k},\nu=\nu_k},
\\
\label{learning-para-02}
\nu_{k+1}&\leftarrow \left\{\nu_{k}+\eta_2\dfrac{\partial}{\partial \nu}\hat{\mathcal{L}}_{\text{c}}\left(\policy,\pi_{{\bm{\theta}}_k},{\bm{\theta}}_{k+\frac{1}{2}},\nu\right)\Big|_{\bm{\theta}=\bm{\theta}_{k},\nu=\nu_k}\right\}_{+},
\end{flalign}
where $\{\cdot\}_{+}$ denotes the positive part, i.e., if $x\leq0$, $\{x\}_{+}=0$, else $\{x\}_{+}=x$. We have shown all the details of the implementation in Algorithm \ref{alg-app-cpu}.

\begin{algorithm}[h]
\caption{Constrained Update Projection (CUP)} 
\label{alg-app-cpu}
\begin{algorithmic}
\STATE \textbf{Initialize:} policy network parameters ${\bm{\theta}}_0$; value network parameter $\bm{\omega}_0$; cost value function parameter $\bm{\nu}_0$, step-size $\nu_0$;
\STATE \textbf{Hyper-parameters:} trajectory horizon $T$; discount rate $\gamma$; episode number $M,N$, mini-batch size $B$, positive constant $\alpha,\eta$;
\FOR{$k=0,1,2,\ldots $}
\STATE Collect batch data of $M$ episodes of horizon $T$ in $\cup_{i=1}^{M}\cup_{t=0}^{T}\left\{(s_{i,t},a_{i,t},r_{i,t+1},c_{i,t+1})\right\}$ according to current policy $\pi_{{\bm{\theta}}_k}$;
\STATE Estimate $c$-return by discount averaging on each  episode: $\hat{J}_{i}^{C}=\sum_{t=0}^{T}\gamma^{t}c_{i,t+1};$
\STATE Compute TD errors $\cup_{i=1}^{M}\cup_{t=0}^{T}\{\delta_{i,t}\}$, cost TD errors $\cup_{i=1}^{M}\cup_{t=0}^{T}\{\delta^{C}_{i,t}\}$:
\[
\delta_{i,t}=r_{i,t}+\gamma V_{\bm{\omega}_k}(s_{i,t})- V_{\bm{\omega}_k}(s_{i,t-1}),~\delta^{C}_{i,t}=c_{i,t}+\gamma V^{C}_{\bm{\nu}_k}(s_{i,t})- V^{C}_{\bm{\nu}_k}(s_{i,t-1});
\]
\STATE Compute GAE: $\cup_{i=1}^{M}\cup_{t=0}^{T}\{\hat{A}_{i,t},\hat{A}^{C}_{i,t}\}$: $\hat{A}_{i,t}=\sum_{j=t}^{T}(\gamma\lambda)^{j-t}\delta_{i,j}, ~\hat{A}^{C}_{i,t}=\sum_{j=t}^{T}(\gamma\lambda)^{j-t}\delta^{C}_{i,j};$
\STATE Compute target function for value function and cost value function as follows,
\[
V^{\text{target}}_{i,t}=\hat{A}_{i,t}+ V_{\bm{\omega}_k}(s_{i,t}),~~V^{\text{target},C}_{i,t}=\hat{A}_{i,t}^{C}+ V^{C}_{\bm{\nu}_k}(s_{i,t})
;
\]
\STATE Store data: $\calD_k=\cup_{i=1}^{M}\cup_{t=0}^{T}\left\{(a_{i,t},s_{i,t},\hat{A}_{i,t},\hat{A}^{C}_{i,t},V^{\text{target}}_{i,t},V^{\text{target},C}_{i,t})\right\}$;
\STATE\tikzmark{start1}$\pi_{\text{old}}\leftarrow\pi_{\bm{\theta}_k}$;~~~~~~~~~~~~~~~~~~~~~~~~~~~~~~~~~~~~~~~~~~~~~~~~~~~~~~~~~~~~~~~~~~~~~~~~~~~~~~~~~~~~~~{\color{blue}{\texttt{Policy Improvement}}}
\FOR{$i=0,1,2,\ldots,M$}
\STATE
\[
{\bm{\theta}}_{k+\frac{1}{2}}=\arg\max_{{\bm{\theta}}}\left\{\frac{1}{T}\sum_{t=1}^{T}\frac{\pi_{\bm{\theta}}(a_{i,t}|s_{i,t})}{\pi_{\text{old}}(a_{i,t}|s_{i,t})}\hat{A}_{i,t}-g\left(\epsilon,\dfrac{1}{T}\sum_{t=1}^{T}\hat{A}_{i,t}\right)\right\};
\]
\ENDFOR
\tikzmark{end1}
\STATE \tikzmark{start2}$\pi_{\text{old}}\leftarrow\pi_{\bm{\theta}_{k+\frac{1}{2}}}$;~~~~~~~~~~~~~~~~~~~~~~~~~~~~~~~~~~~~~~~~~~~~~~~~~~~~~~~~~~~~~~~~~~~~~~~~~~~~~~~~~~~~~~~~~~~~{\color{blue}{\texttt{Projection}}}
\STATE$\nu_{k+1}=(\nu_{k}+\eta(\hat{J}_{i}^{C}-b))_{+}$;
\FOR{$i=0,1,2,\ldots,M$}
\STATE
\begin{flalign}
\nonumber
{\bm{\theta}}_{k+1}&=\arg\min_{{\bm{\theta}}} \dfrac{1}{T}\sum_{t=1}^{T}
\left\{
\mathrm{KL}(\pi_{{\bm{\theta}}_{\text{old}}}(\cdot|s_{i,t}),\pi_{{\bm{\theta}}}(\cdot|s_{i,t}))+\nu_k
\dfrac{1-\gamma\lambda}{1-\gamma}\dfrac{\pi_{\bm{\theta}}(a_{i,t}|s_{i,t})}{\pi_{{\bm{\theta}}_k}(a_{i,t}|s_{i,t})}\hat{A}^{C}_{i,t}
\right\};
\end{flalign}
\ENDFOR
\tikzmark{end2}
\FOR{each mini-batch $\{(a_{j},s_{j},\hat{A}_{j},\hat{A}^{C}_{j},V^{\text{target}}_{j},V^{\text{target},C}_{j})\}$ of size $B$ from $\calD_k$}
\STATE
\[
\bm{\omega}_{k+1}=\arg\min_{\bm{\omega}}
\sum_{j=1}^{B}
\left(
V_{\bm{\omega}}(s_j)-V^{\text{target}}_{j}
\right)^2,
\bm{\nu}_{k+1}=\arg\min_{\bm{\nu}}
\sum_{j=1}^{B}
\left(
V^{c}_{\bm{\nu}}(s_j)-V^{\text{target},C}_{j}
\right)^2;
\]
\ENDFOR
\Textbox{start1}{end1}{}
\Textbox{start2}{end2}{}
\ENDFOR
\end{algorithmic}
\end{algorithm}

\clearpage

\section{Preliminaries}

In this section, we introduce some new notations and results about convex optimization, state distribution, policy optimization and $\lambda$-returns.

\subsection{Strong Duality via Slater's Condition}

\label{app-sec-slater-condition}

We consider a convex optimization problem:
\begin{flalign}
p_{\star}&=\min_{x} f_{0}(x),\\
\mathrm{s.t.}~f_{i}(x)&\leq0,~i=1,2,\cdots,m,
\\
h_{i}(x)&=0,~i=1,2,\cdots,p,
\end{flalign}
where the functions $f_0,f_1,\cdots,f_m$ are convex, and $h_1,\cdots,h_p$ are affine. We denote by $\calD$ the domain of the problem (which is the intersection of the domains of all the functions involved), and by $\calX \subset \calD$ its feasible set.

To the problem we associate the Lagrangian $\calL : \R^{n} \times \R^{m} \times \R^{p} \rightarrow \R$, with values
\begin{flalign}
\calL(x,\lambda,\nu)=f_{0}+\sum_{i=1}^{m}\lambda_{i}f_{i}(x)+\sum_{i=1}^{p}\nu_i h_{i}(x).
\end{flalign}
The dual function is $g : \R^{m} \times \R^{p} \rightarrow \R$, with values
\begin{flalign}
g(\lambda,\nu)=\min_{x} \calL(x,\lambda,\nu).
\end{flalign}
The associated dual problem is
\begin{flalign}
d_{\star}=\max_{\lambda\succeq 0,\nu}g(\lambda,\nu).
\end{flalign}

\textbf{Slater's condition}. We say that the problem satisfies Slater’s condition if it is strictly feasible, that is:
\begin{flalign}
\exists x_0~ \in\calD: f_{i}(x_0)<0,i=1,\cdots,m,~h_{i}(x_0)=0,i=1,\cdots,p.
\end{flalign}

\begin{theorem}[Strong duality via Slater condition]
\label{theeorem:strong-duality}
If the primal problem (8.1) is convex, and satisfies the weak Slater’s condition, then strong duality holds, that is, $p_{\star} = d_{\star}$.
\end{theorem}

We omit the proof of Theorem \ref{theeorem:strong-duality}, for more discussions, please refer to \cite[Chapter 5.9]{boyd2004convex}.

\subsection{State Distribution}
We use $\mathbf{P}_{\pi_{\bm \theta}}\in\R^{|\calS|\times|\calS|}$ to denote the state transition matrix by executing $\policy$, and their components are:
\[
\mathbf{P}_{\pi_{\bm \theta}}[s,s'] =\sum_{a\in\mathcal{A}}\pi_{\bm{\theta}}(a|s)\mathbb{P}(s'|s,a)=:\Pro_{\policy}(s^{'}|s),~~s,s^{'}\in\calS,
\]
which denotes one-step state transformation probability from $s$ to $s^{'}$.

We use $\mathbb{P}_{\policy}(s_t=s|s_0)$ to denote the probability of visiting $s$ after $t$
time steps from the initial state $s_0$ by executing $\policy$.
Particularly, we notice if $t=0$, $s_t\ne s_0$, then $\mathbb{P}_{\policy}(s_t=s|s_0)=0$, i.e.,
\begin{flalign}
\label{special-inititial-pro}
\mathbb{P}_{\policy}(s_t=s|s_0)=0,~~ t=0~\text{and}~s\ne s_0.
\end{flalign}
Then for any initial state $s_0\sim\rho(\cdot)$, the following holds,
\begin{flalign}
\label{pro-pi-t-step-app}
\mathbb{P}_{{\policy}}(s_t=s|s_0)=&\sum_{s^{'}\in\mathcal{S}}\mathbb{P}_{{\policy}}(s_t=s|s_{t-1}=s^{'})\mathbb{P}_{{\policy}}(s_{t-1}=s^{'}|s_0).
\end{flalign}

Recall $d_{\pi_{\bm {\theta}}}^{s_0}(s)$ denotes the normalized discounted distribution of the future state $s$ encountered starting at $s_0$ by executing $\pi_{\bm {\theta}}$,
\[d_{\pi_{\bm {\theta}}}^{s_0}(s)=(1-\gamma)\sum_{t=0}^{\infty}\gamma^{t}\mathbb{P}_{\pi_{\bm {\theta}}}(s_t=s|s_0).\]
Furthermore, since $s_0\sim\rho_{0}(\cdot)$, we define
\[
d_{\pi_{\bm {\theta}}}^{\rho_0}(s)=\mathbb{E}_{s_0\sim\rho_{0}(\cdot)}[d_{\pi_{\bm {\theta}}}^{s_0}(s)]=\int_{s_0\in\calS}\rho_{0}(s_0)d^{s_0}_{\pi_{\bm {\theta}}}(s)\text{d}s_0
\]
as the discounted state visitation distribution over the initial distribution $\rho_0(\cdot)$. We use $\bd_{\pi_{\bm {\theta}}}^{\rho_0}\in\R^{|\calS|}$ to store all the normalized discounted state distributions, and its components are:
\[
\bd_{\pi_{\bm {\theta}}}^{\rho_0}[s]=d_{\pi_{\bm {\theta}}}^{\rho_0}(s),~~s\in\calS.
\]

We use $\bm{\rho}_0\in\R^{|\calS|}$ to denote initial state distribution vector, and their components are:
\[
\bm{\rho}_{0}[s]=\rho_{0}(s),~~s\in\calS.
\]
Then, we rewrite $\bd_{\pi_{\bm {\theta}}}^{\rho_0}$ as the following matrix version,
\begin{flalign}
\bd_{\pi_{\bm {\theta}}}^{\rho_0}=(1-\gamma)\sum_{t=0}^{\infty}(\gamma\bP_{\policy})^{t}\bm{\rho}_{0}=(1-\gamma)(\bI-\gamma\bP_{\policy})^{-1}\bm{\rho}_{0}.
\end{flalign}

\subsection{Objective of MDP}

Recall $\tau=\{s_{t}, a_{t}, r_{t+1}\}_{t\ge0}\sim{\pi_{\bm \theta}}$, according to $\tau$,
we define the expected return $J({\pi_{\bm \theta}}|s_0)$ as follows,
\begin{flalign}
\label{Eq:J-theta-app}
  J({\pi_{\bm \theta}}|s_0)=&\mathbb{E}_{\tau\sim\pi_{\bm {\theta}}}[R(\tau)]
    =\dfrac{1}{1-\gamma}\mathbb{E}_{s\sim d_{\pi_{\bm {\theta}}}^{s_0}(\cdot),a\sim\pi_{\bm {\theta}}(\cdot|s),s^{'}\sim\Pro(\cdot|s,a)}\left[r(s^{'}|s,a)\right],
\end{flalign}
where $R(\tau)=\sum_{t\ge0}\gamma^{t}r_{t+1}$, and the notation $J({\pi_{\bm \theta}}|s_0)$ is ``conditional'' on $s_0$ is to emphasize the trajectory $\tau$ starting from $s_0$.

Since $s_0\sim\rho_{0}(\cdot)$, we define the objective of MDP as follows,
\begin{flalign}
\label{app-objective-fun-01}
J(\pi_{\bm {\theta}})
=\dfrac{1}{1-\gamma}\mathbb{E}_{s\sim d_{\pi_{\bm {\theta}}}^{\rho_0}(\cdot),a\sim\pi_{\bm {\theta}}(\cdot|s),s^{'}\sim\Pro(\cdot|s,a)}\left[r(s^{'}|s,a)\right].
\end{flalign}
The goal of reinforcement learning is to solve the following optimization problem:
\begin{flalign}
\label{Eq:thata-optimal-app}
    \bm{\theta}_{\star}=\arg\max_{\bm{\theta}\in\R^{p}} J(\policy).
\end{flalign}

\subsection{$\lambda$-Return}

Let $\mathcal{B}_{\pi_{\bm\theta}}$ be the \emph{Bellman operator}:
 \begin{flalign}
 \label{bellman-op}
\mathcal{B}_{\pi_{\bm\theta}}:  \mathbb{R}^{|\mathcal{S}|}\rightarrow \mathbb{R}^{|\mathcal{S}|},~~~~ v\mapsto \mathbf{r}_{\pi_{\bm \theta}}+\gamma \mathbf{P}_{\pi_{\bm \theta}}v,
 \end{flalign}
 where
$\mathbf{r}_{\pi_{\bm \theta}}\in\R^{|\calS|}$ is the expected reward according to $\pi_{\bm \theta}$, i.e., their components are: \[\mathbf{r}_{\pi_{\bm \theta}}[s] =\sum_{a\in\mathcal{A}}\sum_{s^{'}\in\mathcal{S}}\pi_{\bm{\theta}}(a|s)r(s'|s,a)=:R_{\policy}(s),~~s\in\calS.\]
Let $\bv_{\pi_{\bm{\theta}}}\in\R^{|\calS|}$ be a vector that stores all the state value functions, and its components are:
\[
\bv_{\pi_{\bm{\theta}}}[s]=V_{\pi_{\bm{\theta}}}(s),~~s\in\calS.
\]
Then, according to Bellman operator (\ref{bellman-op}), we rewrite Bellman equation \citep{bellman1957markovian} as the following matrix version:
\begin{flalign}
\label{bellman-eq-matrix}
\mathcal{B}_{\pi_{\bm\theta}}\bv_{\pi_{\bm{\theta}}}=\bv_{\pi_{\bm{\theta}}}.
\end{flalign}

Furthermore, we define $\lambda$-\emph{Bellman operator} $\mathcal{B}^{\lambda}_{\pi_{\bm\theta}}$ as follows,
\[
\mathcal{B}^{\lambda}_{\pi_{\bm\theta}}=(1-\lambda)\sum_{t=0}^{\infty}\lambda^{t} (\mathcal{B}_{\pi_{\bm\theta}})^{{t}+1},
\]
which implies
 \begin{flalign}
 \label{bellman-op-lam}
\mathcal{B}^{\lambda}_{\pi_{\bm\theta}}: \mathbb{R}^{|\mathcal{S}|}\rightarrow \mathbb{R}^{|\mathcal{S}|},~~~~ 
 v\mapsto \mathbf{r}^{(\lambda)}_{\pi_{\bm \theta}}+\tilde{\gamma}\mathbf{P}^{(\lambda)}_{\pi_{\bm \theta}}v,
 \end{flalign}
 where 
 \begin{flalign}
\label{def:matrix-p-lam-return}
 \mathbf{P}^{(\lambda)}_{\pi_{\bm \theta}}=(1-\gamma\lambda)\sum_{{t}=0}^{\infty}(\gamma\lambda)^{{t}}\bP^{{t}+1}_{\policy},~~ \mathbf{r}^{(\lambda)}_{\pi_{\bm \theta}}=\sum_{{t}=0}^{\infty}(\gamma\lambda\bP_{\policy})^{{t}}\mathbf{r}_{\pi_{\bm \theta}},~~\tilde{\gamma}=\dfrac{\gamma(1-\lambda)}{1-\gamma\lambda}.
 \end{flalign}

 Let
 \begin{flalign}
 \label{lam-pro-value}
 \Pro_{\policy}^{(\lambda)}(s^{'}|s)=\mathbf{P}^{(\lambda)}_{\pi_{\bm \theta}}[s,s^{'}]=:(1-\gamma\lambda)\sum_{{t}=0}^{\infty}(\gamma\lambda)^{{t}}\left(\bP^{{t}+1}_{\policy}[s,s^{'}]\right),
 \end{flalign}
 where $\bP^{{t}+1}_{\policy}[s,s^{'}]$ is the $(s,s^{'})$-th component of matrix $\bP^{{t}+1}_{\policy}$, which is the probability of visiting $s^{'}$ after $t+1$ time steps from
 the state $s$ by executing $\policy$, i.e.,
  \begin{flalign}
 \label{lam-pro-value-01}
 \bP^{{t}+1}_{\policy}[s,s^{'}]= \Pro_{\policy}(s_{t+1}=s^{'}|s).
\end{flalign}
Thus, we rewrite  $\Pro_{\policy}^{(\lambda)}(s^{'}|s)$ (\ref{lam-pro-value}) as follows
 \begin{flalign}
 \label{lam-pro-value-02}
 \Pro_{\policy}^{(\lambda)}(s^{'}|s)=(1-\gamma\lambda)\sum_{{t}=0}^{\infty}(\gamma\lambda)^{{t}}\Pro_{\policy}(s_{t+1}=s^{'}|s),~~s\in\calS.
 \end{flalign}
 \begin{remark}
 Furthermore,
recall the following visitation sequence $\tau=\{s_{t}, a_{t}, r_{t+1}\}_{t\ge0}$ induced by $\policy$,
it is similar to the probability $\mathbb{P}_{{\policy}}(s_t=s^{'}|s_0)$, we introduce $\mathbb{P}^{(\lambda)}_{{\policy}}(s_t=s^{'}|s_0)$
as the probability of  transition from state $s$ to state $s^{'}$after $t$
time steps under the dynamic transformation matrix $ \mathbf{P}^{(\lambda)}_{\pi_{\bm \theta}}$.
Then, the following equity holds
\begin{flalign}
\label{pro-pi-t-step}
\mathbb{P}_{{\policy}}^{(\lambda)}(s_t=s|s_0)=&\sum_{s^{'}\in\mathcal{S}}\mathbb{P}_{{\policy}}^{(\lambda)}(s_t=s|s_{t-1}=s^{'})\mathbb{P}_{{\policy}}^{(\lambda)}(s_{t-1}=s^{'}|s_0).
\end{flalign}
\end{remark}
 Similarly, let
  \begin{flalign}
  \nonumber
 R^{(\lambda)}_{\pi_{\bm \theta}}(s)=:
  \mathbf{r}^{(\lambda)}_{\pi_{\bm \theta}}[s]=&\sum_{{t}=0}^{\infty}(\gamma\lambda\bP_{\policy})^{{t}}\mathbf{r}_{\pi_{\bm \theta}}[s]
  =  \sum_{{t}=0}^{\infty}(\gamma\lambda)^{t}\left(\sum_{s^{'}\in\calS}\Pro_{\policy}(s_{t}=s^{'}|s)R_{\policy}(s^{'})\right)\\
   \label{lam-pro-value-03}
   =&
     \sum_{{t}=0}^{\infty}\sum_{s^{'}\in\calS}(\gamma\lambda)^{t}\Pro_{\policy}(s_{t}=s^{'}|s)R_{\policy}(s^{'}).
 \end{flalign}

It is similar to normalized discounted distribution $d_{\pi_{\bm {\theta}}}^{\rho_0}(s)$, we introduce $\lambda$-return version of discounted state distribution $d_{\pi_{\bm {\theta}}}^{\lambda}(s)$ as follows: $\forall s\in\calS$,
\begin{flalign}
\label{lambda-dis-state-distribution}
d_{\pi_{\bm {\theta}}}^{s_0,\lambda}(s)&=(1-\tilde \gamma)\sum_{t=0}^{\infty}\tilde{\gamma}^{t}\mathbb{P}^{(\lambda)}_{\pi_{\bm {\theta}}}(s_t=s|s_0),\\
d_{\pi_{\bm {\theta}}}^{\lambda}(s)&=\E_{s_0\sim\rho_{0}(\cdot)}\left[d_{\pi_{\bm {\theta}}}^{s_0,\lambda}(s)\right],\\
\label{mat-lambda-dis-state-distribution}
\bd_{\pi_{\bm {\theta}}}^{\lambda}[s]&=d_{\pi_{\bm {\theta}}}^{\lambda}(s),
\end{flalign}
where $\mathbb{P}^{(\lambda)}_{\pi_{\bm {\theta}}}(s_t=s|s_0)$ is the $(s_0,s)$-th component of the matrix $\left(\mathbf{P}^{(\lambda)}_{\pi_{\bm \theta}}\right)^{t}$, i.e.,
\[
\mathbb{P}^{(\lambda)}_{\pi_{\bm {\theta}}}(s_t=s|s_0)=:\left(\mathbf{P}^{(\lambda)}_{\pi_{\bm \theta}}\right)^{t}[s_0,s].
\]
Similarly, $\mathbb{P}^{(\lambda)}_{\pi_{\bm {\theta}}}(s_t=s^{'}|s)$ is the $(s,s^{'})$-th component of the matrix $\left(\mathbf{P}^{(\lambda)}_{\pi_{\bm \theta}}\right)^{t}$, i.e.,
\[
\mathbb{P}^{(\lambda)}_{\pi_{\bm {\theta}}}(s_t=s^{'}|s)=:\left(\mathbf{P}^{(\lambda)}_{\pi_{\bm \theta}}\right)^{t}[s,s^{'}].
\]

Finally, we rewrite $\bd_{\pi_{\bm {\theta}}}^{\rho_0,\lambda}$ as the following matrix version,
\begin{flalign}
\label{matrixversion-lambda-dis-state-distribution}
\bd_{\pi_{\bm {\theta}}}^{\lambda}=(1-\tilde \gamma)\sum_{t=0}^{\infty}\left(\gamma\bP^{(\lambda)}_{\policy}\right)^{t}\bm{\rho}_{0}=(1-\tilde \gamma)\left(\bI-\tilde{\gamma}\bP^{(\lambda)}_{\policy}\right)^{-1}\bm{\rho}_{0}.
\end{flalign}

 \begin{remark}[$\lambda$-Return Version of Bellman Equation]
 According to Bellman equation (\ref{bellman-eq-matrix}), $\bv_{\pi_{\bm{\theta}}}$ is fixed point of $\lambda$-operator $\mathcal{B}^{\lambda}_{\pi_{\bm\theta}}$, i.e.,
 \begin{flalign}
 \label{bellman-eq-return}
 \bv_{\pi_{\bm{\theta}}}=\mathbf{r}^{(\lambda)}_{\pi_{\bm \theta}}+{\tilde{\gamma}} \mathbf{P}^{(\lambda)}_{\pi_{\bm \theta}}\bv_{\pi_{\bm{\theta}}}.
 \end{flalign}
 Recall $\tau=\{s_t,a_t,r_{t+1}\}_{t\ge0}\sim\policy$, according to (\ref{bellman-eq-return}), the value function of initial state $s_0$ is
 \begin{flalign}
  \nonumber
  V_{\policy}(s_0)&= \bv_{\pi_{\bm{\theta}}}[s_0]=\mathbf{r}^{(\lambda)}_{\pi_{\bm \theta}}[s_0]+\tilde \gamma \mathbf{P}^{(\lambda)}_{\pi_{\bm \theta}}\bv_{\pi_{\bm{\theta}}}[s_0]\\
  \label{bellman-eq-1}
 & =R^{(\lambda)}_{\pi_{\bm \theta}}(s_0)+{\tilde{\gamma}}\sum_{s^{'}\in\calS} \Pro_{\policy}^{(\lambda)}(s_1=s^{'}|s_0)V_{\policy}(s^{'}).
 \end{flalign}
 \end{remark}

We unroll the expression of (\ref{bellman-eq-1}) repeatedly, then we have
\begin{flalign}
\nonumber
&V_{{\policy}}(s_0)
\\
\nonumber
=&{R}^{(\lambda)}_{\policy}(s_0)+{\tilde{\gamma}}\sum_{s^{'}\in\mathcal{S}}\mathbb{P}_{{\policy}}^{(\lambda)}(s_1=s^{'}|s_0)
\underbrace{\left(
{R}^{(\lambda)}_{\policy}(s^{'})
+{\tilde{\gamma}}\sum_{s^{''}\in\mathcal{S}}\mathbb{P}_{{\policy}}^{(\lambda)}(s_2=s^{''}|s_1=s^{'})V_{{\policy}}(s^{''})
\right)}_{=V_{{\policy}}(s^{'})}\\
\nonumber
=&{R}^{(\lambda)}_{\policy}(s_0)+{\tilde{\gamma}}\sum_{s^{'}\in\mathcal{S}}\mathbb{P}_{{\policy}}^{(\lambda)}(s_1=s^{'}|s_0){R}^{(\lambda)}_{\policy}(s^{'})\\
\nonumber
&~~~~~~~~~~~~~+{\tilde{\gamma}}^2\sum_{s^{''}\in\mathcal{S}}\underbrace{
\left(
\sum_{s^{'}\in\mathcal{S}}\mathbb{P}_{{\policy}}^{(\lambda)}(s_1=s^{'}|s_0)
\mathbb{P}_{{\policy}}^{(\lambda)}(s_2=s^{''}|s_1=s^{'})
\right)
}_{\overset{(\ref{pro-pi-t-step})}=:\mathbb{P}^{(\lambda)}_{{\policy}}\left(s_2=s^{''}|s_0\right)}V_{{\policy}}(s^{''})\\
\nonumber
=&{R}^{(\lambda)}_{\policy}(s_0)+{\tilde{\gamma}}\sum_{s\in\mathcal{S}}\mathbb{P}_{{\policy}}^{(\lambda)}(s_1=s|s_0){R}^{(\lambda)}_{\policy}(s)
+{\tilde{\gamma}}^2\sum_{s\in\mathcal{S}}\mathbb{P}_{{\policy}}^{(\lambda)}(s_2=s|s_0)V_{{\policy}}(s)\\
\nonumber
=&{R}^{(\lambda)}_{\policy}(s_0)+{\tilde{\gamma}}\sum_{s\in\mathcal{S}}\mathbb{P}_{{\policy}}^{(\lambda)}(s_1=s|s_0){R}^{(\lambda)}_{\policy}(s)\\
\nonumber
&~~~~~~~~~~~~~+{\tilde{\gamma}}^2\sum_{s\in\mathcal{S}}\mathbb{P}_{{\policy}}^{(\lambda)}(s_2=s|s_0)
\left(
{R}^{(\lambda)}_{\policy}(s)+{\tilde{\gamma}}\sum_{s^{'}\in\mathcal{S}}\mathbb{P}_{{\policy}}^{(\lambda)}(s_3=s^{'}|s_2=s)V_{{\policy}}(s^{'})
\right)\\
\nonumber
=&{R}^{(\lambda)}_{\policy}(s_0)+{\tilde{\gamma}}\sum_{s\in\mathcal{S}}\mathbb{P}_{{\policy}}^{(\lambda)}(s_1=s|s_0){R}^{(\lambda)}_{\policy}(s)+{\tilde{\gamma}}^2\sum_{s\in\mathcal{S}}\mathbb{P}_{{\policy}}^{(\lambda)}(s_2=s|s_0){R}^{(\lambda)}_{\policy}(s)\\
\nonumber
&~~~~~~~~~~~~~+{\tilde{\gamma}}^3\sum_{s^{'}\in\mathcal{S}}\underbrace{
\left(
\sum_{s\in\mathcal{S}}\mathbb{P}_{{\policy}}^{(\lambda)}(s_2=s|s_0)
\mathbb{P}_{{\policy}}^{(\lambda)}(s_3=s^{'}|s_2=s)
\right)
}_{=\mathbb{P}_{{\policy}}^{(\lambda)}(s_3=s^{'}|s_0)}V_{{\policy}}(s^{'})\\
\nonumber
=&{R}^{(\lambda)}_{}(s_0)+{\tilde{\gamma}}\sum_{s\in\mathcal{S}}\mathbb{P}_{{\policy}}^{(\lambda)}(s_1=s|s_0){R}^{(\lambda)}_{\policy}(s)+{\tilde{\gamma}}^2\sum_{s\in\mathcal{S}}\mathbb{P}_{{\policy}}^{(\lambda)}(s_2=s|s_0){R}^{(\lambda)}_{\policy}(s)\\
\nonumber
&~~~~~~~~~~~~~+{\tilde{\gamma}}^3\sum_{s\in\mathcal{S}}\mathbb{P}^{(\lambda)}_{{\policy}}(s_3=s|s_0)V_{{\policy}}(s)\\
\nonumber
=&\cdots
\\
\label{re-bellman-eq-01}
=&\sum_{s\in\calS}\sum_{t=0}^{\infty}{\tilde{\gamma}}^{t}\mathbb{P}_{{\policy}}^{(\lambda)}(s_t=s|s_0){R}^{(\lambda)}_{\policy}(s)
\overset{(\ref{lambda-dis-state-distribution})}=\dfrac{1}{1-{\tilde{\gamma}}}\sum_{s\in\calS}d^{s_0,\lambda}_{\policy}(s) {R}^{(\lambda)}_{\policy}(s)
.
\end{flalign}

According to (\ref{Eq:J-theta-app}) and (\ref{re-bellman-eq-01}), we have
\begin{flalign}
\nonumber
J({\policy})=&\sum_{s_0\in\calS}\rho_{0}(s_0)V_{\policy}(s_0)\overset{(\ref{re-bellman-eq-01})}=
\dfrac{1}{1-{\tilde{\gamma}}}\sum_{s_0\in\calS}\rho_{0}(s_0)\sum_{s\in\calS}d^{s_0,\lambda}_{\policy}(s) R^{(\lambda)}_{\pi_{\bm \theta}}(s)\\
\nonumber
=&\dfrac{1}{1-{\tilde{\gamma}}}\sum_{s\in\calS}\underbrace{\left(\sum_{s_0\in\calS}\rho_{0}(s_0)d^{s_0,\lambda}_{\policy}(s)\right)}_{=d^{\lambda}_{\policy}(s)} R^{(\lambda)}_{\pi_{\bm \theta}}(s)\\
\label{lam-return-objective}
=&\dfrac{1}{1-{\tilde{\gamma}}}\sum_{s\in\calS}d^{\lambda}_{\policy}(s)R^{(\lambda)}_{\pi_{\bm \theta}}(s)
=
\dfrac{1}{1-{\tilde{\gamma}}}\E_{s\sim d^{\lambda}_{\policy}(\cdot)}
\left[R^{(\lambda)}_{\pi_{\bm \theta}}(s)\right]
.
\end{flalign}

Finally, we summarize above results in the following Lemma \ref{lem:lam-return-objective}.
\begin{lemma}
\label{lem:lam-return-objective}
The objective $J(\policy)$ (\ref{app-objective-fun-01}) can be rewritten as the following version:
\[
J({\policy})=\dfrac{1}{1-{\tilde{\gamma}}}\sum_{s\in\calS}d^{\lambda}_{\policy}(s)R^{(\lambda)}_{\pi_{\bm \theta}}(s)=
\dfrac{1}{1-{\tilde{\gamma}}}\E_{s\sim d^{\lambda}_{\policy}(\cdot)}
\left[R^{(\lambda)}_{\pi_{\bm \theta}}(s)\right].
\]
\end{lemma}

\clearpage

\section{Proof of Theorem \ref{them:general-performance-difference}}

\label{sec:proof-them-01}

We need the following Proposition \ref{objective-td-error-version} to prove Theorem \ref{them:general-performance-difference}, which illustrates an identity for the objective function of policy optimization.

\begin{proposition}
\label{objective-td-error-version}
For any function $\varphi(\cdot):\calS\rightarrow\R$, for any policy $\policy$, for any trajectory satisfies $\tau=\{s_{t}, a_{t}, r_{t+1}\}_{t\ge0}\sim{\pi_{\bm \theta}}$,
let 
\begin{flalign}
\nonumber
\delta_t^{\varphi}&=r(s_{t+1}|s_t,a_t)+\gamma\varphi(s_{t+1})-\varphi(s_{t}),
\\
\nonumber
\delta^{\varphi}_{\policy,t}(s)&=\E_{s_{t}\sim\Pro_{\policy}(\cdot|s),a_{t}\sim{\policy}(\cdot|s_t),s_{t+1}\sim\Pro(\cdot|s_t,a_t)}\left[\delta_t^{\varphi}\right],
\end{flalign}
then, the objective $J(\policy)$ (\ref{lam-return-objective}) can be rewritten as the following version:
\begin{flalign}
\label{lam-return-phi-objective-prop}
J(\policy)=&\E_{s_0\sim\rho_{0}(\cdot)}[\varphi(s_0)]
+
\dfrac{1}{1-\tilde\gamma}\sum_{s\in\calS}d^{\lambda}_{\policy}(s)
\left(
\sum_{t=0}^{\infty}\gamma^t \lambda^t
\delta^{\varphi}_{\policy,t}(s)
\right)
\\
\nonumber
=&\E_{s_0\sim\rho_{0}(\cdot)}[\varphi(s_0)]
+
\dfrac{1}{1-\tilde\gamma}\E_{s\sim d^{\lambda}_{\policy}(\cdot)}
\left[
\sum_{t=0}^{\infty}\gamma^t \lambda^t
\delta^{\varphi}_{\policy,t}(s)
\right]
.
\end{flalign}
\end{proposition}
We present the proof of of Proposition \ref{objective-td-error-version} at the end of this section, see Section \ref{proof-pro-01-app}.

We introduce a vector $\bm{\delta}^{\varphi}_{\policy,t}\in\R^{|\calS|}$ and its components are: for any $s\in\calS$
\begin{flalign}
\label{revist-td-ex-error}
\bm{\delta}^{\varphi}_{\policy,t}[s]={{\delta}}^{\varphi}_{\policy,t}(s).
\end{flalign}
Then, we rewrite the objective as the following vector version
\begin{flalign}
\label{lam-return-phi-objective-vec-version}
J(\policy)=\E_{s_0\sim\rho_{0}(\cdot)}[\varphi(s_0)]
+
\dfrac{1}{1-\tilde \gamma}\sum_{t=0}^{\infty}\gamma^t \lambda^t
\langle
\bd_{\pi_{\bm {\theta}}}^{\lambda},\bm{\delta}^{\varphi}_{\policy,t}
\rangle,
\end{flalign}
where $\langle\cdot,\cdot\rangle$ denotes inner production between two vectors.

\subsection{Proof of Theorem \ref{them:general-performance-difference}}

\textbf{Theorem \ref{them:general-performance-difference}} 
(Generalized Policy Performance Difference)
\emph{
For any function $\varphi(\cdot):\calS\rightarrow\R$, for two arbitrary policy $\pi_{\bm\theta}$ and $\pi_{{\bm\theta}^{'}}$,  
for any $p,q\in[1,\infty)$ such that $\frac{1}{p}+\frac{1}{q}=1$,  
The following bound holds:
\begin{flalign} 
\label{bound-diff}
\dfrac{1}{1-\tilde \gamma}\sum_{t=0}^{\infty}\gamma^t\lambda^{t} M^{\varphi,-}_{p,q,t}(\policy,\policyy)
\leq J(\pi_{\bm \theta})-J(\pi_{{\bm \theta}^{'}})
\leq\dfrac{1}{1-\tilde \gamma}\sum_{t=0}^{\infty}\gamma^t \lambda^t  M^{\varphi,+}_{p,q,t}(\policy,\policyy),
\end{flalign}
where the terms $M^{\varphi,-}_{p,q,t}$ and $M^{\varphi,+}_{p,q,t}$ are defined in (\ref{app-term-01})-(\ref{app-term-02}).
}

\begin{proof} (of Theorem \ref{them:general-performance-difference})

We consider two arbitrary policies $\pi_{\bm\theta}$ and $\pi_{{\bm\theta}^{'}}$ with different parameters $\bm{\theta}$ and $\bm{\theta}^{'}$, let
\begin{flalign}
\label{dfifference-two-polic}
D_{t}^{\varphi,(\lambda)}(\policy,\policyy)=:
\langle
\bd_{\pi_{\bm {\theta}}}^{\lambda},\bm{\delta}^{\varphi}_{\policy,t}
\rangle
-
\langle
\bd_{\policyy}^{\lambda},\bm{\delta}^{\varphi}_{\policyy,t}
\rangle.
\end{flalign}
According to (\ref{lam-return-phi-objective-vec-version}), we obtain performance difference as follows,
\begin{flalign} 
\nonumber
J(\pi_{\bm \theta})-J(\pi_{{\bm \theta}^{'}})=&\dfrac{1}{1-\tilde \gamma}\sum_{t=0}^{\infty}\gamma^t \lambda^t
\left(
\langle
\bd_{\pi_{\bm {\theta}}}^{\lambda},\bm{\delta}^{\varphi}_{\policy,t}
\rangle
-
\langle
\bd_{\policyy}^{\lambda},\bm{\delta}^{\varphi}_{\policyy,t}
\rangle
\right)
\\
\label{objective-difference-01-app}
=&\dfrac{1}{1-\tilde\gamma}\sum_{t=0}^{\infty}\gamma^t \lambda^tD_{t}^{\varphi,(\lambda)}(\policy,\policyy),
\end{flalign}
which requires us to consider the boundedness of the difference $D_{t}^{\varphi,(\lambda)}(\policy,\policyy)$ (\ref{dfifference-two-polic}) .

{
\color{blue}
{
\textbf{Step 1: Bound the term $D_{t}^{\varphi,(\lambda)}(\policy,\policyy)$ (\ref{dfifference-two-polic}).}
}
}

We rewrite the first term of (\ref{dfifference-two-polic}) as follows,
\begin{flalign}
\label{td-error-ex-04}
\langle \bd_{\pi_{\bm {\theta}}}^{\lambda},{\bm{\delta}}^{\varphi}_{\policy,t}\rangle
=
\langle \bd_{\policyy}^{\lambda},{\bm{\delta}}^{\varphi}_{\policy,t}\rangle
+\langle \bd_{\pi_{\bm {\theta}}}^{\lambda}-\bd_{\policyy}^{\lambda},{\bm{\delta}}^{\varphi}_{\policy,t}\rangle,
\end{flalign}
which is bounded by applying H{\"o}lder's inequality to the term $\langle \bd_{\pi_{\bm {\theta}}}^{\lambda}-\bd_{\policyy}^{\lambda},{\bm{\delta}}^{\varphi}_{\policy,t}\rangle$, we rewrite (\ref{td-error-ex-04}) as follows,
\begin{flalign}
\nonumber
&\langle \bd_{\policyy}^{\lambda},{\bm{\delta}}^{\varphi}_{\policy,t}\rangle-\|\bd_{\pi_{\bm {\theta}}}^{\lambda}-\bd_{\policyy}^{\lambda}\|_{p}\|{\bm{\delta}}^{\varphi}_{\policy,t}\|_{q}
\\
\label{them:ineq-01}
\leq&
\langle \bd_{\pi_{\bm {\theta}}}^{\lambda},{\bm{\delta}}^{\varphi}_{\policy,t}\rangle
\leq
\langle \bd_{\policyy}^{\lambda},{\bm{\delta}}^{\varphi}_{\policy,t}\rangle
+
\|\bd_{\pi_{\bm {\theta}}}^{\lambda}-\bd_{\policyy}^{\lambda}\|_{p}\|{\bm{\delta}}^{\varphi}_{\policy,t}\|_{q}
,
\end{flalign}
where $p,q\in[1,\infty)$ and $\frac{1}{p}+\frac{1}{q}=1$. Let
\[
\epsilon^{\varphi,(\lambda)}_{p,q,t}(\policy,\policyy)=:\|\bd_{\pi_{\bm {\theta}}}^{\lambda}-\bd_{\policyy}^{\lambda}\|_{p}\|{\bm{\delta}}^{\varphi}_{\policy,t}\|_{q},
\]
then we rewrite Eq.(\ref{them:ineq-01}) as follows, 
\begin{flalign}
\label{them:ineq-01-01}
\langle \bd_{\policyy}^{\lambda},{\bm{\delta}}^{\varphi}_{\policy,t}\rangle-\epsilon^{\varphi,(\lambda)}_{p,q,t}(\policy,\policyy)
\leq
\langle \bd_{\pi_{\bm {\theta}}}^{\lambda},{\bm{\delta}}^{\varphi}_{\policy,t}\rangle
\leq
\langle \bd_{\policyy}^{\lambda},{\bm{\delta}}^{\varphi}_{\policy,t}\rangle
+
\epsilon^{\varphi,(\lambda)}_{p,q,t}(\policy,\policyy).
\end{flalign}
Let
\begin{flalign}
\label{def:l-t}
M_{t}^{\varphi}(\policy,\policyy)=:
\underbrace{
\langle \bd_{\policyy}^{\lambda},{\bm{\delta}}^{\varphi}_{\policy,t}\rangle
}_{\text{Term-I}}
-
\underbrace{
\langle
\bd_{\policyy}^{\lambda},\bm{\delta}^{\varphi}_{\policyy,t}
\rangle
}_{\text{Term-II}}
,
\end{flalign}
combining the (\ref{dfifference-two-polic}) and (\ref{them:ineq-01-01}), we achieve the boundedness of $D_{t}^{\varphi}(\policy,\policyy)$ as follows
\begin{flalign}
\label{them:ineq-01-02}
M_{t}^{\varphi}(\policy,\policyy)-\epsilon^{\varphi,(\lambda)}_{p,q,t}(\policy,\policyy)
\leq
D_{t}^{\varphi}(\policy,\policyy)
\leq
M_{t}^{\varphi}(\policy,\policyy)+\epsilon^{\varphi,(\lambda)}_{p,q,t}(\policy,\policyy).
\end{flalign}

{
\color{blue}
{
\textbf{Step 2: Analyze the term $M_{t}^{\varphi}(\policy,\policyy)$ (\ref{def:l-t}).}
}
}

To analyze (\ref{them:ineq-01-02}) further, we need to consider the first term appears in $M_{t}^{\varphi}(\policy,\policyy)$ (\ref{def:l-t}):
\begin{flalign}
\nonumber
\text{Term-I}~(\ref{def:l-t})=&
\langle \bd_{\policyy}^{\lambda},{\bm{\delta}}^{\varphi}_{\policy,t}\rangle\\
\label{revist-inner}
=&
\sum_{s\in\calS}d_{\policyy}^{\lambda}(s){{\delta}}^{\varphi}_{\policy,t}(s)
=\E_{s\sim d_{\policyy}^{\lambda}(\cdot)}\left[{{\delta}}^{\varphi}_{\policy,t}(s)\right]
\\
\label{revist-inner-01}
\overset{(\ref{revist-td-ex-error})}=&\E_{s\sim d_{\policyy}^{\lambda}(\cdot)}
\left[\E_{s_{t}\sim\Pro_{\policy}(\cdot|s)}[
\delta^{\varphi}_{\policy}(s_t)]\right].
\end{flalign}
We notice the following relationship
\begin{flalign}
\label{importance-sam-re}
\delta^{\varphi}_{\policy,t}(s)=
&\underset{\begin{subarray}{c} s_{t}\sim\Pro_{\policy}(\cdot|s)\\a_{t}\sim{\policy}(\cdot|s_t)\\ s_{t+1}\sim\Pro(\cdot|s_t,a_t) \end{subarray}}\E\left[\delta_t^{\varphi}\right]
=\underset{\begin{subarray}{c}  s_{t}\sim\Pro_{\policyy}(\cdot|s)\\a_{t}\sim{\policyy}(\cdot|s_t)\\ s_{t+1}\sim\Pro(\cdot|s_t,a_t) \end{subarray}}\E\left[\dfrac{\policy(a_t|s_t)}{\policyy(a_t|s_t)}\delta_t^{\varphi}\right],
\end{flalign}
which holds since we use importance sampling: for any distribution $p(\cdot)$ and $q(\cdot)$, for any random variable function $f(\cdot)$,
\[
\E_{x\sim p(x)}[f(x)]=\E_{x\sim q(x)}\left[\dfrac{p(x)}{q(x)}f(x)\right].
\]
According to (\ref{revist-inner}), (\ref{importance-sam-re}), we rewrite the term $\langle \bd_{\policyy}^{\lambda},{\bm{\delta}}^{\varphi}_{\policy,t}\rangle$ in Eq.(\ref{def:l-t}) as follows,
\begin{flalign}
\label{app-ex-td-01}
\text{Term-I}~(\ref{def:l-t})=
\langle \bd_{\policyy}^{\lambda},{\bm{\delta}}^{\varphi}_{\policy,t}\rangle
=
\sum_{s\in\calS}d_{\policyy}^{\lambda}(s)\left(\underset{\begin{subarray}{c} s_t \sim \Pro_{\policyy}(\cdot|s)\\ a_{t}\sim{\policyy}(\cdot|s_t)\\ s_{t+1}\sim\Pro(\cdot|s_t,a_t) \end{subarray}}\E\left[\dfrac{\policy(a_t|s_t)}{\policyy(a_t|s_t)}\delta_t^{\varphi}\right]\right).
\end{flalign}

Now, we consider the second term appears in $M_{t}^{\varphi}(\policy,\policyy)$ (\ref{def:l-t}):
\begin{flalign}
\nonumber
\text{Term-II}~(\ref{def:l-t})&=
\langle
\bd_{\policyy}^{\lambda},\bm{\delta}^{\varphi}_{\policyy,t}
\rangle\\
\label{app-ex-td-02}
&=\sum_{s\in\calS}d_{\policyy}^{\lambda}(s){\delta}^{\varphi}_{\policyy,t}(s)
=\sum_{s\in\calS}d_{\policyy}^{\lambda}(s)\left(\underset{\begin{subarray}{c} s_t \sim \Pro_{\policyy}(\cdot|s)\\ a_{t}\sim{\policyy}(\cdot|s_t)\\ s_{t+1}\sim\Pro(\cdot|s_t,a_t) \end{subarray}}\E\left[\delta_t^{\varphi}\right]\right).
\end{flalign}

Finally, take the results (\ref{app-ex-td-01}) and (\ref{app-ex-td-02}) to (\ref{def:l-t}),
we obtain the difference between $\langle \bd_{\policyy}^{\lambda},{\bm{\delta}}^{\varphi}_{\policy,t}\rangle$ and $\langle\bd_{\policyy}^{\lambda},\bm{\delta}^{\varphi}_{\policyy,t}\rangle$, i.e., we achieve a identity for $M_{t}^{\varphi}(\policy,\policyy)$ (\ref{def:l-t}) as follows,
\begin{flalign}
\nonumber
M_{t}^{\varphi}(\policy,\policyy)\overset{(\ref{def:l-t})}=&
\langle \bd_{\policyy}^{\lambda},{\bm{\delta}}^{\varphi}_{\policy,t}\rangle-
\langle
\bd_{\policyy}^{\lambda},\bm{\delta}^{\varphi}_{\policyy,t}
\rangle
\\
\label{diff-01}
\overset{(\ref{app-ex-td-01},(\ref{app-ex-td-02})}=&
\sum_{s\in\calS}d_{\policyy}^{\lambda}(s)\left(\underset{\begin{subarray}{c} s_t \sim \Pro_{\policyy}(\cdot|s)\\ a_{t}\sim{\policyy}(\cdot|s_t)\\ s_{t+1}\sim\Pro(\cdot|s_t,a_t) \end{subarray}}\E\left[\left(\dfrac{\policy(a_t|s_t)}{\policyy(a_t|s_t)}-1\right)\delta_t^{\varphi}\right]\right).
\end{flalign}
To simplify expression, we introduce a notation as follows,
\begin{flalign}
\Delta_{t}^{\varphi}(\policy,\policyy,s)&=:\underset{\begin{subarray}{c} s_t \sim \Pro_{\policyy}(\cdot|s)\\ a_{t}\sim{\policyy}(\cdot|s_t)\\ s_{t+1}\sim\Pro(\cdot|s_t,a_t) \end{subarray}}\E\left[\left(\dfrac{\policy(a_t|s_t)}{\policyy(a_t|s_t)}-1\right)\delta_t^{\varphi}\right],
\end{flalign}
and we use a vector $\bm{\Delta}_{t}^{\varphi}(\policy,\policyy)\in\R^{|\calS|}$ to store all the values $\{\Delta_{t}^{\varphi}(\policy,\policyy,s)\}_{s\in\calS}$:
\[
\bm{\Delta}_{t}^{\varphi}(\policy,\policyy)[s]=\Delta_{t}^{\varphi}(\policy,\policyy,s).
\]
Then we rewrite $\langle \bd_{\policyy}^{\lambda},{\bm{\delta}}^{\varphi}_{\policy,t}\rangle-\langle\bd_{\policyy}^{\lambda},\bm{\delta}^{\varphi}_{\policyy,t}\rangle$ (\ref{diff-01}) as follows,
\begin{flalign}
\nonumber
M_{t}^{\varphi}(\policy,\policyy)=&
\langle \bd_{\policyy}^{\lambda},{\bm{\delta}}^{\varphi}_{\policy,t}\rangle-
\langle
\bd_{\policyy}^{\lambda},\bm{\delta}^{\varphi}_{\policyy,t}
\rangle
\\
\nonumber
\overset{(\ref{diff-01})}=&
\sum_{s\in\calS}d_{\policyy}^{\lambda}(s)\Delta_{t}^{\varphi}(\policy,\policyy,s)
=
\langle \bd_{\policyy}^{\lambda},\bm{\Delta}_{t}^{\varphi}(\policy,\policyy)\rangle
.
\end{flalign}

{
\color{blue}
{
\textbf{Step 3: Bound on $J(\policy)-J(\policyy)$.}
}
}

Recall (\ref{them:ineq-01-02}), taking above result in it, we obtain
\begin{flalign}
\label{them:ineq-01-03}
\langle \bd_{\policyy}^{\lambda},\bm{\Delta}_{t}^{\varphi}(\policy,\policyy)\rangle-\epsilon^{\varphi,(\lambda)}_{p,q,t}(\policy,\policyy)
\leq
D_{t}^{\varphi}(\policy,\policyy)
\leq
\langle \bd_{\policyy}^{\lambda},\bm{\Delta}_{t}^{\varphi}(\policy,\policyy)\rangle+\epsilon^{\varphi,(\lambda)}_{p,q,t}(\policy,\policyy).
\end{flalign}

Finally, let
\begin{flalign}
\label{app-term-01}
&M^{\varphi,-}_{p,q,t}(\policy,\policyy)=
\langle \bd_{\policyy}^{\lambda},\bm{\Delta}_{t}^{\varphi}(\policy,\policyy)\rangle-\epsilon^{\varphi,(\lambda)}_{p,q,t}(\policy,\policyy)\\
\nonumber
=&\sum_{s\in\calS}d_{\policyy}^{\lambda}(s)\left(\underset{\begin{subarray}{c} s_t \sim \Pro_{\policyy}(\cdot|s)\\ a_{t}\sim{\policyy}(\cdot|s_t)\\ s_{t+1}\sim\Pro(\cdot|s_t,a_t) \end{subarray}}\E\left[\left(\dfrac{\policy(a_t|s_t)}{\policyy(a_t|s_t)}-1\right)\delta_t^{\varphi}\right]\right)
-
\|\bd_{\pi_{\bm {\theta}}}^{\lambda}-\bd_{\policyy}^{\lambda}\|_{p}\|{\bm{\delta}}^{\varphi}_{\policy,t}\|_{q}
\\
\nonumber
=&
\E_{s\sim{d}_{\policyy}^{\lambda}(\cdot)}
\left[
\underset{\begin{subarray}{c} s_t \sim \Pro_{\policyy}(\cdot|s)\\ a_{t}\sim{\policyy}(\cdot|s_t)\\ s_{t+1}\sim\Pro(\cdot|s_t,a_t) \end{subarray}}\E\left[\left(\dfrac{\policy(a_t|s_t)}{\policyy(a_t|s_t)}-1\right)\delta_t^{\varphi}\right]
\right]
-
\|\bd_{\pi_{\bm {\theta}}}^{\lambda}-\bd_{\policyy}^{\lambda}\|_{p}\|{\bm{\delta}}^{\varphi}_{\policy,t}\|_{q}.
\end{flalign}
and
\begin{flalign}
\label{app-term-02}
&M^{\varphi,+}_{p,q,t}(\policy,\policyy)=
\langle \bd_{\policyy}^{\lambda},\bm{\Delta}_{t}^{\varphi}(\policy,\policyy)\rangle+\epsilon^{\varphi,(\lambda)}_{p,q,t}(\policy,\policyy)\\
\nonumber
=&\sum_{s\in\calS}d_{\policyy}^{\lambda}(s)\left(\underset{\begin{subarray}{c} s_t \sim \Pro_{\policyy}(\cdot|s)\\ a_{t}\sim{\policyy}(\cdot|s_t)\\ s_{t+1}\sim\Pro(\cdot|s_t,a_t) \end{subarray}}\E\left[\left(\dfrac{\policy(a_t|s_t)}{\policyy(a_t|s_t)}-1\right)\delta_t^{\varphi}\right]\right)
+
\|\bd_{\pi_{\bm {\theta}}}^{\lambda}-\bd_{\policyy}^{\lambda}\|_{p}\|{\bm{\delta}}^{\varphi}_{\policy,t}\|_{q}
\\
\nonumber
=&
\E_{s\sim{d}_{\policyy}^{\lambda}(\cdot)}
\left[
\underset{\begin{subarray}{c} s_t \sim \Pro_{\policyy}(\cdot|s)\\ a_{t}\sim{\policyy}(\cdot|s_t)\\ s_{t+1}\sim\Pro(\cdot|s_t,a_t) \end{subarray}}\E\left[\left(\dfrac{\policy(a_t|s_t)}{\policyy(a_t|s_t)}-1\right)\delta_t^{\varphi}\right]
\right]
+
\|\bd_{\pi_{\bm {\theta}}}^{\lambda}-\bd_{\policyy}^{\lambda}\|_{p}\|{\bm{\delta}}^{\varphi}_{\policy,t}\|_{q}.
\end{flalign}

According to (\ref{objective-difference-01-app}) and (\ref{them:ineq-01-03}), we achieve the boundedness of performance difference between two arbitrary policies $\pi_{\bm\theta}$ and $\pi_{{\bm\theta}^{'}}$:
\begin{flalign} 
\label{objective-difference-app-004}
\underbrace{\dfrac{1}{1-\tilde \gamma}
\sum_{t=0}^{\infty}\gamma^t\lambda^{t} M^{\varphi,-}_{p,q,t}(\policy,\policyy)}_{=:L^{\varphi,-}_{p,q,}}
\leq J(\pi_{\bm \theta})-J(\pi_{{\bm \theta}^{'}})
\leq
\underbrace{
\dfrac{1}{1-\tilde \gamma}\sum_{t=0}^{\infty}\gamma^t \lambda^t  M^{\varphi,+}_{p,q,t}(\policy,\policyy)
}_{=:L^{\varphi,+}_{p,q,}}
.
\end{flalign}

\end{proof}

\subsection{Proof of Proposition \ref{objective-td-error-version}}
\label{proof-pro-01-app}

\begin{proof}(of Proposition \ref{objective-td-error-version}).

{
\color{blue}
{
\textbf{Step 1: Rewrite the objective $J(\policy)$ in Eq.(\ref{lam-return-objective}).}
}
}

We rewrite the discounted distribution $\bd_{\pi_{\bm {\theta}}}^{\lambda}$ (\ref{matrixversion-lambda-dis-state-distribution}) as follows,
\begin{flalign}
\label{vector:state-distribution}
\bm{\rho}_{0}-\dfrac{1}{1-{\tilde{\gamma}}}\bd_{\policy}^{\lambda}+\dfrac{{\tilde{\gamma}}}{1-{\tilde{\gamma}}}\bP^{(\lambda)}_{\policy}\bd_{\policy}^{\lambda}=\bm{0}.
\end{flalign}
Let $\varphi(\cdot)$ be a real number function defined on the state space $\calS$, i.e., $\varphi:\calS\rightarrow\R$.
Then we define a vector function $\bm{\phi}(\cdot)\in\R^{|\calS|}$ to collect all the values  $\{\varphi(s)\}_{s\in\calS}$, and its components are 
\[
\bm{\phi}[s]=\varphi(s),~~s\in\calS.
\]
Now, we take the inner product between the vector $\bm{\phi}$ and (\ref{vector:state-distribution}), we have
\begin{flalign}
\nonumber
0&=\langle \bm{\rho}_{0}-\dfrac{1}{1-{\tilde{\gamma}}}\bd_{\policy}^{\lambda}+\dfrac{{\tilde{\gamma}}}{1-{\tilde{\gamma}}}\bP^{(\lambda)}_{\policy}\bd_{\policy}^{\lambda},\bm{\phi} \rangle
\\
\label{state-distribution-inner-initial-vec}
&=
\langle \bm{\rho}_{0},\bm{\phi}\rangle
-\dfrac{1}{1-{\tilde{\gamma}}}\langle \bd_{\policy}^{\lambda},\bm{\phi}\rangle
+
\dfrac{{\tilde{\gamma}}}{1-{\tilde{\gamma}}}\langle\bP^{(\lambda)}_{\policy}\bd_{\policy}^{\lambda},\bm{\phi}\rangle.
\end{flalign}
We express the first term $\langle \bm{\rho}_{0},\bm{\phi}\rangle$ of (\ref{state-distribution-inner-initial-vec}) as follows,
\begin{flalign}
\label{first-term}
\langle \bm{\rho}_{0},\bm{\phi}\rangle=\sum_{s\in\calS}\rho_{0}(s)\varphi(s)=\E_{s\sim\rho_{0}(\cdot)}[\varphi(s)].
\end{flalign}
We express the second term $\langle \bd_{\policy}^{\lambda},\bm{\phi}\rangle$ of (\ref{state-distribution-inner-initial-vec}) as follows,
\begin{flalign}
\label{sec-term}
-\dfrac{1}{1-{\tilde{\gamma}}}\langle \bd_{\policy}^{\lambda},\bm{\phi}\rangle=-\dfrac{1}{1-{\tilde{\gamma}}}\sum_{s\in\calS} d_{\policy}^{\lambda} (s)\varphi(s)
=-\dfrac{1}{1-{\tilde{\gamma}}}\E_{s\sim d_{\policy}^{\lambda} (\cdot)} [\varphi(s)].
\end{flalign}
We express the third term $\langle {\tilde{\gamma}}\bP^{(\lambda)}_{\policy}\bd_{\policy}^{\lambda},\bm{\phi}\rangle$ of (\ref{state-distribution-inner-initial-vec}) as follows,
\begin{flalign}
\nonumber
\dfrac{{\tilde{\gamma}}}{1-{\tilde{\gamma}}}\langle\bP^{(\lambda)}_{\policy}\bd_{\policy}^{\lambda},\bm{\phi}\rangle=&\dfrac{{\tilde{\gamma}}}{1-{\tilde{\gamma}}}\sum_{s^{'}\in\calS}\left(\bP^{(\lambda)}_{\policy}\bd_{\policy}^{\lambda}\right)[s^{'}]\varphi(s^{'})\\
\label{third-term}
=&\dfrac{{\tilde{\gamma}}}{1-{\tilde{\gamma}}}
\sum_{s^{'}\in\calS}
\left(
\sum_{s\in\calS} \Pro_{\policy}^{(\lambda)}(s^{'}|s)d_{\policy}^{\lambda}(s)
\right)
\varphi(s^{'}).
\end{flalign}

According to Lemma \ref{lem:lam-return-objective}, put the results (\ref{lam-return-objective}) and (\ref{state-distribution-inner-initial-vec}) together, we have
\begin{flalign}
\nonumber
J(\policy)\overset{(\ref{lam-return-objective}),(\ref{state-distribution-inner-initial-vec})}=&\dfrac{1}{1-{\tilde{\gamma}}}\sum_{s\in\calS}d^{\lambda}_{\policy}(s)R^{(\lambda)}_{\pi_{\bm \theta}}(s)+\langle \bm{\rho}_{0}-\dfrac{1}{1-{\tilde{\gamma}}}\bd_{\policy}^{\lambda}+\dfrac{{\tilde{\gamma}}}{1-{\tilde{\gamma}}}\bP^{(\lambda)}_{\policy}\bd_{\policy}^{\lambda},\bm{\phi} \rangle\\
\label{lam-return-objective-01}
=&\E_{s_0\sim\rho_{0}(\cdot)}[\varphi(s_0)]
+
\dfrac{1}{1-{\tilde{\gamma}}}\sum_{s\in\calS}d^{\lambda}_{\policy}(s)
\left(
R^{(\lambda)}_{\pi_{\bm \theta}}(s)+{\tilde{\gamma}}
\sum_{s^{'}\in\calS} \Pro_{\policy}^{(\lambda)}(s^{'}|s)\varphi(s^{'})
-\varphi(s)
\right),
\end{flalign}
where the last equation holds since we unfold (\ref{state-distribution-inner-initial-vec}) according to (\ref{first-term})-(\ref{third-term}).

{
\color{blue}
{
\textbf{Step 2: Rewrite the term $\left(
R^{(\lambda)}_{\pi_{\bm \theta}}(s)+{\tilde{\gamma}}
\sum_{s^{'}\in\calS} \Pro_{\policy}^{(\lambda)}(s^{'}|s)\varphi(s^{'})
-\varphi(s)
\right)$ in Eq.(\ref{lam-return-objective-01}).}
}
}

Then, we unfold the second term of (\ref{lam-return-objective-01}) as follows,
\begin{flalign}
\label{app-04}
&R^{(\lambda)}_{\pi_{\bm \theta}}(s)+{\tilde{\gamma}}
\sum_{s^{'}\in\calS} \Pro_{\policy}^{(\lambda)}(s^{'}|s)\varphi(s^{'})
-\varphi(s)
\\
\nonumber
\overset{(\ref{lam-pro-value-02}),(\ref{lam-pro-value-03})}=&\sum_{{t}=0}^{\infty}({{\gamma}}\lambda\bP_{\policy})^{{t}}\mathbf{r}_{\pi_{\bm \theta}}[s]
+{\tilde{\gamma}}
(1-\gamma\lambda)\sum_{s^{'}\in\calS} \sum_{{t}=0}^{\infty}({{\gamma}}\lambda)^{{t}}\left(\bP^{{t}+1}_{\policy}[s,s^{'}]\right)\varphi(s^{'})
-\varphi(s)\\
\overset{(\ref{def:matrix-p-lam-return})}=&
\sum_{{t}=0}^{\infty}({{\gamma}}\lambda\bP_{\policy})^{{t}}\mathbf{r}_{\pi_{\bm \theta}}[s]
+{{\gamma}}
(1-\lambda)\sum_{s^{'}\in\calS} \sum_{{t}=0}^{\infty}({{\gamma}}\lambda)^{{t}}\Pro_{\policy}(s_{t+1}=s^{'}|s)\varphi(s^{'})
-\varphi(s).
\end{flalign}

Recall the terms $ \mathbf{P}^{(\lambda)}_{\pi_{\bm \theta}},~\mathbf{r}^{(\lambda)}_{\pi_{\bm \theta}}[s]$ defined in (\ref{def:matrix-p-lam-return})-(\ref{lam-pro-value-03}),
\begin{flalign}
\label{app-004}
R^{(\lambda)}_{\pi_{\bm \theta}}(s)+\gamma(1-\lambda)\sum_{s^{'}\in\calS} \Pro_{\policy}^{(\lambda)}(s^{'}|s)\varphi(s^{'})-\varphi(s)
\end{flalign}

We consider the first term $R^{(\lambda)}_{\pi_{\bm \theta}}(s)$ of (\ref{app-04}) as follows,
\begin{flalign}
\label{app-005}
R^{(\lambda)}_{\pi_{\bm \theta}}(s)\overset{(\ref{def:matrix-p-lam-return})-(\ref{lam-pro-value-03})}=
\mathbf{r}^{(\lambda)}_{\pi_{\bm \theta}}[s]=\sum_{{t}=0}^{\infty}(\gamma\lambda)^{t}\bP_{\policy}^{{t}}\mathbf{r}_{\pi_{\bm \theta}}[s]= \sum_{{t}=0}^{\infty}\sum_{s_t\in\calS}(\gamma\lambda)^{t}\Pro_{\policy}(s_{t}|s)R_{\policy}(s_t).
\end{flalign}

We consider the second term $\tilde\gamma\sum_{s\in\calS} \Pro_{\policy}^{(\lambda)}(s^{'}|s)\varphi(s)-\varphi(s)$ of (\ref{app-04}) as follows,
\begin{flalign}
\nonumber
&\tilde\gamma\sum_{s^{'}\in\calS} \Pro_{\policy}^{(\lambda)}(s^{'}|s)\varphi(s^{'})-\varphi(s)\\
\overset{(\ref{lam-pro-value-02})}=&\tilde\gamma
(1-\gamma\lambda)\sum_{s^{'}\in\calS} \sum_{{t}=0}^{\infty}(\gamma\lambda)^{{t}}\Pro_{\policy}(s_{t+1}=s^{'}|s)\varphi(s^{'})
-\varphi(s)\\
\overset{(\ref{def:matrix-p-lam-return})}
=&\gamma
(1-\lambda)\sum_{s^{'}\in\calS} \sum_{{t}=0}^{\infty}(\gamma\lambda)^{{t}}\Pro_{\policy}(s_{t+1}=s^{'}|s)\varphi(s^{'})
-\varphi(s)\\
\nonumber
=&\gamma\sum_{s^{'}\in\calS} \sum_{{t}=0}^{\infty}(\gamma\lambda)^{{t}}\Pro_{\policy}(s_{t+1}=s^{'}|s)\varphi(s^{'})
-\sum_{s^{'}\in\calS} 
\underbrace{\left(\sum_{{t}=0}^{\infty}(\gamma\lambda)^{{t}+1}\Pro_{\policy}(s_{t+1}=s^{'}|s)\varphi(s^{'})\right)}_{=\sum_{{t}=1}^{\infty}(\gamma\lambda)^{{t}}\Pro_{\policy}(s_{t}=s^{'}|s)\varphi(s^{'})}
-\varphi(s)\\
\label{app-001}
=&\gamma\sum_{s^{'}\in\calS} \sum_{{t}=0}^{\infty}(\gamma\lambda)^{{t}}\Pro_{\policy}(s_{t+1}=s^{'}|s)\varphi(s^{'})
-
\underbrace
{
\left(
\sum_{s^{'}\in\calS} \sum_{{t}=1}^{\infty}(\gamma\lambda)^{{t}}\Pro_{\policy}(s_{t}=s^{'}|s)\varphi(s^{'})+\varphi(s)
\right)
}_{=\sum_{s^{'}\in\calS} \sum_{{t}=0}^{\infty}(\gamma\lambda)^{{t}}\Pro_{\policy}(s_{t}=s^{'}|s)\varphi(s^{'})}
\\
\label{app-003}
=&\gamma\sum_{s^{'}\in\calS} \sum_{{t}=0}^{\infty}(\gamma\lambda)^{{t}}\Pro_{\policy}(s_{t+1}=s^{'}|s)\varphi(s^{'})
-\sum_{s_t\in\calS} \sum_{{t}=0}^{\infty}(\gamma\lambda)^{{t}}\Pro_{\policy}(s_{t}|s)\varphi(s),
\end{flalign}
where the equation from Eq.(\ref{app-001}) to Eq.(\ref{app-003}) holds since: according to (\ref{special-inititial-pro}), we use the following identity 
\begin{flalign}
\nonumber
\sum_{s^{'}\in\calS} \Pro_{\policy}(s_{0}=s^{'}|s)\varphi(s^{'})=\varphi(s).
\end{flalign}

Furthermore,  take the result (\ref{app-005}) and (\ref{app-003}) to (\ref{app-004}), we have
\begin{flalign}
\nonumber
&R^{(\lambda)}_{\pi_{\bm \theta}}(s)+\tilde\gamma
\sum_{s^{'}\in\calS} \Pro_{\policy}^{(\lambda)}(s^{'}|s)\varphi(s^{'})-\varphi(s)\\
\nonumber
=& \sum_{{t}=0}^{\infty}(\gamma\lambda)^{t}
\Bigg(
\sum_{s_t\in\calS}\Pro_{\policy}(s_{t}|s)
R_{\policy}(s_t)+\gamma\sum_{s^{'}\in\calS}
\underbrace{\Pro_{\policy}(s_{t+1}=s^{'}|s)\varphi(s^{'})}_{\overset{(\ref{pro-pi-t-step-app})}=
\sum_{s_t\in\calS}
\Pro_{\policy}(s_{t+1}=s^{'}|s_t)\Pro_{\policy}(s_t|s)\varphi(s^{'})}
\\
\label{app-020}
&
~~~~~~~~~~~~~~~~~~~~~~~~~~~~~~~~~~~~~~~~~~~~~~~~~~~~~~~~~~~~~~~~~~~~~~~~~~~~~~~~~~~~
-\sum_{s_{t}\in\calS}\Pro_{\policy}(s_{t}|s)\varphi(s_t)
\Bigg)
\\
\nonumber
=& \sum_{{t}=0}^{\infty}(\gamma\lambda)^{t}
\left(
\sum_{s_t\in\calS}\Pro_{\policy}(s_{t}|s)R_{\policy}(s_t)+\gamma\sum_{s_t\in\calS}\Pro_{\policy}(s_{t}|s)\sum_{s_{t+1}\in\calS}\Pro_{\policy}(s_{t+1}|s_{t})\varphi(s_{t+1})
\right.
\\
\label{app-021}
&
\left.
~~~~~~~~~~~~~~~~~~~~~~~~~~~~~~~~~~~~~~~~~~~~~~~~~~~~~~~~~~~~~~~~~~~~~~~~~~~~~~~~~~~~
-\sum_{s_t\in\calS}\Pro_{\policy}(s_{t}|s)\varphi(s_t)
\right)
\\
\nonumber
= &\sum_{{t}=0}^{\infty}(\gamma\lambda)^{t}\sum_{s_t\in\calS}
\Pro_{\policy}(s_{t}|s)
\left(
\underbrace{
\sum_{a_t\in\mathcal{A}}{\policy}(a_t|s_t)\sum_{s_{t+1}\in\calS}\Pro(s_{t+1}|s_t,a_t)r(s_{t+1}|s_t,a_t)
}_{=R_{\policy}(s_t)}
\right.
\\
\nonumber
&\left. ~~~~~~~~~~~~~~~~~~~~~~~~~~~~~~~~~~~~~~~~~~~~~~~~~~~~~~~~+\gamma\underbrace{\sum_{a_t\in\mathcal{A}}{\policy}(a_t|s_t)\sum_{s_{t+1}\in\calS}\Pro(s_{t+1}|s_t,a_t)}_{=\Pro_{\policy}(s_{t+1}|s_{t})}\varphi(s_{t+1})
-\varphi(s_{t})
\right)\\
\label{app-022}
=& \sum_{{t}=0}^{\infty}(\gamma\lambda)^{t}\sum_{s_t\in\calS}\Pro_{\policy}(s_{t}|s)\sum_{a_t\in\mathcal{A}}{\policy}(a_t|s_t)\sum_{s_{t+1}\in\calS}\Pro(s_{t+1}|s_t,a_t)
\left(r(s_{t+1}|s_t,a_t)+\gamma\varphi(s_{t+1})-\varphi(s_{t})\right)\\
\label{app-023}
=& \sum_{{t}=0}^{\infty}(\gamma\lambda)^{t}\E_{s_{t}\sim\Pro_{\policy}(\cdot|s),a_{t}\sim{\policy}(\cdot|s_t),s_{t+1}\sim\Pro(\cdot|s_t,a_t)}\left[r(s_{t+1}|s_t,a_t)+\gamma\varphi(s_{t+1})-\varphi(s_{t})\right],
\end{flalign}
the equation from Eq.(\ref{app-003}) to Eq.(\ref{app-020}) holds since:
\[
\Pro_{\policy}(s_{t+1}|s)\overset{(\ref{pro-pi-t-step-app})}=
\sum_{s_t\in\calS}
\Pro_{\policy}(s_{t+1}|s_t)\Pro_{\policy}(s_t|s);
\]
the equation from Eq.(\ref{app-020}) to Eq.(\ref{app-021}) holds since we use the Markov property of the definition of MDP: for each time $t\in\N$,
\[\Pro_{\policy}(s_{t+1}=s^{'}|s_t=s)=\Pro_{\policy}(s^{'}|s);\]
the equation (\ref{app-022}) the following identity:
\[
\sum_{a_t\in\mathcal{A}}{\policy}(a_t|s_t)=1,~~~~\sum_{s_{t+1}\in\calS}\Pro(s_{t+1}|s_t,a_t)=1,
\]
then
\[
\varphi(s_{t})=\sum_{a_t\in\mathcal{A}}{\policy}(a_t|s_t)\sum_{s_{t+1}\in\calS}\Pro(s_{t+1}|s_t,a)\varphi(s_{t}).
\]

{
\color{blue}
{
\textbf{Step 3: Put all the result together.}
}
}

Finally, let 
\begin{flalign}
\nonumber
\delta_t^{\varphi}&=r(s_{t+1}|s_t,a_t)+\gamma\varphi(s_{t+1})-\varphi(s_{t}),
\\
\nonumber
\delta^{\varphi}_{\policy,t}(s)&=\E_{s_{t}\sim\Pro_{\policy}(\cdot|s),a_{t}\sim{\policy}(\cdot|s_t),s_{t+1}\sim\Pro(\cdot|s_t,a_t)}\left[\delta_t^{\varphi}\right],
\end{flalign}
combining the results (\ref{lam-return-objective-01}) and (\ref{app-023}), we have
\begin{flalign}
\label{lam-return-phi-objective-prop}
J(\policy)=&\E_{s_0\sim\rho_{0}(\cdot)}[\varphi(s_0)]
+
\dfrac{1}{1-\tilde\gamma}\sum_{s\in\calS}d^{\lambda}_{\policy}(s)
\left(
\sum_{t=0}^{\infty}\gamma^t \lambda^t
\delta^{\varphi}_{\policy,t}(s)
\right)
\\
\nonumber
=&\E_{s_0\sim\rho_{0}(\cdot)}[\varphi(s_0)]
+
\dfrac{1}{1-\tilde\gamma}\E_{s\sim d^{\lambda}_{\policy}(\cdot)}
\left[
\sum_{t=0}^{\infty}\gamma^t \lambda^t
\delta^{\varphi}_{\policy,t}(s)
\right]
.
\end{flalign}
This concludes the proof of Proposition \ref{objective-td-error-version}.
\end{proof}

\subsection{Proposition~\ref{propo-03}}
\label{sec:prop-kl}

All above bound results appear in (\ref{pro1-bound-01}) and (\ref{pro2-bound-02}) can be extended for a total variational divergence to KL-divergence between policies, which are desirable for policy optimization.

We obtain
\begin{flalign}
\label{inequlities}
\E_{s\sim{d}_{\policyy}^{\lambda}(\cdot)}\left[D_{\text{TV}}(\policyy,\policy)[s]\right]
\leq&
\E_{s\sim{d}_{\policyy}^{\lambda}(\cdot)}\left[\sqrt{\frac{1}{2}\text{KL}(\policyy,\policy)[s]}\right]
\leq
\sqrt{\frac{1}{2}\E_{s\sim{d}_{\policyy}^{\lambda}(\cdot)}\left[\text{KL}(\policyy,\policy)[s]\right]},
\end{flalign}
where $\text{KL}(\cdot,\cdot)$ is KL-divergence, and \[\text{KL}(\policyy,\policy)[s]=\text{KL}(\policyy(\cdot|s),\policy(\cdot|s));\] the first inequality follows Pinsker's inequality \citep{csiszar2011information} and the second inequality follows Jensen's inequality.
According to (\ref{inequlities}), we obtain the next Proposition \ref{propo-03}.

\textbf{Proposition~\ref{propo-03}}.
\emph{
All the bounds in (\ref{pro1-bound-01}) and (\ref{pro2-bound-02}) hold if we make the following substitution:
\[
\E_{s\sim{d}_{\policyy}^{\lambda}(\cdot)}\left[D_{\emph{TV}}(\policyy,\policy)[s]\right]
\leftarrow
\sqrt{\frac{1}{2}\E_{s\sim{d}_{\policyy}^{\lambda}(\cdot)}\left[\emph{KL}(\policyy,\policy)[s]\right]}.
\]
}

\clearpage

\section{Lemma \ref{lem:difference-distri}}

In this section, we show Lemma \ref{lem:difference-distri} that presents an upper bound to the difference between two $\lambda$-version of normalized discounted distribution.
Before we present our main results, we review the norms induced by $p$-norms for matrix.

\subsection{Norms Induced by $p$-norms for Matrix}

If the $p$-norm for vectors $(1 \leq p \leq \infty)$ is used for both spaces 
$\R^{n}$ and $\R^{m}$, then the corresponding operator norm is:
\[
{\displaystyle \|\bA\|_{p}=\sup _{\bx\neq \bm{0}}{\frac {\|\bA\bx\|_{p}}{\|\bx\|_{p}}}.}
\]
These induced norms are different from the "entry-wise" $p$-norms and the \emph{Schatten} $p$-norms for matrices treated below, which are also usually denoted by 
$
{\displaystyle \|\bA\|_{p}.}
$

In the special cases of $p=1$ and $p=\infty$, the induced matrix norms can be computed or estimated by
\[{\displaystyle \|\bA\|_{1}=\max _{1\leq j\leq n}\sum _{i=1}^{m}|a_{ij}|,}\]
which is simply the maximum absolute column sum of the matrix;
\[{\displaystyle \|\bA\|_{\infty }=\max _{1\leq i\leq m}\sum _{j=1}^{n}|a_{ij}|,}\]
which is simply the maximum absolute row sum of the matrix.
Thus, the following equation holds
\begin{flalign}
\label{app-a-norm-1-infty}
 \|\bA^{\top}\|_{\infty }= \|\bA\|_{1}.
\end{flalign}

\subsection{Lemma \ref{lem:difference-distri}}
\label{sec:difference-distri}

\begin{lemma}
\label{lem:difference-distri}
The divergence between discounted future state visitation distributions, $\|\bd_{\policyy}^{\lambda}-\bd_{\pi_{\bm {\theta}}}^{\lambda}\|_1$ is bounded as follows,
\begin{flalign}
\nonumber
\|\bd_{\policyy}^{\lambda}-\bd_{\pi_{\bm {\theta}}}^{\lambda}\|_1&\leq
\dfrac{1}{1-\tilde{\gamma}}\cdot\dfrac{\tilde{\gamma}\left(\gamma\lambda(|\calS|-1)+1\right)}{1-\gamma\lambda}\E_{s\sim d_{\pi_{\bm {\theta}}}^{\lambda}(\cdot)}\left[2D_{\mathrm{TV}}(\policyy,\policy)[s]\right],
\end{flalign}
where $D_{\mathrm{TV}}(\policyy,\policy)[s]$ is the total variational divergence between action distributions at state $s$, i.e.,
\[
2D_{\mathrm{TV}}(\policyy,\policy)[s]=\sum_{a\in\calA}\left|{{\policyy}}(a|s)-{{\policy}}(a|s)\right|.
\]
\end{lemma}

\begin{proof} (of Lemma \ref{lem:difference-distri}).
Recall Eq.(\ref{matrixversion-lambda-dis-state-distribution}), we know,
\[
\bd_{\pi_{\bm {\theta}}}^{\lambda}=(1-\tilde \gamma)\sum_{t=0}^{\infty}\left(\gamma\bP^{(\lambda)}_{\policy}\right)^{t}\bm{\rho}_{0}=(1-\tilde \gamma)\left(\bI-\tilde{\gamma}\bP^{(\lambda)}_{\policy}\right)^{-1}\bm{\rho}_{0}.
\]
To short the expression, we introduce some additional notations as follows.
\begin{flalign}
\label{app-g-01}
\bG_{\policy}=\left(\bI-\tilde{\gamma}\bP^{(\lambda)}_{\policy}\right)^{-1},~\bG_{\policyy}=\left(\bI-\tilde{\gamma}\bP^{(\lambda)}_{\policyy}\right)^{-1},~\bD=\bP^{(\lambda)}_{\policyy}-\bP^{(\lambda)}_{\policy}.
\end{flalign}

Then, after some simple algebra, the following holds
\begin{flalign}
\label{app-g-02}
\bG_{\policy}^{-1}-\bG_{\policyy}^{-1}=\left(\bI-\tilde{\gamma}\bP^{(\lambda)}_{\policy}\right)-\left(\bI-\tilde{\gamma}\bP^{(\lambda)}_{\policyy}\right)=\tilde\gamma\bD.
\end{flalign}
Furthermore, by left-multiplying by $\bG_{\policy}$ and right-multiplying by $\bG_{\policyy}$, we achieve
\begin{flalign}
\label{app-g-03}
\bG_{\policyy}-\bG_{\policy}=\tilde\gamma\bG_{\policyy}\bD\bG_{\policy}.
\end{flalign}
Grouping all the results from (\ref{app-g-01})-(\ref{app-g-03}), recall (\ref{matrixversion-lambda-dis-state-distribution}),
\begin{flalign}
\label{matrixversion-lambda-dis-state-distribution-001}
\bd_{\pi_{\bm {\theta}}}^{\lambda}=(1-\tilde \gamma)\sum_{t=0}^{\infty}\left(\gamma\bP^{(\lambda)}_{\policy}\right)^{t}\bm{\rho}_{0}=(1-\tilde \gamma)\left(\bI-\tilde{\gamma}\bP^{(\lambda)}_{\policy}\right)^{-1}\bm{\rho}_{0}=(1-\tilde \gamma)\bG_{\policy}\bm{\rho}_{0},
\end{flalign}
then we have
\begin{flalign}
\nonumber
\bd_{\policyy}^{\lambda}-\bd_{\pi_{\bm {\theta}}}^{\lambda}
=&(1-\tilde \gamma)\left(\bG_{\policyy}-\bG_{\policy}\right)\bm{\rho}_{0}\\
\nonumber
\overset{(\ref{app-g-03})}=&
(1-\tilde \gamma)\tilde\gamma\bG_{\policyy}\bD\bG_{\policy}\bm{\rho}_{0}
\\
\label{app-error-gap-01}
\overset{(\ref{matrixversion-lambda-dis-state-distribution-001})}=&\tilde\gamma\bG_{\policyy}\bD\bd_{\pi_{\bm {\theta}}}^{\lambda}.
\end{flalign}
Applying (\ref{app-error-gap-01}), we have
\begin{flalign}
\label{app-distri-difference-01}
\|\bd_{\policyy}^{\lambda}-\bd_{\pi_{\bm {\theta}}}^{\lambda}\|_1\overset{(\ref{app-error-gap-01})}\leq
\tilde\gamma\|\bG_{\policyy}\|_1\|\bD\bd_{\pi_{\bm {\theta}}}^{\lambda}\|_1.
\end{flalign}
Firstly, we bound the term $\|\bG_{\policyy}\|_1$ as follows,
\begin{flalign}
\label{app-distri-difference-02}
\|\bG_{\policyy}\|_1=\left\|\left(\bI-\tilde{\gamma}\bP^{(\lambda)}_{\policyy}\right)^{-1}\right\|_1
\leq\sum_{t=0}^{\infty}\tilde{\gamma}^{t}\left\|\bP^{(\lambda)}_{\policyy}\right\|_1=\dfrac{1}{1-\tilde\gamma}.
\end{flalign}
Thus, recall $\tilde{\gamma}=\dfrac{\gamma(1-\lambda)}{1-\gamma\lambda}$, we obtain
\begin{flalign}
\|\bd_{\policyy}^{\lambda}-\bd_{\pi_{\bm {\theta}}}^{\lambda}\|_1\leq&
\tilde\gamma\|\bG_{\policyy}\|_1\|\bD\bd_{\pi_{\bm {\theta}}}^{\lambda}\|_1\leq\dfrac{\tilde{\gamma}}{1-\tilde{\gamma}}\|\bD\bd_{\pi_{\bm {\theta}}}^{\lambda}\|_1\\
\leq &\dfrac{1}{1-\tilde{\gamma}}\cdot\dfrac{\tilde{\gamma}\left(\gamma\lambda(|\calS|-1)+1\right)}{1-\gamma\lambda}\E_{s\sim d_{\pi_{\bm {\theta}}}^{\lambda}(\cdot)}\left[2D_{\mathrm{TV}}(\policyy,\policy)[s]\right],
\end{flalign}
where the last equation holds due to Lemma \ref{app-lem-tem-result},
this concludes the proof of Lemma \ref{lem:difference-distri} .
\end{proof}

\begin{lemma} 
\label{app-lem-tem-result}
The term $\|\bD\bd_{\pi_{\bm {\theta}}}^{\lambda}\|_1$ is bounded as follows,
\begin{flalign}
\nonumber
\|\bD\bd_{\pi_{\bm {\theta}}}^{\lambda}\|_1
  \leq\dfrac{\gamma\lambda(|\calS|-1)+1}{1-\gamma\lambda}\E_{s\sim d_{\pi_{\bm {\theta}}}^{\lambda}(\cdot)}\left[2D_{\mathrm{TV}}(\policyy,\policy)[s]\right].
\end{flalign}
\end{lemma}

\begin{proof}
Now, we analyze $\|\bD\bd_{\pi_{\bm {\theta}}}^{\lambda}\|_1$ as follows,
\begin{flalign}
\nonumber
&\|\bD\bd_{\pi_{\bm {\theta}}}^{\lambda}\|_1
=\sum_{s\in\calS}\left|\sum_{s^{'}\in\calS}\bD(s^{'}|s)d_{\pi_{\bm {\theta}}}^{\lambda}(s)\right|
\overset{(\ref{lam-pro-value-02})}=\sum_{s\in\calS}\left|\sum_{s^{'}\in\calS}\left(
 \Pro_{\policyy}^{(\lambda)}(s^{'}|s)-\Pro_{\policy}^{(\lambda)}(s^{'}|s)\right)\right|d_{\pi_{\bm {\theta}}}^{\lambda}(s)\\
\nonumber
 \overset{(\ref{lam-pro-value-02})}=&\sum_{s\in\calS}\left|
 (1-\gamma\lambda)\sum_{{t}=0}^{\infty}(\gamma\lambda)^{{t}}\sum_{s^{'}\in\calS}\left(\Pro_{\policyy}(s_{t+1}=s^{'}|s)-\Pro_{\policy}(s_{t+1}=s^{'}|s)\right)\right|d_{\pi_{\bm {\theta}}}^{\lambda}(s),
\end{flalign}
which implies that to bound  $\|\bD\bd_{\pi_{\bm {\theta}}}^{\lambda}\|_1$, we need to bound the following difference
\[\sum_{{t}=0}^{\infty}(\gamma\lambda)^{{t}}\sum_{s^{'}\in\calS}\left(\Pro_{\policyy}(s_{t+1}=s^{'}|s)-\Pro_{\policy}(s_{t+1}=s^{'}|s)\right).\]

{
\color{blue}
\underline
{
{
\textbf{Step 1: Rewrite $\sum_{{t}=0}^{\infty}(\gamma\lambda)^{{t}}\sum_{s^{'}\in\calS}\left(\Pro_{\policyy}(s_{t+1}=s^{'}|s)-\Pro_{\policy}(s_{t+1}=s^{'}|s)\right).$}
}
}
}

Let $s_0=s$, then 
\begin{flalign}
\mathbb{P}_{{\policy}}(s_{t+1}=s^{'}|s)\overset{(\ref{pro-pi-t-step-app})}=&\sum_{s_1\in\mathcal{S}}\mathbb{P}_{{\policy}}(s_{t+1}=s^{'}|s_1)\mathbb{P}_{{\policy}}(s_1|s_0)\\
=&\sum_{s_1\in\mathcal{S}}\mathbb{P}_{{\policy}}(s_{t+1}=s^{'}|s_1)\left(\sum_{a\in\calA}\policy(a|s_0)\mathbb{P}(s_1|s_0,a)\right).
\end{flalign}
Similarly,
\begin{flalign}
\mathbb{P}_{{\policyy}}(s_{t+1}=s^{'}|s)\overset{(\ref{pro-pi-t-step-app})}=&\sum_{s_1\in\mathcal{S}}\mathbb{P}_{{\policyy}}(s_{t+1}=s^{'}|s_1)\mathbb{P}_{{\policyy}}(s_1|s_0)\\
=&\sum_{s_1\in\mathcal{S}}\mathbb{P}_{{\policyy}}(s_{t+1}=s^{'}|s_1)\left(\sum_{a\in\calA}\policyy(a|s_0)\mathbb{P}(s_1|s_0,a)\right).
\end{flalign}
Firstly, we consider the following term
\begin{flalign}
\nonumber
 \sum_{{t}=0}^{\infty}(\gamma\lambda)^{{t}}\sum_{s^{'}\in\calS}&\Pro_{\policy}(s_{t+1}=s^{'}|s)
 =\sum_{{t}=0}^{\infty}(\gamma\lambda)^{{t}}\sum_{s^{'}\in\calS}\sum_{s_1\in\mathcal{S}}\mathbb{P}_{{\policy}}(s_{t+1}=s^{'}|s_1)\left(\sum_{a\in\calA}\policy(a|s)\mathbb{P}(s_1|s,a)\right)\\
   \nonumber
 =&\sum_{{t}=0}^{\infty}(\gamma\lambda)^{{t}}\sum_{s^{'}\in\calS}\sum_{s_1\in\mathcal{S}}\bP^{t}_{{\policy}}[s_1,s^{'}]\left(\sum_{a\in\calA}\policy(a|s)\mathbb{P}(s_1|s,a)\right)\\
  \nonumber
 =&\sum_{s^{'}\in\calS}\sum_{s_1\in\mathcal{S}}\left(\sum_{{t}=0}^{\infty}(\gamma\lambda\bP_{{\policy}})^{t}\right)[s_1,s^{'}]\left(\sum_{a\in\calA}\policy(a|s)\mathbb{P}(s_1|s,a)\right)\\
 \label{temp-app-01}
 =&\sum_{s^{'}\in\calS}\sum_{s_1\in\mathcal{S}}\left(\bI-\gamma\lambda\bP_{{\policy}}\right)^{-1}[s_1,s^{'}]\left(\sum_{a\in\calA}\policy(a|s)\mathbb{P}(s_1|s,a)\right).
\end{flalign}
To short expression, we introduce a new notation as follows,
\begin{flalign}
\bF_{\policy}=\left(\bI-\gamma\lambda\bP_{{\policy}}\right)^{-1}.
\end{flalign}
Then, we rewrite (\ref{temp-app-01}) as follows,
\begin{flalign}
 \label{temp-app-07}
 \sum_{{t}=0}^{\infty}(\gamma\lambda)^{{t}}\sum_{s^{'}\in\calS}\Pro_{\policy}(s_{t+1}=s^{'}|s)
 =\sum_{s^{'}\in\calS}\sum_{s_1\in\mathcal{S}}\bF_{\policy}[s_1,s^{'}]\left(\sum_{a\in\calA}\policy(a|s)\mathbb{P}(s_1|s,a)\right).
\end{flalign}
Furthermore, according to (\ref{temp-app-07}), we obtain
\begin{flalign}
\nonumber
&\sum_{{t}=0}^{\infty}(\gamma\lambda)^{{t}}\sum_{s^{'}\in\calS}\left(\Pro_{\policyy}(s_{t+1}=s^{'}|s)-\Pro_{\policy}(s_{t+1}=s^{'}|s)\right)\\
\nonumber
=&\sum_{s^{'}\in\calS}\sum_{s_1\in\mathcal{S}}\bF_{\policyy}[s_1,s^{'}]\left(\sum_{a\in\calA}\policyy(a|s)\mathbb{P}(s_1|s,a)\right)
-
\sum_{s^{'}\in\calS}\sum_{s_1\in\mathcal{S}}\bF_{\policy}[s_1,s^{'}]\left(\sum_{a\in\calA}\policy(a|s)\mathbb{P}(s_1|s,a)\right)\\
\nonumber
=&\sum_{s^{'}\in\calS}\sum_{s_1\in\mathcal{S}}\bF_{\policyy}[s_1,s^{'}]\left(\sum_{a\in\calA}\policyy(a|s)\mathbb{P}(s_1|s,a)\right)
-
\sum_{s^{'}\in\calS}\sum_{s_1\in\mathcal{S}}\bF_{\policyy}[s_1,s^{'}]\left(\sum_{a\in\calA}\policy(a|s)\mathbb{P}(s_1|s,a)\right)\\
\nonumber
&~~~~+\sum_{s^{'}\in\calS}\sum_{s_1\in\mathcal{S}}\bF_{\policyy}[s_1,s^{'}]\left(\sum_{a\in\calA}\policy(a|s)\mathbb{P}(s_1|s,a)\right)
-
\sum_{s^{'}\in\calS}\sum_{s_1\in\mathcal{S}}\bF_{\policy}[s_1,s^{'}]\left(\sum_{a\in\calA}\policy(a|s)\mathbb{P}(s_1|s,a)\right)\\
\label{temp-app-05}
=&\sum_{s^{'}\in\calS}\sum_{s_1\in\mathcal{S}}\bF_{\policyy}[s_1,s^{'}]\left(\sum_{a\in\calA}\Big(\policyy(a|s)-\policy(a|s)\Big)\mathbb{P}(s_1|s,a)\right)\\
\label{temp-app-02}
&~~~~+\sum_{s^{'}\in\calS}\sum_{s_1\in\mathcal{S}}\Big(\bF_{\policyy}[s_1,s^{'}]-\bF_{\policy}[s_1,s^{'}]\Big)\left(\sum_{a\in\calA}\policy(a|s)\mathbb{P}(s_1|s,a)\right),
\end{flalign}
which implies that to bound the following difference
\[\sum_{{t}=0}^{\infty}(\gamma\lambda)^{{t}}\sum_{s^{'}\in\calS}\left(\Pro_{\policyy}(s_{t+1}=s^{'}|s)-\Pro_{\policy}(s_{t+1}=s^{'}|s)\right),\]
we need to bound (\ref{temp-app-05}) and (\ref{temp-app-02}).

{
\color{blue}
\underline
{
{
\textbf{Step 2: Bound the difference (\ref{temp-app-02}).}
}
}
}

Due to the simple fact: for any inverse matrices $\bA$ and $\bB$, then the following identity holds
\[
\bA^{-1}-\bB^{-1}=\bA^{-1}(\bB-\bA)\bB^{-1},
\]
we rewrite the difference $\bF_{\policyy}-\bF_{\policy}$ as follows,
\begin{flalign}
\nonumber
\bF_{\policyy}-\bF_{\policy}&=\left(\bI-\gamma\lambda\bP_{{\policyy}}\right)^{-1}-\left(\bI-\gamma\lambda\bP_{{\policy}}\right)^{-1}\\
\nonumber
&=\gamma\lambda\left(\bI-\gamma\lambda\bP_{{\policyy}}\right)^{-1}
\left(
\bP_{{\policyy}}-\bP_{{\policy}}
\right)
\left(\bI-\gamma\lambda\bP_{{\policy}}\right)^{-1}.
\end{flalign}
Then, we rewrite (\ref{temp-app-02}) as the following matrix version
\begin{flalign}
\nonumber
&\sum_{s^{'}\in\calS}\sum_{s_1\in\mathcal{S}}\Big(\bF_{\policyy}[s_1,s^{'}]-\bF_{\policy}[s_1,s^{'}]\Big)\left(\sum_{a\in\calA}\policy(a|s)\mathbb{P}(s_1|s,a)\right)\\
\nonumber
=&\sum_{s^{''}\in\calS}\sum_{s^{'}\in\mathcal{S}}\Big(\bF_{\policyy}[s^{'},s^{''}]-\bF_{\policy}[s^{'},s^{''}]\Big)\left(\sum_{a\in\calA}\policy(a|s)\mathbb{P}(s^{'}|s,a)\right)
=\left\|\left(\bF^{\top}_{\policyy}-\bF^{\top}_{\policy}\right)\bp_{\policy}(s)\right\|_{1},
\end{flalign}
where $\bp_{\policy}(s)\in\R^{|\calS|}$, and
\[
\bp_{\policy}(s)=\left(
\Pro_{\policy}(s_1|s),\Pro_{\policy}(s_2|s),\cdots,\Pro_{\policy}\left(s_{|\calS|} |s\right)
\right)^{\top}.
\]
According to (\ref{app-a-norm-1-infty}), we obtain 
\begin{flalign}
\nonumber
&\left\|\left(\bF^{\top}_{\policyy}-\bF^{\top}_{\policy}\right)\bp_{\policy}(s)\right\|_{1}=\left\|\bp^{\top}_{\policy}(s)\left(\bF_{\policyy}-\bF_{\policy}\right)\right\|_{\infty}\\
\nonumber
=&\gamma\lambda\left\|\bp^{\top}_{\policy}(s)\left(\bI-\gamma\lambda\bP_{{\policyy}}\right)^{-1}
\left(
\bP_{{\policyy}}-\bP_{{\policy}}
\right)
\left(\bI-\gamma\lambda\bP_{{\policy}}\right)^{-1}\right\|_{\infty}\\
\label{temp-app-08}
=&\dfrac{\gamma\lambda}{1-\gamma\lambda}\Bigg\|\underbrace{\bp^{\top}_{\policy}(s)(1-\gamma\lambda)\left(\bI-\gamma\lambda\bP_{{\policyy}}\right)^{-1}}_{\bbf^{\top}_{s}}
\left(
\bP_{{\policyy}}-\bP_{{\policy}}
\right)
\left(\bI-\gamma\lambda\bP_{{\policy}}\right)^{-1}\Bigg\|_{\infty}\\
\nonumber
\leq&\dfrac{\gamma\lambda}{1-\gamma\lambda}\left\|\bbf^{\top}_{s}
\left(
\bP_{{\policyy}}-\bP_{{\policy}}
\right)\right\|_{\infty}
\left\|\left(\bI-\gamma\lambda\bP_{{\policy}}\right)^{-1}\right\|_{\infty}\\
\label{temp-app-03}
=&\dfrac{\gamma\lambda}{(1-\gamma\lambda)^2}\left\|\bbf^{\top}_{s}
\left(
\bP_{{\policyy}}-\bP_{{\policy}}
\right)\right\|_{\infty}\\
\label{temp-app-04}
=&\dfrac{2\gamma\lambda}{(1-\gamma\lambda)^2}\sum_{s\in\calS}D_{\mathrm{TV}}(\policyy,\policy)[s],
\end{flalign}
where in Eq.(\ref{temp-app-08}), we introduce a notation $\bbf^{\top}_{s}\in\R^{|\calS|}$ as follows,
\[
\bbf^{\top}_{s}=:\bp^{\top}_{\policy}(s)(1-\gamma\lambda)\left(\bI-\gamma\lambda\bP_{{\policyy}}\right)^{-1};
\]
Eq.(\ref{temp-app-03}) holds since:
\[\left\|\left(\bI-\gamma\lambda\bP_{{\policy}}\right)^{-1}\right\|_{\infty}=\dfrac{1}{1-\gamma\lambda};\]
Eq.(\ref{temp-app-04}) holds since:
 \begin{flalign}
 \nonumber
 \left\|\bbf^{\top}\left(\bP_{{\policyy}}-\bP_{{\policy}}\right)\right\|_{\infty}&=\sum_{s\in\calS}\sum_{s^{'}\in\calS}\bbf_{s}[s^{'}]\left|\Pro_{{\policyy}}(s^{'}|s)-\Pro_{{\policy}}(s^{'}|s)\right|\\
  \nonumber
&=\sum_{s\in\calS}\sum_{s^{'}\in\calS}\bbf_{s}[s^{'}]\left|\sum_{a\in\calA}\Pro(s^{'}|s,a)\left(\policyy(a|s)-\policy(a|s)\right)\right|\\
  \nonumber
&\leq\sum_{s\in\calS}\sum_{s^{'}\in\calS}\bbf_{s}[s^{'}]\sum_{a\in\calA}\left|\policyy(a|s)-\policy(a|s)\right|\\
\nonumber
&=\sum_{s\in\calS}\sum_{a\in\calA}\left|\policyy(a|s)-\policy(a|s)\right|,
 \end{flalign}
 where the last equation holds due to the following fact
 \[
 \sum_{s^{'}\in\calS}\bbf_{s}[s^{'}]= \sum_{s^{'}\in\calS}\bp^{\top}_{\policy}(s)(1-\gamma\lambda)\left(\bI-\gamma\lambda\bP_{{\policyy}}\right)^{-1}[s^{'}]=1.
 \]
 Thus, the difference (\ref{temp-app-02}) is bounded as follows,
 \begin{flalign}
 \nonumber
\sum_{s^{'}\in\calS}\sum_{s_1\in\mathcal{S}}\Big(\bF_{\policyy}[s_1,s^{'}]-\bF_{\policy}[s_1,s^{'}]\Big)\left(\sum_{a\in\calA}\policy(a|s)\mathbb{P}(s_1|s,a)\right)\leq\dfrac{2\gamma\lambda}{(1-\gamma\lambda)^2}\sum_{s\in\calS}D_{\mathrm{TV}}(\policyy,\policy)[s].
\end{flalign}

 {
\color{blue}
\underline
{
{
\textbf{Step 3: Bound the difference (\ref{temp-app-05}).}
}
}
}

 We turn to consider (\ref{temp-app-05}):
 \begin{flalign}
 \nonumber
&\sum_{s^{'}\in\calS}\sum_{s_1\in\mathcal{S}}\bF_{\policyy}[s_1,s^{'}]\left(\sum_{a\in\calA}\Big(\policyy(a|s)-\policy(a|s)\Big)\mathbb{P}(s_1|s,a)\right)\\
 \nonumber
 =&\sum_{s^{''}\in\calS}\sum_{s^{'}\in\mathcal{S}}\bF_{\policyy}[s^{'},s^{''}]\left(\sum_{a\in\calA}\Big(\policyy(a|s)-\policy(a|s)\Big)\mathbb{P}(s^{'}|s,a)\right)\\
 \nonumber
 =&\left\|\bF^{\top}_{\policyy}(\bp_{\policyy}(s)-\bp_{\policy}(s))\right\|_{1}\\
  \nonumber
 \leq&\left\|\bF^{\top}_{\policyy}\right\|_{1}\left\|\bp_{\policyy}(s)-\bp_{\policy}(s)\right\|_{1}\\
 \label{temp-app-06}
=&\left\|\bp_{\policyy}(s)-\bp_{\policy}(s)\right\|_{1}\leq\dfrac{2}{1-\gamma\lambda}D_{\mathrm{TV}}(\policyy,\policy)[s],
\end{flalign}
where the last Eq.(\ref{temp-app-06}) holds since:
\[\left\|\bF^{\top}_{\policyy}\right\|_{1}=\dfrac{1}{1-\gamma\lambda},\]
and
\begin{flalign}
\nonumber
\left\|\bp_{\policyy}(s)-\bp_{\policy}(s)\right\|_{1}&=\sum_{s^{'}\in\calS}\left|\Pro_{{\policyy}}(s^{'}|s)-\Pro_{{\policy}}(s^{'}|s)\right|\\
  \nonumber
 &=\sum_{s^{'}\in\calS}\left|\sum_{a\in\calA}\Pro(s^{'}|s,a)\left(\policyy(a|s)-\policy(a|s)\right)\right|\\
 \nonumber
 &\leq\sum_{a\in\calA}\sum_{s^{'}\in\calS}\Pro(s^{'}|s,a)\left|\policyy(a|s)-\policy(a|s)\right|\\
  \nonumber
 &=\sum_{a\in\calA}\left|\policyy(a|s)-\policy(a|s)\right|=2D_{\mathrm{TV}}(\policyy,\policy)[s].
\end{flalign}

{
\color{blue}
\underline
{
{
\textbf{Step 4: Put all the result together.}
}
}
}

Finally, according to (\ref{temp-app-05}), (\ref{temp-app-02}), (\ref{temp-app-04}), and (\ref{temp-app-06}), we obtain
\begin{flalign}
\nonumber
\|\bD\bd_{\pi_{\bm {\theta}}}^{\lambda}\|_1=&\sum_{s\in\calS}\left|
 (1-\gamma\lambda)\sum_{{t}=0}^{\infty}(\gamma\lambda)^{{t}}\sum_{s^{'}\in\calS}\left(\Pro_{\policyy}(s_{t+1}=s^{'}|s)-\Pro_{\policy}(s_{t+1}=s^{'}|s)\right)\right|d_{\pi_{\bm {\theta}}}^{\lambda}(s)\\
 \nonumber
 \leq&\sum_{s\in\calS}d_{\pi_{\bm {\theta}}}^{\lambda}(s)\left[
 \dfrac{2\gamma\lambda}{1-\gamma\lambda}\sum_{s\in\calS}D_{\mathrm{TV}}(\policyy,\policy)[s]+2D_{\mathrm{TV}}(\policyy,\policy)[s]
 \right]\\
  \nonumber
  =&\sum_{s\in\calS}d_{\pi_{\bm {\theta}}}^{\lambda}(s)\left[
 \dfrac{2\gamma\lambda|\calS|}{1-\gamma\lambda}D_{\mathrm{TV}}(\policyy,\policy)[s]+2D_{\mathrm{TV}}(\policyy,\policy)[s]
 \right]\\
 \nonumber
  =& \dfrac{\gamma\lambda(|\calS|-1)+1}{1-\gamma\lambda}\E_{s\sim d_{\pi_{\bm {\theta}}}^{\lambda}(\cdot)}\left[2D_{\mathrm{TV}}(\policyy,\policy)[s]\right].
\end{flalign}
This concludes the result of Lemma \ref{app-lem-tem-result}.
\end{proof}

\clearpage

\section{Proof of Theorem \ref{them-re-cost}}
\label{sec:app-them2}

Before we present the main result, we define some notations.
\begin{flalign}
\chi_k=&\E_{s\sim{d}_{{\pi_{\bm{\theta}_k}}}^{\lambda}(\cdot)}\left[\mathtt{KL}\left(\pi_{\bm{\theta}_k},\pi_{\bm{\theta}_{k+\frac{1}{2}}}\right)[s]\right],\\
\iota=&\dfrac{\tilde{\gamma}\left(\gamma\lambda(|\calS|-1)+1\right)}{(1-\tilde{\gamma})(1-\gamma\lambda)}.
\end{flalign}

\begin{proof}(of Theorem \ref{them-re-cost})

According to Bregman divergence, if policy $\pi_{\bm{\theta}_{k}}$ is feasible, policy $\pi_{\bm{\theta}_{k+1}}$ is generated according to (\ref{projection}), then
the following 
\[
\mathrm{KL}\left(\pi_{\bm{\theta}_{k}},\pi_{\bm{\theta}_{k+\frac{1}{2}}}\right)
\ge 
\mathrm{KL}\left(\pi_{\bm{\theta}_{k}},\pi_{\bm{\theta}_{k+1}}\right)
+
\mathrm{KL}\left(\pi_{\bm{\theta}_{k+1}},\pi_{\bm{\theta}_{k+\frac{1}{2}}}\right)
\]
implies
\[
\chi_k=\E_{s\sim{d}_{{\pi_{\bm{\theta}_k}}}^{\lambda}(\cdot)}\left[\mathrm{KL}\left(\pi_{\bm{\theta}_k},\pi_{\bm{\theta}_{k+\frac{1}{2}}}\right)[s]\right]
\ge
\E_{s\sim{d}_{{\pi_{\bm{\theta}_k}}}^{\lambda}(\cdot)}\left[\mathrm{KL}\left(\pi_{\bm{\theta}_{k+1}},\pi_{\bm{\theta}_{k}}\right)[s]\right].
\]
 According to the asymptotically symmetry of KL divergence if we update the policy within a local region, then, we have
 \[
 \chi_k\ge\E_{s\sim{d}_{{\pi_{\bm{\theta}_k}}}^{\lambda}(\cdot)}\left[\mathrm{KL}\left(\pi_{\bm{\theta}_{k+\frac{1}{2}}},\pi_{\bm{\theta}_{k}}\right)[s]\right]\ge
\E_{s\sim{d}_{{\pi_{\bm{\theta}_k}}}^{\lambda}(\cdot)}\left[\mathrm{KL}\left(\pi_{\bm{\theta}_{k+1}},\pi_{\bm{\theta}_{k}}\right)[s]\right].
 \]
Furthermore,  according to Proposition \ref{propo-01} and Proposition \ref{propo-03}, we have
\begin{flalign}
\nonumber
&J(\pi_{\bm{\theta}_{k+1}})-J(\pi_{\bm{\theta}_{k}})\\
\nonumber
\ge&
\dfrac{1}{1-\tilde\gamma}\E_{s\sim{d}_{\pi_{\bm{\theta}_{k}}}^{\lambda}(\cdot),a\sim\pi_{\bm{\theta}_{k+1}}(\cdot|s)}
\left[
A^{\mathtt{GAE}(\gamma,\lambda)}_{\pi_{\bm{\theta}_{k}}}(s,a)
-
\iota\epsilon^{V}_{\pi_{\bm{\theta}_{k+1}}}(\pi_{\bm{\theta}_k})D_{\text{TV}}(\pi_{\bm{\theta}_{k}},\pi_{\bm{\theta}_{k+1}})[s]
\right]\\
\nonumber
\ge&
\dfrac{1}{1-\tilde\gamma}\E_{s\sim{d}_{\pi_{\bm{\theta}_{k}}}^{\lambda}(\cdot),a\sim\pi_{\bm{\theta}_{k+1}}(\cdot|s)}
\left[
-
\iota\alpha_{k}\epsilon^{V}_{\pi_{\bm{\theta}_{k+1}}}(\pi_{\bm{\theta}_k})
\sqrt{\dfrac{1}{2}\mathrm{KL}(\pi_{\bm{\theta}_{k}},\pi_{\bm{\theta}_{k+1}})[s]}
\right]\\
\nonumber
\ge&-\dfrac{\iota}{1-\tilde\gamma}\alpha_{k}\sqrt{2\chi_{k}}\epsilon^{V}_{\pi_{\bm{\theta}_{k+1}}}(\pi_{\bm{\theta}_k}).
\end{flalign}

Similarly,  according to Proposition \ref{propo-01} and Proposition \ref{pro-02}, and since policy $\pi_{\bm{\theta}_{k+1}}$ satisfies
\begin{flalign}
\label{app-c-01}
J^{c}(\pi_{{\bm{\theta}}_k})+\dfrac{1}{1-\tilde\gamma}\E_{s\sim{d}_{\pi_{\bm{\theta}_k}}^{\lambda}(\cdot),a\sim\pi_{\bm{\theta}_{k+1}}(\cdot|s)}
\left[
A^{\mathtt{GAE}(\gamma,\lambda)}_{{\pi_{\bm{\theta}_k}},C}(s,a)\right]+\beta_k
\sqrt{
\E_{s\sim{d}_{{\pi_{\bm{\theta}_k}}}^{\lambda}(\cdot)}\left[\mathrm{KL}(\pi_{\bm{\theta}_k},\pi_{\bm{\theta}_{k+1}})[s]\right]}\leq b,
\end{flalign}
and 
\begin{flalign}
\label{app-c-02}
&J^{c}(\pi_{\bm{\theta}_{k+1}})-J^{c}(\pi_{\bm{\theta}_{k}})\\
\nonumber
\leq&
\frac{1}{1-\tilde\gamma}\E_{s\sim{d}_{\pi_{\bm{\theta}_{k}}}^{\lambda}(\cdot),a\sim\pi_{\bm{\theta}_{k+1}}(\cdot|s)}
\left[
A^{\mathtt{GAE}(\gamma,\lambda)}_{\pi_{\bm{\theta}_{k}},C}(s,a)
+
\iota\beta_k\epsilon^{C}_{\pi_{\bm{\theta}_{k+1}}}
D_{\text{TV}}(\pi_{\bm{\theta}_{k}},\pi_{\bm{\theta}_{k+1}})[s]
\right].
\end{flalign}
Combining (\ref{app-c-01})- (\ref{app-c-02}), we have
\begin{flalign}
\label{app-c-02}
&J^{c}(\pi_{\bm{\theta}_{k+1}})-J^{c}(\pi_{\bm{\theta}_{k}})\\
\nonumber
\leq&b+
\frac{1}{1-\tilde\gamma}\E_{s\sim{d}_{\pi_{\bm{\theta}_{k}}}^{\lambda}(\cdot),a\sim\pi_{\bm{\theta}_{k+1}}(\cdot|s)}
\left[
\iota\beta_k\epsilon^{C}_{\pi_{\bm{\theta}_{k+1}}}
\sqrt{\dfrac{1}{2}\E_{s\sim{d}_{\pi_{\bm{\theta}_{k}}}^{\lambda}(\cdot)}\left[\mathrm{KL}(\pi_{\bm{\theta}_{k}},\pi_{\bm{\theta}_{k+1}})[s]\right]}
\right]\\
\leq&b+
\frac{1}{1-\tilde\gamma}\E_{s\sim{d}_{\pi_{\bm{\theta}_{k}}}^{\lambda}(\cdot),a\sim\pi_{\bm{\theta}_{k+1}}(\cdot|s)}
\left[
\iota\beta_k\sqrt{2\chi_k}\epsilon^{C}_{\pi_{\bm{\theta}_{k+1}}}
\right].
\end{flalign}

\end{proof}

\clearpage

\section{Experiments}

\label{sec:app-ex}

The Python code for our implementation of CUP is provided along with this submission in the supplementary material.

All experiments were implemented in Pytorch 1.7.0 with CUDA 11.0 and conducted on an Ubuntu 20.04.2 LTS (GNU/Linux 5.8.0-59-generic x86 64) with 40 CPU cores (Intel(R)
Xeon(R) Silver 4210R CPU @ 2.40GHz), 251G memory and 4 GPU cards (GeForce RTX 3080).
The baseline algorithm FOCOPS based on the open-source \url{https://github.com/ymzhang01/focops}, which were offical code library. The other baseline algorithms include CPO, TRPO-L, PPO-L based on \url{https://github.com/openai/safety-starter-agents}, which published by openai.

\subsection{Algorithm Parameters}
\begin{table}[h]
    \centering
    \begin{adjustbox}{width={\textwidth}}
    \begin{tabular}{l l l l l l}
    \toprule
       Hyperparameter & CUP & PPO-L & TRPO-L & CPO & FOCOPS \\
        \midrule
    No. of hidden layers & 2 & 2 & 2 & 2 & 2 \\
      No. of hidden nodes  & 64 & 64 & 64 & 64 & 64 \\
      Activation & $\tanh$ & $\tanh$ & $\tanh$ & $\tanh$ & $\tanh$\\
      Initial log std & -0.5 & -0.5 & -1 & -0.5 & -0.5\\
      Discount for reward $\gamma$ & 0.99 & 0.99 & 0.99 & 0.99 & 0.99 \\
      Discount for cost $\gamma_{C}$ & 0.99 & 0.99 & 0.99 & 0.99 & 0.99 \\
      Batch size & 5000 & 5000 & 5000 & 5000 & 5000 \\
      Minibatch size & 64 & 64 & N/A & N/A & 64 \\
      No. of optimization epochs & 10 & 10 & N/A & N/A & 10 \\
      Maximum episode length & 1000 & 1000 & 1000 & 1000 & 1000 \\
      GAE parameter (reward) & 0.95 & 0.95 & 0.95 & 0.95 & 0.95 \\ 
      GAE parameter (cost) & 0.95 & 0.95 & 0.95 & 0.95 & 0.95 \\
      Learning rate for policy & $3\times 10^{-4}$ & $3\times 10^{-4}$ & N/A & N/A & $3\times 10^{-4}$\\
      Learning rate for reward value net & $3\times 10^{-4}$ & $3\times 10^{-4}$ & $3\times 10^{-4}$ & $3\times 10^{-4}$ & $3\times 10^{-4}$ \\
      Learning rate for cost value net & $3\times 10^{-4}$ & $3\times 10^{-4}$ & $3\times 10^{-4}$ & $3\times 10^{-4}$ & $3\times 10^{-4}$ \\
      Learning rate for $\nu$ & 0.01 & 0.01 & 0.01 & N/A & 0.01 \\
      $L2$-regularization coeff. for value net & $10^{-3}$ & $3\times 10^{-3}$ & $3\times 10^{-3}$ & $3\times 10^{-3}$ & $10^{-3}$ \\
      Clipping coefficient & N/A & 0.2 & N/A & N/A & N/A \\
      Damping coeff. & N/A & N/A & 0.01 & 0.01 & N/A \\
       Backtracking coeff. & N/A & N/A & 0.8 & 0.8 & N/A \\
       Max backtracking iterations & N/A & N/A & 10 & 10 & N/A \\
       Max conjugate gradient iterations & N/A & N/A & 10 & 10 & N/A\\
       Iterations for training value net  & 1 & 1 & 80 & 80 & 1 \\
     Temperature $\lambda$ & 1.5 & N/A & N/A & N/A & 1.5 \\
      Trust region bound $\delta$ & 0.02 & N/A & 0.01 & 0.01 & 0.02 \\
      Initial $\nu$, $\nu_{\max}$ & 0, 2 & 0, 1 & 0, 2 & N/A & 0, 2 \\
       \bottomrule
    \end{tabular}
    \end{adjustbox}
    \caption{Hyper-parameters for robots.}
    \label{tab:hyperparameters}
\end{table}

\subsection{Environment}

\label{sec:app-ex-env}

\subsubsection{Environment 1: Robots with Speed Limit.}

We consider two tasks from MuJoCo \citep{brockman2016openai}: Walker2d-v3 and Hopper-v3, where the setting of cost follows \citep{zhang2020first}.
For agents move on a two-dimensional plane, the cost is calculated as follows,
\[
C(s,a)=\sqrt{v^2_x+v^2_y};
\]
for agents move along a straight line, the cost is calculated as
\[
C(s,a)=|v_x|,
\]
where $v_x$, $v_y$ are the velocities of the agent in the $x$ and $y$ directions respectively.

\subsubsection{Environment 2: Circle.}

The Circle Environment follows \citep{AchiamHTA17}, and we use open-source implementation of the circle environments from \url{https://github.com/ymzhang01/mujoco-circle}.
According to \cite{zhang2020first}, those experiments were implemented in OpenAI Gym \citep{brockman2016openai} while the circle tasks in \cite{AchiamHTA17} were implemented in rllab \citep{duan2016benchmarking}. 
We also excluded the Point agent from the original experiments since it is not a valid agent in OpenAI Gym. The first two dimensions in the state space are the $(x,y)$ coordinates of the center mass of the agent, hence the state space for both agents has two extra dimensions compared to the standard Ant-v0 and Humanoid-v0 environments from OpenAI Gym.

Now, we present some necessary details of this environment taken from \citep{zhang2020first}.
\begin{figure}[h]
    \centering
    \includegraphics[scale=0.5]{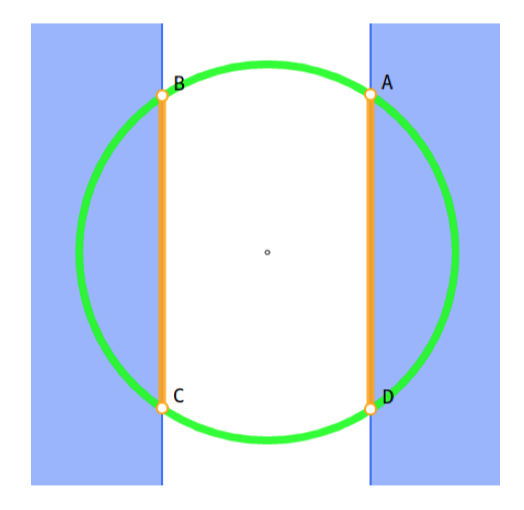}
    \caption{In the Circle task, reward is maximized by moving along the green circle. The agent is not allowed to enter the blue regions, so its optimal constrained path follows the line segments $AD$ and $BC$ (figure and caption taken from \citep{AchiamHTA17,zhang2020first}).}
    \label{fig:circle_task_geometry}
\end{figure}

In the circle tasks, the goal is for an agent to move along the circumference of a circle while remaining within a safety region smaller than the radius of the circle. The exact geometry of the task is shown in Figure \ref{fig:circle_task_geometry}.
The reward and cost functions are defined as:
\begin{flalign}
\nonumber
R(s)=\dfrac{-yv_x+xv_y}{1+|\sqrt{x^2+y^2}-r|},~~
C(s)=\mathbb{I}(|x|>x_{\lim}),
\end{flalign}
where $x,y$ are the positions of the agent on the plane, $v_x,v_y$ are the velocities of the agent along the $x$ and $y$ directions, $r$ is the radius of the circle, and $x_{\lim}$ specifies the range of the safety region. The radius is set to $r=10$ for both Ant and Humanoid while $x_{\lim}$ is set to 3 and 2.5 for Ant and Humanoid respectively. Note that these settings are identical to those of the circle task in \cite{AchiamHTA17,zhang2020first}.

\subsection{Safety Gym}

In Safety Gym environments, the agent perceives the world through a robot’s sensors and interacts with the world through its actuators \citep{Ray2019}.
In this section, we consider two robots: Point and Car, where the presentation of those safety environments are taken from \citep{Ray2019}, for more details, please refer to \citep[Page 8--10]{Ray2019}. 
In this section, we experiment with the Safety Gym environment-builder two tasks: Goal, Button.

\begin{figure}[t!]
 \centering
\subfigure[Point]
{\includegraphics[width=6.7cm,height=3.35cm]{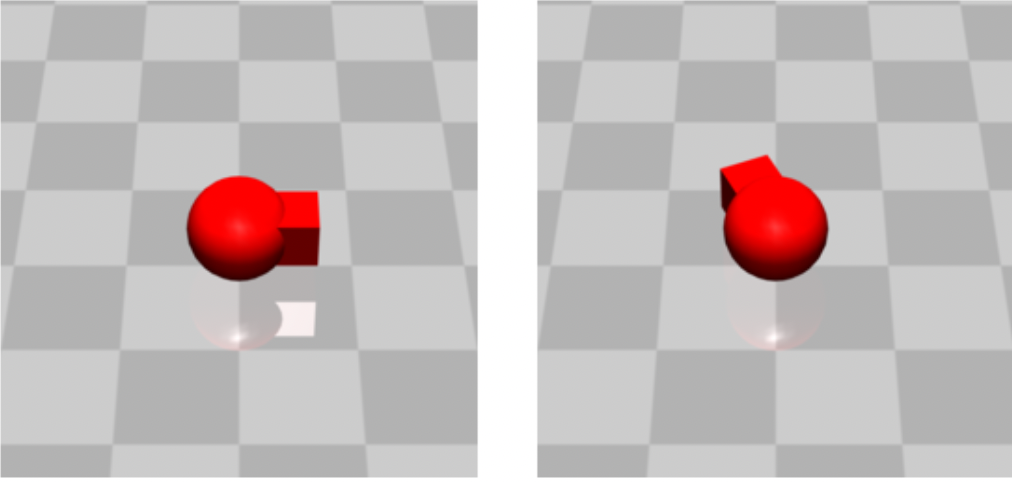}}
\subfigure[Car]
{\includegraphics[width=6.7cm,height=3.35cm]{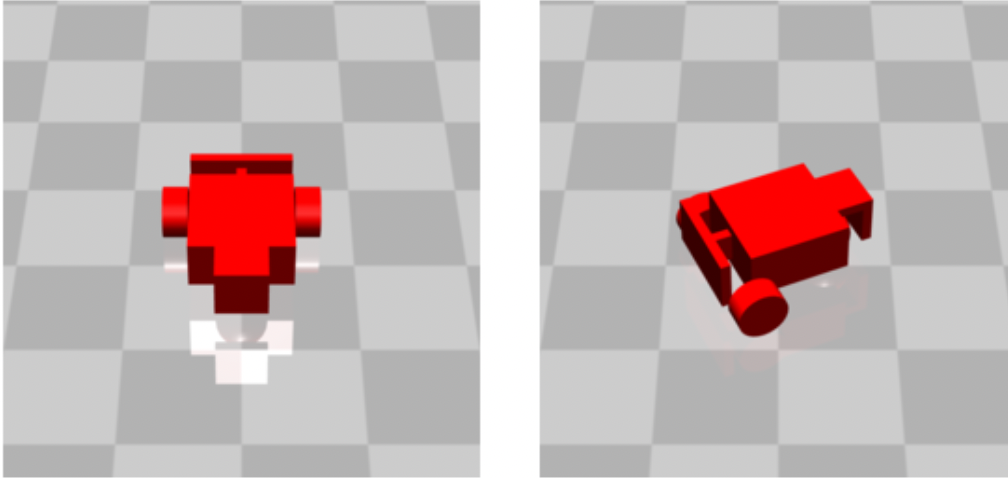}}
 \caption{
Fig (a): a 2D robot that can turn and move;
Fig (b): a wheeled robot with a differential drive control, in “Button”, the objective is to press the highlighted button (visually indicated with a faint gray cylinder), where figures and caption taken from Safety Gym \citep{Ray2019}.
 }
\label{fig:env-point-goal}
\end{figure}
\begin{figure}[t!]
 \centering
\subfigure[Goal]
{\includegraphics[width=6.7cm,height=3.35cm]{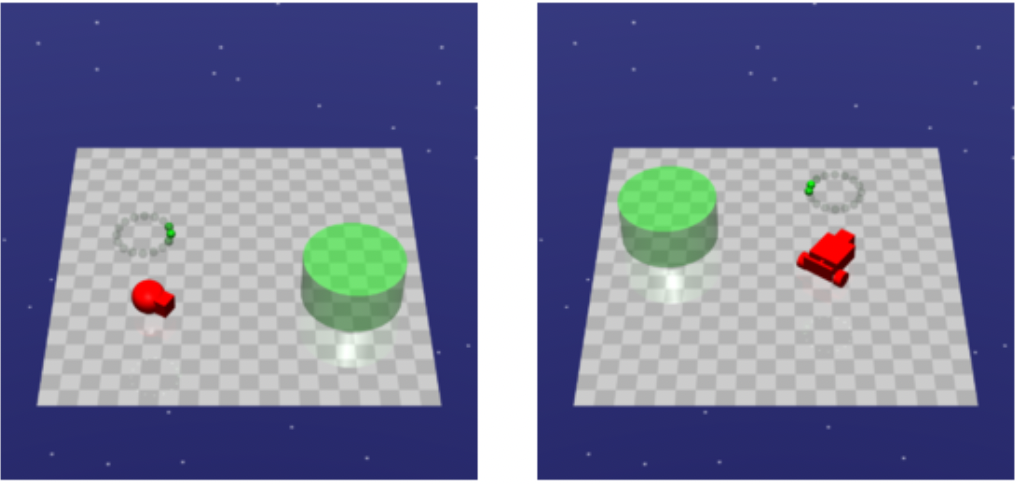}}
\subfigure[Button]
{\includegraphics[width=6.7cm,height=3.35cm]{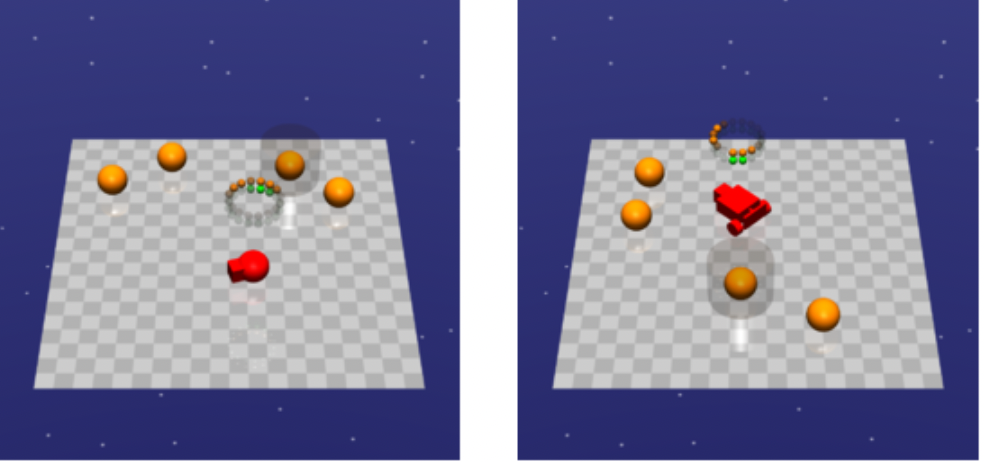}}
 \caption{
Fig (a): In “Goal,” the objective is to move the robot inside the green goal area;
Fig (b): In “Button”, the objective is to press the highlighted button (visually indicated with a faint gray cylinder), where figures and caption are taken from Safety Gym \citep{Ray2019}.
 }
\label{fig:env-car-button}
\end{figure}

\subsubsection{Safety Gym Robots}

 We consider two robots: Point and Car. All actions for all robots are continuous and linearly scaled to $[-1, +1]$, which is typical for 3D robot-based RL environments and (anecdotally) improves learning with neural nets. Modulo scaling, the action parameterization is based on a mix of hand-tuning and MuJoCo actuator defaults, and we caution that it is not clear if these choices are optimal. Some safe exploration techniques are action-layer interventions, like projecting to the closest predicted safe action \citep{dalal2018safe}, and these methods can be sensitive to action parameterization. As a result, action parameterization may merit more careful consideration than is usually given. Future work on action space design might be to find action parameterizations that respect physical measures we care about—for example, an action space where a fixed distance corresponds to a fixed amount of energy.

\textbf{Point}: A robot constrained to the 2D plane, with one actuator for turning and another for moving forward/backward. This factored control scheme makes the robot particularly easy to control for navigation. Point has a small square in front that makes it easier to visually determine the robot’s direction and helps the point push a box element that appears in one of our tasks.

\textbf{Car}: The car is a slightly more complex robot that has two independently-driven parallel wheels and a free-rolling rear wheel. The car is not fixed to the 2D plane but mostly resides in it. For this robot, both are turning and moving forward/backward require coordinating both of the actuators. It is similar in design to simple robots used in education.

\subsubsection{ Tasks}

Tasks in Safety Gym are mutually exclusive, and an individual environment can only use a single task. Reward functions are configurable, allowing rewards to be either sparse (rewards only obtained on task completion) or dense (rewards have helpful, hand-crafted shaping terms). Task details are shown as follows.

\textbf{Goal}: Move the robot to a series of goal positions. When a goal is achieved, the goal location is randomly reset to someplace new, while keeping the rest of the layout the same. The sparse reward component is attained on achieving a goal position (robot enters the goal circle). The dense reward component gives a bonus for moving towards the goal (shown in Figure \ref{fig:env-point-goal}).

\textbf{Button}: Press a series of goal buttons. Several immobile “buttons” are scattered throughout the environment, and the agent should navigate to and press (contact) the currently-highlighted button, which is the goal button. After the agent presses the correct button, the environment will select and highlight a new goal button, keeping everything else fixed. The sparse reward component is attained on pressing the current goal button. The dense reward component gives a bonus for moving towards the current goal button. We show a visualization in Figure \ref{fig:env-car-button}).

\begin{table}[t]
 \centering
  
 \vskip 0.1in
\begin{adjustbox}{width={\textwidth}}
\begin{tabular}{ccccccccc}
		\hline
		Environment & & CPO & TRPO-L & PPO-L & FOCOPS & CUP  \\ \hline
		Safexp-PointGoal1-v0 & Return & $21.29\pm3.49$ & $19.23\pm1.45$ & $16.17\pm5.89$ & $12.46\pm1.49$  & $\boldsymbol{23.74\pm0.12}$ \\ 
		Cost limit (25.0) & Constraint & $39.00\pm5.19$ & $28.20\pm5.21$ & $21.82\pm6.31$ & $34.67\pm2.62$  & $24.74\pm0.91$ \\ \hline
		Safexp-PointButton1-v0 & Return & $17.69\pm1.22$ & $5.39\pm1.02$ & $4.74\pm2.73$ & $8.36\pm0.34$ &  $\boldsymbol{19.52\pm1.38}$ \\ 
		Cost limit (25.0) & Constraint & $69.61\pm8.29$ & $25.15\pm4.88$ & $30.37\pm7.58$ & $18.56\pm1.31$  & $26.67\pm1.84$ \\ \hline
		Safexp-CarGoal1-v0 & Return & $\boldsymbol{33.00\pm0.00}$ & $17.78\pm2.34$ & $19.93\pm1.13$ & $17.73\pm3.50$  & $27.41\pm1.80$ \\ 
		Cost limit 25.0) & Constraint & $30.50\pm1.44$ & $23.00\pm4.11$ & $29.64\pm4.79$ & $25.50\pm1.43$  & $30.81\pm1.60$ \\ \hline
		Safexp-CarButton1-v0 & Return & $5.80\pm1.06$ & $0.48\pm0.15$ & $0.41\pm0.13$ & $9.47\pm1.67$ & $\boldsymbol{12.12\pm1.91}$ \\ 
		Cost limit (25.0) & Constraint & $93.88\pm13.90$ & $23.17\pm9.76$ & $16.23\pm15.55$ & $19.60\pm1.52$ & $29.41\pm0.40$ \\ \hline
	\end{tabular}
 \end{adjustbox}
  \caption{Average results for CPO, PPO-L, TRPO-L, FOCOPS, CUP over 10 seeds after 500 iterations on Safety-Gym. The agent interacts with the environment 5000 times per iteration. Constraint limits are in brackets under the environment names.}
 \label{tab:safery-gym-com}
\end{table}

\begin{figure}[t]
 \centering
 \subfigure{
 \includegraphics[width=6.7cm,height=2.9cm]{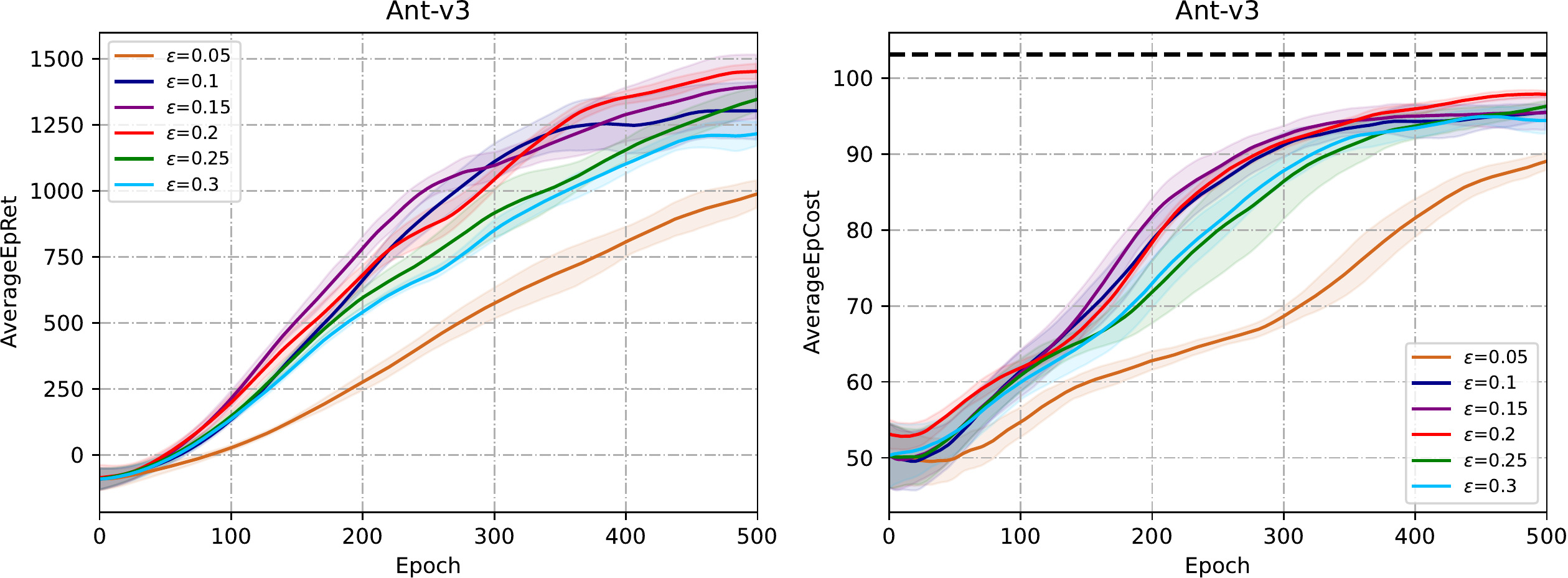}
 }
 
  \subfigure{
 \includegraphics[width=6.7cm,height=2.9cm]{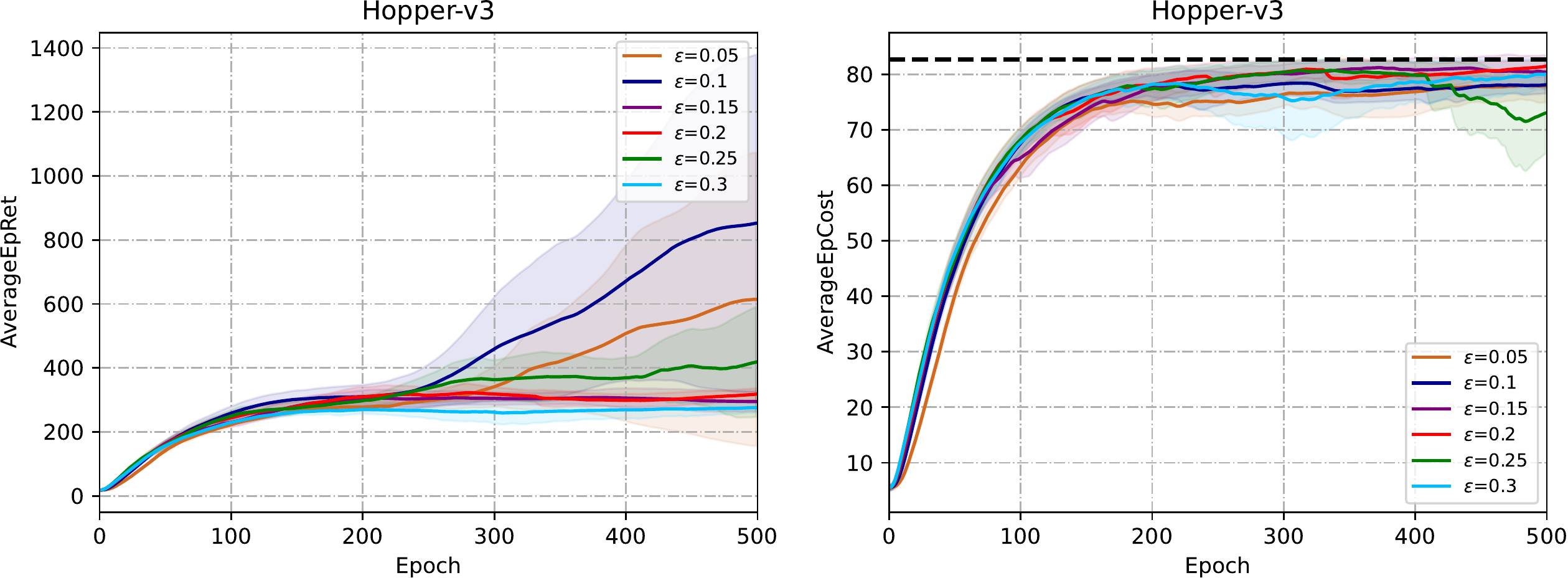}
 }
   \subfigure{
 \includegraphics[width=6.7cm,height=2.9cm]{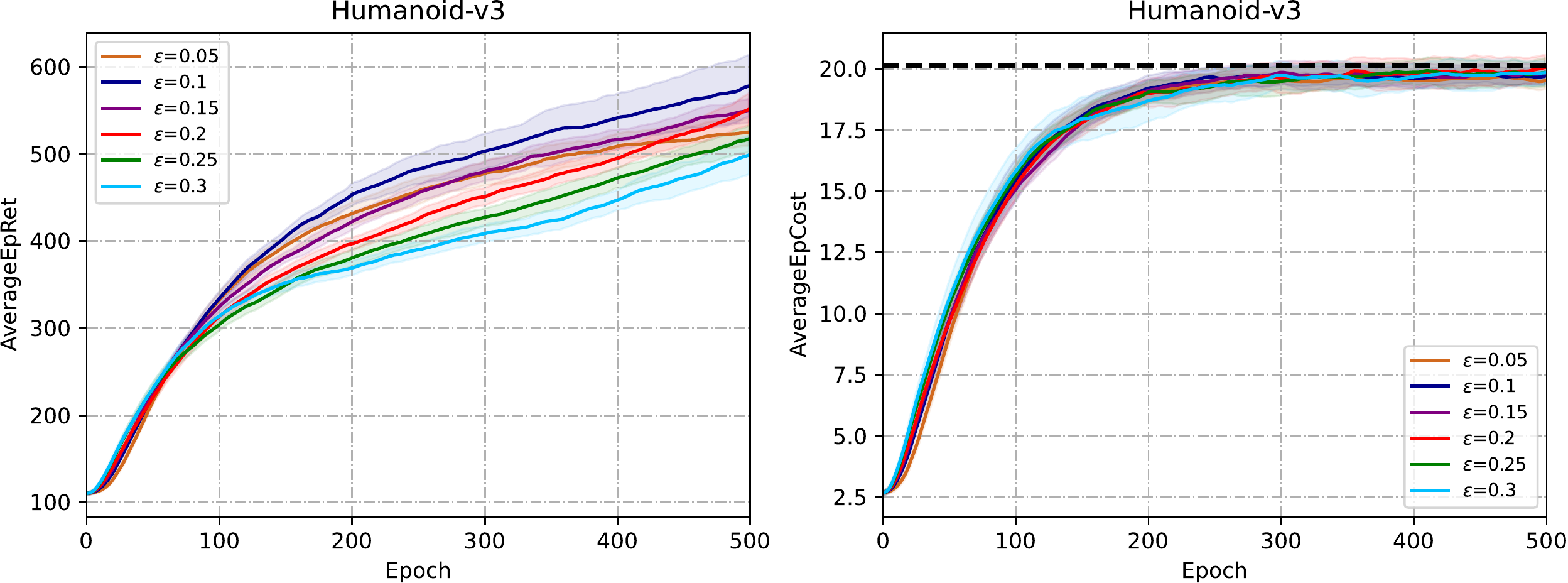}
 }
    \subfigure{
 \includegraphics[width=6.7cm,height=2.9cm]{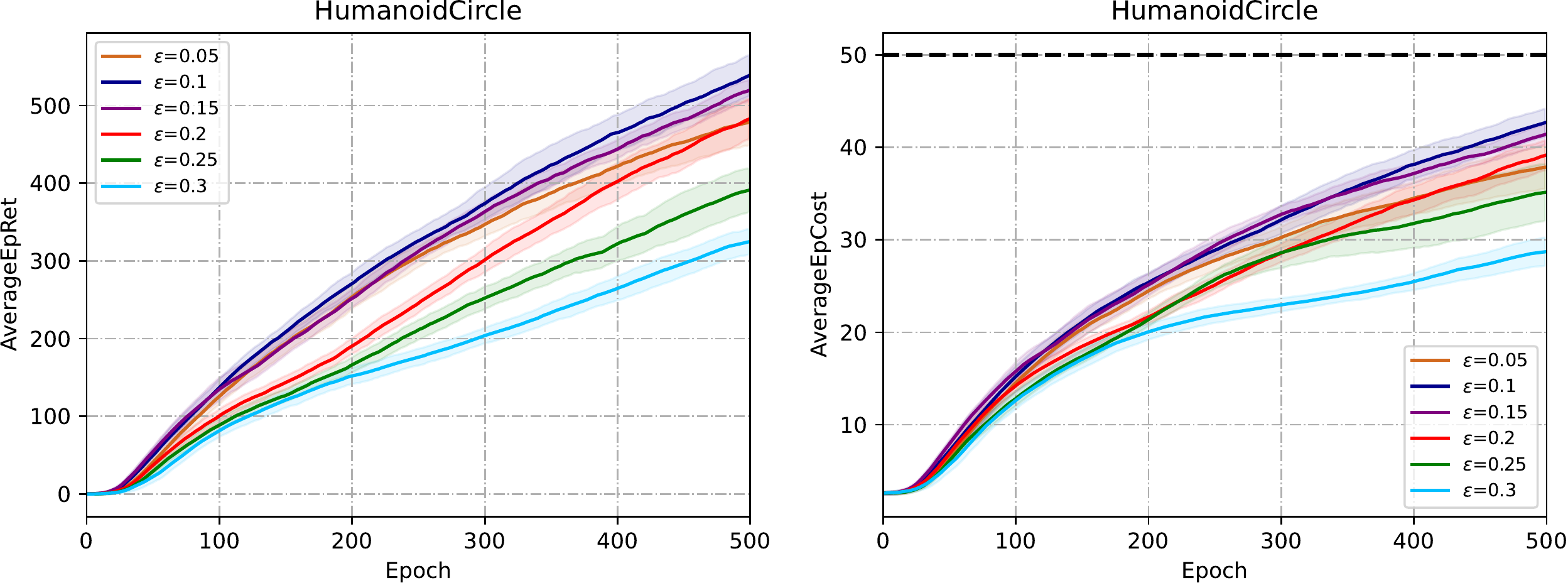}
 }
    \subfigure{
 \includegraphics[width=6.7cm,height=2.9cm]{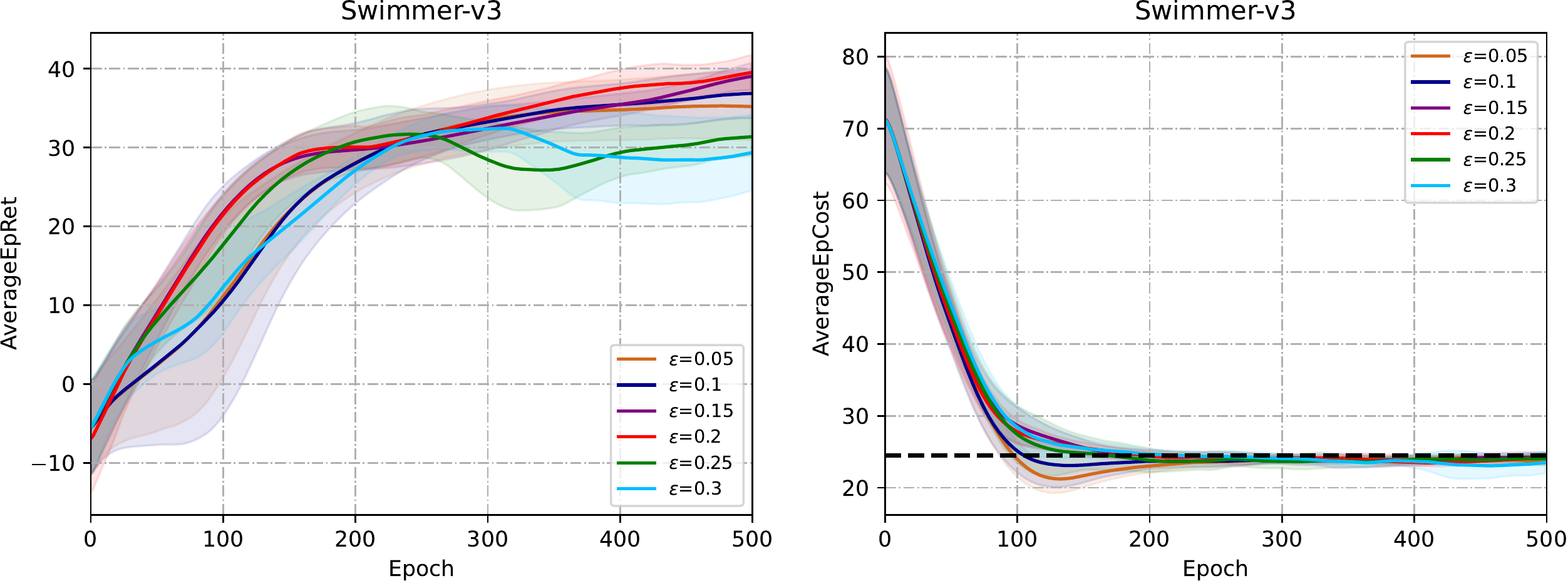}
 }
  \caption{Performance with respect to penalty factor $\epsilon$ appears in Algorithm \ref{alg-app-cpu}.}
 \label{app-epsicom}
\end{figure}

\subsection{Discussions}

Results of Figure \ref{app-epsicom} show that the performance of CUP is still very stable for different settings of $\epsilon$.
Additionally, the constraint value of CUP also still fluctuates around the target value. The different value achieved by CUP in different setting $\epsilon$ is affected by the simulated environment and constraint thresholds, which are easy to control

The results of Table \ref{tab:safery-gym-com} show that the proposed CUP significantly outperforms all the baseline algorithms except on the Safexp-CarGoal1-v0 task.
Notably, on the Safexp-PointButton1-v0 task, CUP achieve $21.27\pm 1.42$ within the safety region, while the best baseline algorithm is CPO that only obtains a reward of $17.69\pm1.22$ but it violates the cost limit $25$ more than a value of 44. 
This result is consistent with the result of Figure \ref{fig:safety-gym-comparison}.
Besides, from Table \ref{tab:safery-gym-com}, we know although CPO achieves a reward of $33\pm00$ significantly outperforms the proposed CUP in Safexp-CarGoal1-v0, CPO needs a cost $30.50\pm1.44$ higher than CUP.

\end{document}